\documentclass{article}
\usepackage[letterpaper,bindingoffset=0.2in,%
left=1in,right=1in,top=1in,bottom=1in,%
footskip=.25in]{geometry}

\usepackage{microtype}
\usepackage{graphicx}
\usepackage{subfigure}
\usepackage{booktabs} 

\usepackage{amsfonts}       
\usepackage{amsmath,amsthm}

\DeclareMathOperator*{\argmin}{arg\,min}
\usepackage{algorithm}
\usepackage{algorithmic}
\usepackage{float}
\restylefloat{table}
\usepackage{caption}
\usepackage[T1]{fontenc}
\usepackage{tikz}
\usetikzlibrary{arrows}
\usetikzlibrary{calc}
\usetikzlibrary{snakes}
\usepackage{cite}
\usepackage{graphicx}
\usepackage{array}
\usepackage{mdwmath}
\usepackage{mdwtab}
\usepackage{eqparbox}
\usepackage{fixltx2e}
\usepackage{url}
\usepackage{pgf}
\usepackage{lipsum}
\usepackage{multirow}
\usepackage{xspace}
\usepackage{xcolor}
\usepackage{algorithmic}
\usepackage{bm}
\usepackage{bbm}
\usepackage{booktabs}
\usepackage{etoolbox}
\usepackage{makecell}
\usepackage{nicefrac}
\usepackage{textcomp}
\usepackage{stmaryrd}
\usepackage{mathrsfs}
\usepackage{paralist}
\usepackage{comment}

\usepackage[font=footnotesize,labelsep=space,labelfont=bf]{caption}

\usepackage{url}

\usepackage[pdftex,bookmarks=true,pdfstartview=FitH,colorlinks,linkcolor=blue,filecolor=blue,citecolor=blue,urlcolor=blue,pagebackref=false]{hyperref}
\makeatletter

\newcommand{\Rmnum}[1]{\expandafter\@slowromancap\romannumeral #1@}
\makeatother

\newtheorem{theorem}{Theorem}
\newtheorem*{theorem*}{Theorem}
\newtheorem{lemma}{Lemma}
\newtheorem*{lemma*}{Lemma}
\newtheorem{claim}{Claim}
\newtheorem{corollary}{Corollary}
\newtheorem*{cor*}{Corollary}
\newtheorem{remark}{Remark}
\newtheorem{fact}{Fact}

\newcommand{\namedref}[2]{\hyperref[#2]{#1~\ref*{#2}}}

\definecolor{darkred}{rgb}{0.5, 0, 0} 
\definecolor{darkblue}{rgb}{0,0,0.5} 
\hypersetup{
    colorlinks=true,
    linkcolor=darkred,
    citecolor=darkblue,
    urlcolor=darkblue   
}

\newcommand{\bc}{\ensuremath{{\bf c}}\xspace}
\newcommand{\bd}{\ensuremath{{\bf d}}\xspace}

\newcommand{\bg}{\ensuremath{{\bf g}}\xspace}
\newcommand{\bh}{\ensuremath{{\bf h}}\xspace}

\newcommand{\bx}{\ensuremath{{\bf x}}\xspace}

\newcommand{\by}{\ensuremath{{\bf y}}\xspace}

\newcommand{\bz}{\ensuremath{{\bf z}}\xspace}

\newcommand{\bw}{\ensuremath{{\bf w}}\xspace}

\newcommand{\bu}{\ensuremath{{\bf u}}\xspace}

\newcommand{\bA}{\ensuremath{{\bf A}}\xspace}

\newcommand{\bX}{\ensuremath{{\bf X}}\xspace}
\newcommand{\bG}{\ensuremath{{\bf G}}\xspace}

\newcommand{\calO}{\ensuremath{\mathcal{O}}\xspace}

\newcommand{\R}{\ensuremath{\mathbb{R}}\xspace}

\renewcommand{\paragraph}[1]{\smallskip\noindent{\bf #1}~}

\newcommand{\calC}{\mathcal{C}}
\newcommand{\bbR}{\mathbb{R}}
\usepackage{stfloats}
\usepackage{threeparttable}
\usepackage{comment}
\usepackage{xcolor}

\usepackage{lmodern}
\usepackage{wrapfig}
\usepackage{enumitem}
\usepackage{caption}
\usepackage{wrapfig}

\newcommand{\cred}[1]{{\color{red} #1}}

\allowdisplaybreaks

\newtheorem*{claim*}{Claim}
\newtheorem*{corollary*}{Corollary}
\newtheorem{proposition}{Proposition}
\usepackage{hyperref}

\title{QuPeD: Quantized Personalization via Distillation with Applications to Federated Learning}
\author{
	Kaan Ozkara
	\and
	Navjot Singh
	\and
	Deepesh Data
	\and
	Suhas Diggavi
}
\date{University of California, Los Angeles, USA\\
	kaan@ucla.edu, navjotsingh@ucla.edu, deepesh.data@gmail.com, suhas@ee.ucla.edu
}
\begin{document}
	
	\maketitle

	\begin{abstract}
	Traditionally, federated learning (FL) aims to train a single global model while collaboratively using multiple clients and a server. Two natural challenges that FL algorithms face are heterogeneity in data across clients and collaboration of clients with {\em diverse resources}. In this work, we introduce a \textit{quantized} and \textit{personalized} FL algorithm QuPeD that facilitates collective (personalized model compression) training via \textit{knowledge distillation} (KD)  among clients who have access to heterogeneous data and resources. For personalization, we allow clients to learn \textit{compressed personalized models} with different quantization parameters and model dimensions/structures. Towards this, first we propose an algorithm for learning quantized models through a relaxed optimization problem, where quantization values are also optimized over. When each client participating in the (federated) learning process has different requirements for the compressed model (both in model dimension and precision), we formulate a compressed personalization framework by introducing knowledge distillation loss for local client objectives collaborating through a global model. We develop an alternating proximal gradient update for solving this compressed personalization problem, and analyze its convergence properties. Numerically, we validate that QuPeD outperforms competing personalized FL methods, FedAvg, and local training of clients in various heterogeneous settings.
\end{abstract}
	
	\section{Introduction} \label{sec:intro}

Federated Learning (FL) is a learning procedure where the aim is to utilize vast amount of data residing in numerous (in millions) edge devices (clients) to train machine learning models without collecting clients' data \cite{mcmahan2017communicationefficient}. Formally, if there are $n$ clients and $f_i$ denotes the local loss function at client $i$, then traditional FL learns a single global model by minimizing 
\begin{align} \label{fed1}
	\argmin_{\bw\in\R^d} \Big(f(\bw) := \frac{1}{n} \sum_{i=1}^n f_i(\bw)\Big).
\end{align}
It has been realized lately that a {\em single} model may not provide good performance to all the clients in settings where data is distributed {\em heterogeneously}.
This leads to the need for personalized learning, where each client wants to learn its own model  \cite{fallah2020personalized,dinh2020personalized}. Since a locally learned client model may not generalize well due to insufficient data, in personalized FL process, clients maintain personalized models locally and utilize other clients' data via a global model.
{\em Resource diversity} among clients, which is inherent to FL as the participating edge devices may vary widely in terms of resources, is often overlooked in personalized FL literature. This resource diversity may necessitate clients to learn personalized models with {\em different} as well as {\em different dimension/architecture}. Systematically studying both these resource heterogeneity together with data heterogeneity in personalized FL is the primary objective of this paper. 

In this work, we propose a model compression framework\footnote{Model compression (MC) allows inference time deployment of a compressed model. Though MC is a generic term comprising different methods, we will focus on its quantization (number of bits per model parameter) aspect.} for personalized FL via knowledge distillation (KD) \cite{hinton2015distilling} that addresses both data and resource heterogeneity in a unified manner. Our framework allows collaboration among clients with different resource requirements both in terms of {\em precision} as well as {\em model dimension/structure}, for learning personalized quantized models (PQMs). 
Motivated by FL, where edge devices are resource constrained when actively used ({\em e.g.} when several applications are actively running on a battery powered smartphone) and available for training when not in use ({\em e.g.}, while charging and on wi-fi), we do training in full precision for learning compressed models to be deployed for {\em inference time}.
For efficient model compression, we learn the quantization parameters for each client by including quantization levels in the optimization problem itself.
First, we investigate our approach in a {\em centralized} setup, by formulating a relaxed optimization problem and minimizing it through alternating proximal gradient steps, inspired by \cite{Bolte13}.
To extend this to FL for learning PQMs with {\em different} dimensions/architectures, we employ our centralized algorithm locally at clients and introduce KD loss for collaboration of personalized and global models. Although there exist empirical works where KD is used in personalized FL \cite{li2019fedmd}, we formalize it as an optimization problem, solve it using alternating proximal updates, and analyze its convergence. \\

\textbf{Contributions.}
Our contributions can be summarized as follows:
\begin{itemize}[leftmargin=*]
	\item In the centralized case, we propose a novel relaxed optimization problem that enables optimization over quantization values (centers) as well as model parameters. We use alternating proximal updates to minimize the objective and analyze its convergence properties.  

	\item More importantly, our work is the first to formulate a personalized FL optimization problem where clients may have different model dimensions and precision requirements for their personalized models. 
	Our proposed scheme combines alternating proximal updates with knowledge distillation.
	
	\item For optimizing a non-convex objective, in the centralized setup, we recover the standard convergence rate of $\calO(\nicefrac{1}{T})$ 
	(despite optimizing over quantization centers), and for federated setting, we recover the standard convergence rate of $\calO(\nicefrac{1}{\sqrt{T}})$ 
	(despite learning PQMs with different precisions/dimensions). 
	In the federated setting, our convergence bound has an error term that depends on multiplication of two terms averaged over clients: one characterizing client's local model smoothness and the other data heterogeneity with respect to overall data distribution.\footnote{ An error term depending on data heterogeneity is commonly observed in personalized FL algorithms \cite{fallah2020personalized,dinh2020personalized}.}
	
	\item  We perform image classification experiments on multiple datasets in various resource and data heterogeneity settings, and compare performance of QuPeD against Per-FedAvg \cite{fallah2020personalized}, pFedMe \cite{dinh2020personalized}, QuPeL \cite{ozkara2021qupel}, FedAvg \cite{mcmahan2017communicationefficient},  and local training of clients. 
We observe that QuPeD in full precision outperforms all these methods on all the datasets that we considered for our experiments; and even with aggressive 2-bit quantization it outperforms these methods in {\em full precision} on CIFAR-10.
\end{itemize}
Our work should not be confused with works in distributed/federated learning, where models/gradients are compressed for {\em communication efficiency} \cite{basu2019qsparse,karimireddy2019error}. We also achieve communication efficiency through local iterations, but the main goal of our work is personalized quantization for inference. \\ 

\textbf{Related work.} To the best of our knowledge, this is the first work in personalized federated learning where the aim is to learn quantized and personalized models potentially having different dimensions/structures for inference. Our work can be seen in the intersection of {\em personalized federated learning} and {\em learning quantized models}; we also employ {\em knowledge distillation} for collaboration. \\

\textit{Personalized federated learning:}
Recent works adopted different approaches for learning personalized models: 
{\sf(i)} Combine global and local models throughout the training \cite{deng2020adaptive,mansour2020approaches,hanzely2020federated};
{\sf (ii)} first learn a global model and then personalize it locally \cite{fallah2020personalized,pmlr-v139-acar21a}; 
{\sf (iii)} consider multiple global models to collaborate among only those clients that share similar personalized models \cite{zhang2021personalized,mansour2020approaches,ghosh2020efficient,smith2017federated}; 
{\sf (iv)} augment the traditional FL  objective via a penalty term that enables collaboration between global and personalized models \cite{hanzely2020federated,hanzely2020lower,dinh2020personalized}.


\textit{Learning quantized models:}
There are two kinds of approaches for training quantized networks that are of our interest. The first one approximates the hard quantization function by using a soft surrogate \cite{Yang_2019_CVPR, gong2019differentiable,louizos2018relaxed,dbouk2020dbq}, while the other one iteratively projects the model parameters onto the fixed set of centers \cite{bai2018proxquant,BinaryRelax, leng2018extremely, hou2017loss}. Each approach has its own limitation; see Section~\ref{subsec:centralized-motivation} for a discussion. While the initial focus in learning quantized networks was on achieving good empirical performance, there are some works that analyzed convergence properties \cite{li2017training,BinaryRelax,bai2018proxquant}, but only in the centralized case. 
Among these, \cite{bai2018proxquant} analyzed convergence for a relaxed/regularized loss function using proximal updates.

\textit{Knowledge distillation (KD):} KD \cite{hinton2015distilling} is a framework for transfer learning that is generally used to train a small student network using the soft labels generated by a deep teacher network. It can also be used to train two or more networks mutually by switching teacher and students in each iteration \cite{zhang2018deep}. KD has been employed in FL settings as an alternative to simple aggregation which is not feasible when clients have models with different dimensions \cite{lin2020ensemble}. \cite{li2019fedmd} used KD in personalized FL by assuming existence of a public dataset.
\cite{polino2018model} used KD in combination with quantization in a centralized case for model compression; in contrast, we do not use KD for model compression but for collaboration between personalized and global model. Unlike the above works which are empirical, our paper is the first to formalize an optimization problem for personalized FL training with KD and analyze its convergence properties. Our proposed scheme yields personalized client models with different precision/dimension through local alternating proximal updates; see Section~\ref{subsec:KD-motivate} for details.    \\

\textbf{Paper organization:} In Section~\ref{sec:problem}, we formulate the optimization problem to be minimized. In Sections~\ref{sec:centralized} and~\ref{sec:personalized}, we describe our algorithms along-with the main convergence results for the centralized and personalized settings, respectively. Section~\ref{sec:experiments} provides extensive numerical results. In Sections~\ref{sec:proof_thm1} and \ref{sec:proof_thm2} we provide the proofs for convergence results of centralized and personalized settings. Omitted proofs/details and experimental results are in appendices.

	
	\section{Problem Formulation} \label{sec:problem}
Our goal in this paper is for clients to collaboratively learn personalized quantized models (with potentially different precision and model sizes/types). To this end, below, we first state our final objective function that we will end up optimizing in this paper for learning personalized quantized models, and then in the rest of this section we will describe the genesis of this objective.

Recall from \eqref{fed1}, in the traditional FL setting, the local loss function at client $i$ is denoted by $f_i$. For personalized compressed model training, we define the following augmented loss function at client $i$:
\begin{equation}\label{qpfl}
	\begin{aligned} 
		F_i(\bx_i,\bc_i,\bw) &:= (1-\lambda_p)\left(f_i(\bx_i)+f_i(\widetilde{Q}_{\bc_i}(\bx_i))\right)+\lambda R(\bx_i,\bc_i) \\
		&\hspace{5cm} + \lambda_p\left( f^{KD}_i(\bx_i,\bw)+f^{KD}_i(\widetilde{Q}_{\bc_i}(\bx_i),\bw)\right).  
	\end{aligned}
\end{equation}
Here, $\bw \in \mathbb{R}^d$ denotes the global model, $\bx_i \in \mathbb{R}^{d_i}$ denotes the personalized model of dimension $d_i$ at client $i$, $\bc_i \in \mathbb{R}^{m_i}$ denotes the model quantization centers (where $m_i$ is the number of centers), 
$\widetilde{Q}_{\bc_i}$ denotes the soft-quantization function with respect to (w.r.t.) the set of centers $\bc_i$, $R(\bx_i,\bc_i)$ denotes the distance function, $f^{KD}_i$ denotes the knowledge distillation (KD) loss \cite{hinton2015distilling} between the two input models on client $i$'s dataset, $\lambda$ is a design parameter for enforcing quantization (large $\lambda$ forces weights to be close to respective centers), and $\lambda_p$ controls the weighted average of regular loss and KD loss functions (higher $\lambda_p$ can be used when client data is limited
). We will formally define the undefined quantities later in this section.
Consequently, our main objective becomes:
\begin{align} \label{per1}
	\min_{\substack{\bw\in\R^d,\{\bx_i\in\R^{d_i}, \bc_i\in\R^{m_i}:i=1,\hdots,n\}}} \Big(F\big(\bw,\{\bx_i\},\{\bc_i\}\big) := \frac{1}{n}\sum_{i=1}^n F_i(\bx_i,\bc_i,\bw)\Big). 
\end{align}

Thus, our formulation allows different clients to have personalized models $\bx_1,\hdots,\bx_n$ with different dimensions $d_1,\hdots,d_n$ and architectures, different number of quantization levels $m_1,\hdots,m_n$ (larger the $m_i$, higher the precision), and different quantization values in those quantization levels. 
Note that there are two layers of personalization, first is due to data heterogeneity, which is reflected in clients learning different models $\bx_1,\hdots,\bx_n$, and second is due to resource diversity, which is reflected in clients learning models with different sizes, both in terms in dimension as well as precision.
In Section~\ref{subsec:centralized-motivation}, we motivate how we came up with the first three terms in \eqref{qpfl}, which are in fact about a centralized setting because the function $f_i$ and the parameters involved, i.e., $\bx_i,\bc_i$, are local to client $i$; and then, in Section~\ref{subsec:KD-motivate}, we motivate the use of the last two terms containing $f^{KD}_{i}$ in \eqref{qpfl}.

%

\subsection{Model Compression in the Centralized Setup}\label{subsec:centralized-motivation}



Consider a setting where an objective function $f:\R^{d+m}\to\R$ (which could be a neural network loss function) is optimized over both the quantization centers $\bc\in\R^m$ and the assignment of model parameters (or weights) $\bx\in\R^d$ to those centers. There are two ways to approach this problem, and we describe these approaches, their limitations, and the possible resolutions below. \\

\textbf{Approach 1.} A natural approach is to explicitly put a constraint that weights belong to the set of centers, which suggests solving the following problem: $\min_{\bx,\bc}  f(\bx) + \delta_{\bc}(\bx)$,
where $\delta_{\bc}$ denotes the indicator function for $\bc\in\R^m$, and for any $\bc\in\R^m,\bx\in\bbR^d$, define $\delta_{\bc}(\bx):=0$ if $\forall$ $j\in[d]$, $x_j =c$ for some $c\in\{c_1,\hdots,c_m\}$, otherwise, define $\delta_{\bc}(\bx):=\infty$. However, the discontinuity of $\delta_{\bc}(\bx)$ makes minimize this objective challenging. To mitigate this, like recent works~\cite{bai2018proxquant,BinaryRelax}, we can approximate $\delta_{\bc}(\bx)$ using a distance function $R(\bx,\bc)$ that is continuous everywhere ({\em e.g.}, the $\ell_1$-distance, $R(\bx,\bc):=\min\{ \frac{1}{2} \|\bz-\bx\|_1:z_i \in \{c_1,\cdots,c_m\}, \forall i \}$).\footnote{\cite{bai2018proxquant} and \cite{BinaryRelax} proposed to approximate the indicator function $\delta_{\bc}(\bx)$ using a distance function $R_\bc(\bx)$, where $\bc$ is fixed, and unlike ours, it is not a variable that the loss function is optimized over.} This suggests solving: 
\begin{align}
	\min_{\bx,\bc} f(\bx) + \lambda R(\bx,\bc). \label{opt4}
\end{align}
The centers are optimized to be close to the mean or median (depending on $R$) of the weights; however, there is no guarantee that this will help minimizing objective $f$. 
We believe that modeling the direct effect that centers have on the loss is crucial for a complete quantized training (see Appendix~\ref{appendix:experiments} for empirical verification of this fact), and our second approach is based on this idea. \\

\textbf{Approach 2.} We can embed the quantization function into the loss function itself, thus solving the problem: $\min_{\bx,\bc} \big( h(\bx,\bc):=f(Q_{\bc}(\bx)) \big)$,
where 
for every $\bx\in\mathbb{R}^d, \bc\in\mathbb{R}^m$, the (hard) quantization function is defined as $Q_{\bc}(\bx)_i := c_k$, where $k=\argmin_{j\in[m]}\{|x_i-c_j|\}$, which maps individual weights to the closest centers.
Note that 
$Q_{\bc}(\bx)$ is actually a staircase function for which the derivative w.r.t.\ $\bx$ is 0 almost everywhere, which discourages the use of gradient-based methods to optimize the above objective.
To overcome this, similar to ~\cite{Yang_2019_CVPR,gong2019differentiable}, we can use a {\em soft} quantization function $\widetilde{Q}_\bc(\bx)$ that is differentiable everywhere with derivative not necessarily 0. For example, element-wise {\em sigmoid} or {\em tanh} functions, used by \cite{Yang_2019_CVPR} and \cite{gong2019differentiable}, respectively.\footnote{In their setup, the quantization centers are fixed. In contrast, we are also optimizing over these centers.}
%
This suggests the following relaxation: 
\begin{align}
	\min_{\bx,\bc} \big(h(\bx,\bc):=f(\widetilde{Q}_\bc(\bx))\big). \label{opt5}
\end{align}
Though we can observe the effect of centers on neural network loss in \eqref{opt5} \footnote{Non-relaxed version of the optimization problem for the first time formalizes the heuristic updates that are employed for quantization values in works such as \cite{han2016deep,polino2018model}. In particular gradient descent on centers using $\nabla_{\bc}f(Q_{\bc}(x))$ is equivalent to the updates in \cite{han2016deep,polino2018model}.}; however, the gradient w.r.t.\ $\bx$ is heavily dependent on the choice of $\widetilde{Q}_\bc$ and optimizing over $\bx$ might deviate too much from optimizing the neural network loss function. For instance, in the limiting case when 
$\widetilde{Q}_\bc(\bx) \rightarrow Q_\bc(\bx)$, gradient w.r.t.\ $\bx$ is $0$ almost everywhere; hence, every point becomes a first order stationary point.\\

\textbf{Our proposed objective for model quantization.}
Our aim is to come up with an objective function that would not diminish the significance of both $\bx$ and $\bc$ in the overall procedure. To leverage the benefits of both, we combine both optimization problems \eqref{opt4} and \eqref{opt5} into one problem: 
\begin{equation}\label{mainopt}
	\min_{\bx,\bc} \big(F_{\lambda}(\bx,\bc) := f(\bx)  + f(\widetilde{Q}_\bc(\bx))+ \lambda R(\bx,\bc)\big).
\end{equation}
Here, the first term preserves the connection of $\bx$ to neural network loss function, and the second term enables the optimization of centers w.r.t.\ the neural network training loss itself. As a result, we obtain an objective function that is continuous everywhere, and for which we can use Lipschitz tools in the convergence analysis -- which previous works did not exploit. In fact, we show the existence of Lipschitz parameters for a specific soft quantization function $\widetilde{Q}_\bc(\bx)$ based on {\em sigmoid} in Appendix~\ref{appendix:Preliminaries}. 

\begin{remark} \label{remark:centralized_compare}
	It is important to note that with this new objective function, we are able to optimize not only over weights but also over centers. This allows us to theoretically analyze how the movements of the centers affect the convergence. As far as we know, this has not been the case in the literature of quantized neural network training. 
	Moreover, we observe numerically that optimizing over centers improves performance of the network; see Appendix~\ref{appendix:experiments}.
\end{remark}

\subsection{Towards Personalized Quantized Federated Learning: Knowledge Distillation}\label{subsec:KD-motivate}
Note that the objective function defined in \eqref{mainopt} can be used for learning a quantized model locally at any client. 
There are multiple ways to extend that objective for learning personalized quantized models (PQMs) via collaboration. For example, when all clients want to learn personalized models with the {\em same} dimension (but with different quantization levels), then one natural approach is to add an $\ell_2$ penalty term in the objective that would prevent local models from drifting away from the global model and from simply fitting to local data. This approach, in fact, has been adopted in \cite{dinh2020personalized,hanzely2020federated} for learning personalized models and in \cite{li2020federated} for heterogeneous FL, though not quantized ones. In our previous work QuPeL \cite{ozkara2021qupel} we analyzed a quantized approach for learning PQMs (having the {\em same} dimension). In Section~\ref{sec:experiments}, we demonstrate that QuPeD (for the same task but using KD as opposed to the $\ell_2$ penalty) outperforms QuPeL. \\
%
%

In this paper, since we allow clients to learn PQMs with potentially {\em different} dimensions, the above approach of adding a $\ell_2$ penalty term in the objective is not feasible. Observe that, the purpose of incorporating a $\ell_2$ penalty in the objective is to ensure that the personalized models do not have significantly different output class scores compared to the global model which is trained using the data generated at all clients; this does not, however, require the global model to have the same dimension as that of local models and can be satisfied by augmenting the local objective \eqref{mainopt} with a certain {\em knowledge distillation} (KD) loss. 
In our setting, since clients' goal is to learn personalized models with different dimensions that may also have {\em different} quantization levels, we augment the local objective \eqref{mainopt} with two separate KD losses: $f^{KD}_{i}(\bx_i,\bw)$ and $f^{KD}_{i}(\widetilde{Q}_{\bc_i}(\bx_i),\bw)$, where the first one ensures that the behavior of $\bx_i\in\R^{d_i}$ is not very different from that of $\bw\in\R^d$, and the second one ensures the same for the quantized version of $\bx_i$ and $\bw$.
Formally, we define them using KL divergence as follows: $f^{KD}_{i}(\bx_i,\bw):= D_{KL}(s^{w}_{i}(\bw)\|s_i(\bx_i))$ and $f^{KD}_{i}(\widetilde{Q}_{\bc_i}(\bx_i),\bw):=D_{KL}(s^{w}_{i}(\bw)\|s_i(\widetilde{Q}_{\bc_i}(\bx_i)))$, where $s^w_i$ and $s_i$ denote functions whose inputs are global and personalized models, respectively -- and data samples implicitly -- and outputs are the softmax classification probabilities of the network. \\

We need to train $\bx_i$ and $ \bw$ mutually. Identifying the limitations of existing approaches for theoretical analysis (as mentioned in related work in Section~\ref{sec:intro}), we use reverse KL updates ({\em i.e.}, taking gradient steps w.r.t.\ the first parameter in $D_{KL}(\cdot,\cdot)$) to train the teacher network $\bw$ from the student network $\bx_i$. This type of update can be shown to converge and also empirically outperforms \cite{shen2020federated} (see Section~\ref{sec:experiments}). We want to emphasize that though there are works \cite{li2019fedmd,lin2020ensemble} that have used KD in FL and studied its performance (only empirically), ours is the first work that carefully formalizes it as an optimization problem (that also incorporate quantization) which is necessary to analyze convergence properties.

	\section{Centralized Model Quantization Training} \label{sec:centralized}

In this section, we propose a centralized training scheme (Algorithm \ref{algo:centralized}) for minimizing \eqref{mainopt} by optimizing over $\bx \in \mathbb{R}^d$ (the model parameters) and $\bc \in \mathbb{R}^m$ (quantization values/centers). During training, we keep $\bx$ full precision and learn the optimal quantization parameters $\bc$. The learned quantization values are then used to hard-quantize the personalized models to get quantized models for deployment in a memory-constrained setting.\\
		\begin{algorithm}[h]
	\caption{Centralized Model Quantization}
	{\bf Input:} Regularization parameter $\lambda$; initialize the full precision model $\bx^{0}$ and quantization centers $\bc^{0}$; a penalty function enforcing quantization $R(\bx,\bc)$; a soft quantizer $\widetilde{Q}_{\bc}(\bx)$; and learning rates $\eta_1,\eta_2$.\\
	\begin{algorithmic}[1] 	\label{algo:centralized}
		\FOR{$t=0$ \textbf{to} $T-1$}
		\STATE Compute {$\bg^{t} =  \nabla_{\bx^{t}} f(\bx^{t})+ \nabla_{\bx^{t}}  f(\widetilde{Q}_{\bc^{t}}(\bx^{t}))$}
		\STATE $\bx^{t+1}=\text{prox}_{\eta_1 \lambda R_{\bc^{t}}}(\bx^{t}-\eta_1 \bg^{t} )$ 
		\STATE Compute $\bh^{t} = \nabla_{\bc^{t}} f(\widetilde{Q}_{\bc^{t}}(\bx^{t+1})) $
		\STATE $\bc^{t+1}=\text{prox}_{\eta_2 \lambda R_{\bx^{t+1}}}(\bc^{t}-\eta_2 \bh^{t} )$ \\
		\ENDFOR
	\end{algorithmic}
	{\bf Output:} Quantized model $\hat{\bx}^{T}= Q_{\bc^{T}}(\bx^{T})$ 
\end{algorithm}


\textbf{Description of the algorithm.}
We optimize \eqref{mainopt} through alternating proximal gradient descent. 
The model parameters and the quantization vector are initialized to random vectors $\bx^{0}$ and $\bc^{0}$.
The objective in \eqref{mainopt} is composed of two parts: the loss function $f(\bx) + f(\tilde{Q}_{\bc}(\bx))$ and a quantization inducing term $R(\bx,\bc)$, which we control by a regularization coefficient $\lambda$.
At each $t$, we compute gradient $\bg^{t}$ of the loss function w.r.t.\ $\bx^{t}$ (line 2), and then take the gradient step followed by $\mathrm{prox}$ step for updating $\bx^t$ to $\bx^{t+1}$ (line 3). 
For the centers, we similarly take a gradient step and follow it by a $\mathrm{prox}$ step for updating $\bc^t$ to $\bc^{t+1}$ (line 4-5).
%
These update steps ensure that we learn the model parameters and quantization vector tied together through proximal\footnote{As a short notation, we use  $\text{prox}_{\eta_1\lambda R_{\bc^{t}}}$ to denote $\text{prox}_{\eta_1\lambda R(\cdot,\bc^{t})}$, and $\text{prox}_{\eta_2\lambda R_{\bx^{t+1}}}$ for $\text{prox}_{\eta_2\lambda R(\bx^{t+1},\cdot)}$.} 
mapping of the regularization function $R$. Finally, we quantize the full-precision model $\bx^{T}$ using $Q_{\bc^T}$ (line 7). \\ 

\textbf{Assumptions.} We make the following assumptions on $f$:\\

{\bf A.1} 
$f(\bx)>-\infty, \forall\bx\in\R^d$, which implies $F_\lambda(\bx,\bc) > -\infty$ for any $\bx\in\R^d,\bc\in\R^m,\lambda\in\R$.\\

{\bf A.2} 
$f$ is $L$-smooth, i.e., for all $\bx,\by \in \mathbb{R}^d$, we have $f(\by) \leq f(\bx) + \left\langle\nabla f(\bx), \by-\bx\right\rangle + \frac{L}{2}\|\bx-\by\|^2$.\\

{\bf A.3} $f$ has bounded gradients, 
$\|\nabla f(\bx)\|_2 \leq G<\infty,\forall\bx\in\R^d$.\\


{\bf A.4} \textit{(Smoothness of the soft quantizer):} We assume that $\widetilde{Q}_{\bc}(\bx)$ is $l_{Q_1}$-Lipschitz and $L_{Q_1}$-smooth w.r.t.\ $\bx$, i.e., for $\bc\in\R^m$: $\forall$ $\bx,\by \in \mathbb{R}^d$:
	$\|\widetilde{Q}_{\bc}(\bx) - \widetilde{Q}_{\bc}(\by) \| \leq l_{Q_1} \|\bx-\by\|$ and 
	$\|\nabla_{\bx} \widetilde{Q}_{\bc}(\bx) - \nabla_{\by} \widetilde{Q}_{\bc}(\by) \|  \leq L_{Q_1} \|\bx-\by\|$. 
We also assume $\widetilde{Q}_{\bc}(\bx)$ is $l_{Q_2}$-Lipschitz and $L_{Q_2}$-smooth w.r.t.\ $\bc$, i.e., for $\bx\in\R^d$: $\forall$ $\bc,\bd \in \mathbb{R}^m$:  
	$\|\widetilde{Q}_{\bc}(\bx) - \widetilde{Q}_{\bd}(\bx) \|  \leq l_{Q_2} \|\bc-\bd\|$ and 
	$\|\nabla_{\bc} \widetilde{Q}_{\bc}(\bx) - \nabla_{\bd} \widetilde{Q}_{\bd}(\bx) \|  \leq L_{Q_2} \|\bc-\bd\|$.\\

{\bf A.5} \textit{(Bound on partial gradients of soft quantizer):} There exists constants $G_{Q_1},G_{Q_2} <\infty $ such that:
	$\|\nabla_{\bx} \widetilde{Q}_{\bc}(\bx) \|_F = \|\nabla \widetilde{Q}_{\bc}(\bx)_{1:d,:} \|_F \leq G_{Q_1}$ and 
	$\|\nabla_{\bc} \widetilde{Q}_{\bc}(\bx) \|_F = \|\nabla \widetilde{Q}_{\bc}(\bx)_{d+1:d+m,:} \|_F \leq G_{Q_2}$,
where $\bX_{p:q,:}$ denotes sub-matrix of $\bX$ with rows between $p$ and $q$, and $\|\cdot\|_F$ is the Frobenius norm. \\

\textbf{Notes on Assumptions.} \textbf{A.1} and \textbf{A.2} are standard assumptions for convergence analysis of smooth objectives; and \textbf{A.3} is commonly used for non-convex optimization, {\em e.g.}, for personalized FL in \cite{fallah2020personalized}. \textbf{A.4} and \textbf{A.5} are assumed to make the composite function $f(\widetilde{Q}_{\bc}(\bx))$ smooth. The choice of $\widetilde{Q}_{\bc}(\bx)$ (see Appendix~\ref{appendix:Preliminaries}) naturally satisfies \textbf{A.4} and \textbf{A.5}.\\

\textbf{Convergence result.} Now we state our main convergence result (proved in Section~\ref{sec:proof_thm1}) for minimizing $F_{\lambda}(\bx,\bc)$ in \eqref{mainopt} w.r.t.\ $(\bx,\bc)\in\R^{d+m}$ via Algorithm~\ref{algo:centralized}. This provides first-order guarantees for convergence of $(\bx,\bc)$ to a stationary point and recovers the $\mathcal{O}\left(\nicefrac{1}{T}\right)$ convergence rate of \cite{Bolte13,bai2018proxquant}.
\begin{theorem}\label{thm:centralized}
	Consider running Algorithm~\ref{algo:centralized} for $T$ iterations for minimizing \eqref{mainopt} with $\eta_1=\nicefrac{1}{2(L+GL_{Q_1}+G_{Q_1}Ll_{Q_1})}$ and $\eta_2=\nicefrac{1}{2(GL_{Q_2}+G_{Q_2}Ll_{Q_2})}$. For any $t\in[T]$, define $\bG^{t} := [\nabla_{\bx^{t+1}} F_{\lambda}\left(\bx^{t+1},\bc^{t}\right)^T, \nabla_{\bc^{t+1}} F_{\lambda}\left(\bx^{t+1},\bc^{t+1}\right)^T]^T$. Then, under {\bf A.1-A.5} and for $L_{\min}=\min\{\frac{1}{\eta_1},\frac{1}{\eta_2}\}$, 
	$L_{\max} = \max\{\frac{1}{\eta_1},\frac{1}{\eta_2}\}$, we have:
	\begin{align*}
	\frac{1}{T}\sum_{t=0}^{T-1}\|\bG^{t}\|^2_{2}   = \mathcal{O} \big(\frac{L_{\max}^2\left(F_{\lambda}\left(\bx^{0},\bc^0\right){-}F_{\lambda}(\bx^{T},\bc^T)\right)}{L_{\min}T}\big).
	\end{align*}
\end{theorem}

Theorem~\ref{thm:centralized} is proved in Section~\ref{sec:proof_thm1}.
\begin{remark}
	In Theorem~\ref{thm:centralized}, we see that gradient norm decays without any constant error terms. The convergence rate depends on Lipschitz smoothness constants of $f$ and $f(\widetilde{Q}_{\bc}(.))$ through $L_{\max}$ and $L_{\min}$. Choosing a smoother $\widetilde{Q}_{\bc}(.)$ would speed up convergence; however, if chosen too small, this could result in an accuracy loss when hard-quantizing the parameters at the end of the algorithm.
\end{remark}

\begin{remark}[Number of centers and convergence] \label{remark:thm1-centers}
The number of quantization levels $m$ has a direct effect on convergence through the soft quantization function $\widetilde{Q}_{\bc}({\bx})$ and Lipschitz constants. Note that as $m \to \infty$, we have $\widetilde{Q}_{\bc}({\bx}) \to \bx, \forall \bx$. In this case, $\widetilde{Q}_{\bc}({\bx})$ is $0$-smooth and $1$-Lipschitz w.r.t.\ all parameters. As a result, we would have $L_{\min}=L$ and $L_{\max}=2L$. 
	Note that the ratio $\nicefrac{L^2_{\max}}{L_{\min}}$ increases as the quantization becomes more aggressive, and consequently, aggressive quantization has a scaling effect on convergence rate.
\end{remark}

	
	\section{Personalized Quantization for FL via Knowledge Distillation} \label{sec:personalized}

We now consider the FL setting where we aim to learn quantized and personalized models for each client with different precision and model dimensions in heterogeneous data setting.
Our proposed method QuPeD (Algorithm \ref{algo:personalized}),  
utilizes the centralized scheme of Algorithm~\ref{algo:centralized} locally at each client to minimize \eqref{per1} over $\left(\{\bx_i,\bc_i\}_{i=1}^{n},\bw\right)$. Here, $\bx_i,\bc_i$, denote the model parameters and the quantization vector (centers) for client $i$, and $\bw$ denotes the global model that facilitates collaboration among clients which is encouraged through the knowledge distillation (KD) loss in the local objectives \eqref{qpfl}. \\ 

			\begin{algorithm}[h]
		\caption{QuPeD: Quantized Personalization via Distillation}
		{\bf Input:} Regularization parameters $\lambda,\lambda_p$; synchronization gap $\tau$; for client $i\in[n]$, initialize full precision personalized models $\bx_i^{0}$, quantization centers $\bc_i^{0}$, local model $\bw_i^0$, learning rates $\eta^{(i)}_1,\eta^{(i)}_2,\eta_3$; quantization enforcing penalty function $R(\bx,\bc)$; soft quantizer $\widetilde{Q}_{\bc}(\bx)$; number of clients to be sampled $K$. \\
		\vspace{-0.3cm}
		\begin{algorithmic}[1] \label{algo:personalized}
			\FOR{$t=0$ \textbf{to} $T-1$}
			\IF{$\tau$ divides $t$}
			\STATE \textbf{On Server do:}\\
			Choose a subset of clients $\mathcal{K}_t \subseteq [n]$ with size $K$
			\STATE Broadcast $\bw^{t}$ to all \textbf{Clients}
			\STATE \textbf{On Clients} $i\in\mathcal{K}_t$ \textbf{to} $n$ (in parallel) \textbf{do}:
			\STATE Receive $\bw^{t}$ from \textbf{Server}; set $\bw_i^{t} = \bw^{t}$
			\ENDIF	
			\STATE \textbf{On Clients} $i\in\mathcal{K}_t$ \textbf{to} $n$ (in parallel) \textbf{do}:
			
			\STATE Compute $\bg_{i}^{t} := (1-\lambda_p)(\nabla_{\bx_{i}^{t}} f_i(\bx_{i}^{t}) + \nabla_{\bx_{i}^{t}} f_i(\widetilde{Q}_{\bc_{i}^{t}}(\bx_{i}^{t}))) +  \lambda_p(\nabla_{\bx_{i}^{t}} f^{KD}_i(\bx_{i}^{t}, \bw_{i}^{t})+\nabla_{\bx_{i}^{t}} f^{KD}_i(\widetilde{Q}_{\bc_{i}^{t}}(\bx_{i}^{t}), \bw_{i}^{t})) $
			\STATE $\bx_{i}^{t+1}=\text{prox}_{\eta^{(i)}_1 \lambda R_{\bc_{i}^{t}}}(\bx_{i}^{t} - \eta^{(i)}_1 \bg_{i}^{t}  )$\\
			\STATE Compute $\bh_i^{t} := (1-\lambda_p)\nabla_{\bc_{i}^{t}} f_i(\widetilde{Q}_{\bc_{i}^{t}}(\bx_{i}^{t+1})) + \lambda_p\nabla_{\bc_{i}^{t}} f^{KD}_i(\widetilde{Q}_{\bc_{i}^{t}}(\bx_{i}^{t+1}), \bw_{i}^{t})) $
			\STATE $\bc_{i}^{t+1}=\text{prox}_{\eta^{(i)}_2 \lambda R_{\bx_{i}^{t+1}}}(\bc_{i}^{t}-\eta^{(i)}_2 \bh_i^{t}  )$\\
			\STATE $\bw_{i}^{t+1} = 
			\bw_{i}^{t}-\eta_3\lambda_p(\nabla_{\bw_{i}^{t}} f^{KD}_i(\bx_{i}^{t+1}, \bw_{i}^{t})+\nabla_{\bw_{i}^{t}} f^{KD}_i(\widetilde{Q}_{\bc_{i}^{t+1}}(\bx_{i}^{t+1}), \bw_{i}^{t})) $\\
			\IF{$\tau$ divides $t+1$}
			\STATE Clients send $\bw_{i}^{t}$ to \textbf{Server}
			\STATE Server receives $\{\bw_i^{t}\}$; computes $\bw^{t+1} = \frac{1}{K} \sum_{i\in\mathcal{K}_t}^n \bw_i^{t}$
			\ENDIF
			\ENDFOR
			\STATE $\hat{\bx}_{i}^{T} = Q_{\bc_{i}^{T}}(\bx_{i}^{T})$ for all $i \in [n]$
		\end{algorithmic}
		{\bf Output:} Quantized personalized models $\{\hat{\bx}_i^{T}\}_{i=1}^n$
	\end{algorithm}

\textbf{Description of the algorithm.}
Since clients perform local iterations, apart from maintaining $\bx_i^t,\bc_i^t$ at each client $i\in[n]$, it also maintains a model $\bw_i^t$ that helps in utilizing other clients' data via collaboration.
We call $\{\bw_i^t\}_{i=1}^n$ local copies of the global model at clients at time $t$. Client $i$ updates $\bw_i^t$ in between communication rounds based on its local data and synchronizes that with the server 
which aggregates them to update the global model.
Note that the local objective in \eqref{qpfl} can be split into the weighted average of loss functions $(1-\lambda_p)(f_i(\bx_i) + f_i(\widetilde{Q}_{\bc_i}(\bx_i)))+ \lambda_p (f^{KD}_i(\bx_i,\bw_i) + f^{KD}_i(\widetilde{Q}_{\bc_i}(\bx_i),\bw_i))$ and the term enforcing 
quantization $\lambda R(\bx_i,\bc_i)$. 
At any iteration $t$ that is not a communication round (line 3), 
client $i$ first computes the gradient $\bg_i^t$ of the loss function w.r.t.\ $\bx_i^{t}$ (line 4) and then takes a gradient step followed by the proximal step using $R$ (line 5) to update from $\bx_i^t$ to $\bx_i^{t+1}$. Then it computes the gradient $\bh_i^t$ of the loss function w.r.t.\ $\bc_i^{t}$ (line 6) and updates the centers followed by the proximal step (line 7). 
Finally, 
it updates $\bw_i^{t}$ to $\bw_i^{t+1}$ by taking a gradient step of the loss function at $\bw_i^{t}$ (line 8).
Thus, the local training of $\bx_i^t,\bc_i^t$ also incorporates knowledge from other clients' data through $\bw_i^{t}$.
When $t$ is divisible by $\tau$, clients upload $\{\bw_i^{t}\}$ to the server (line 10) which aggregates them (line 15) and broadcasts the updated global model (line 16). 
At the end of training, clients learn their personalized models $\{\bx_i^{T}\}_{i=1}^{n}$ and quantization centers $\{\bc_i^{T}\}_{i=1}^{n}$. Finally, client $i$ quantizes $\bx_i^{T}$ using $Q_{\bc_i^T}$ (line 19). \\

\textbf{Assumptions.}
In addition to assumptions {\bf A.1 - A.5} (with {\bf A.3} and {\bf A.5} modified to have client specific gradient bounds \{$G^{(i)}, G^{(i)}_{Q_1}, G^{(i)}_{Q_2}$\}  as they have different model dimensions\footnote{We keep smoothness constants to be the same across clients for notational simplicity, however, our result can easily be extended to that case.}), we assume:\\ 

%
%
{\bf A.6} \textit{(Bounded diversity):} At any $t \in \{0,\cdots,T-1\}$ and any client $i\in[n]$, the variance of the local gradient (at client $i$) w.r.t. the global gradient is bounded, i.e., there exists $\kappa_i<\infty$, such that for every $\{\bx_i^{t+1}\in\R^d,\bc_i^{t+1}\in\R^{m_i}:i\in[n]\}$ and $\bw^t\in\R^d$ generated according to Algorithm~\ref{algo:personalized}, we have:
$\| \nabla_{\bw^t} F_i(\bx_{i}^{t+1},\bc_{i}^{t+1},\bw^t) - \frac{1}{n} \sum_{j=1}^n \nabla_{\bw^t} F_j(\bx_{j}^{t+1},\bc_{j}^{t+1},\bw^{t})\|^2 \leq \kappa_i$.
This assumption is equivalent to the bounded diversity assumption in \cite{dinh2020personalized,fallah2020personalized}; see Appendix~\ref{appendix:Preliminaries}.\\

{\bf A.7} 
\textit{(Smoothness of $f^{KD}$):} We assume $f_i^{KD}(\bx,\bw)$ is $L_{D_1}$-smooth w.r.t.\ $\bx$, $L_{D_2}$-smooth w.r.t.\ $\bw$ for all $i \in [n]$; as a result it is $L_D$-smooth w.r.t.\ $[\bx,\bw]$ where $L_D=\max\{L_{D_1},L_{D_2}\}$. Furthermore, we assume  $f_i^{KD}(\widetilde{Q}_\bc(\bx),\bw)$ is $L_{DQ_1}$-smooth w.r.t.\ $\bx$, $L_{DQ_2}$-smooth w.r.t.\ $\bc$ and $L_{DQ_3}$-smooth w.r.t.$\bw$ for all $i \in [n]$; as a result it is $L_{DQ}$-smooth w.r.t.\ $[\bx, \bc, \bw]$ where $L_{DQ}=\max\{L_{DQ_1},L_{DQ_2},L_{DQ_3}\}$. This assumption holds as a corollary of Assumptions \textbf{A.1}-\textbf{A.5} (see Appendix~\ref{appendix:Preliminaries} for details).\\

\textbf{Convergence result.} In Theorem~\ref{thm:personalized} we present the convergence result when there is full client participation, i.e. $K=n$, in Section~\ref{sec:proof_thm2} we discuss the modification in convergence result under client sampling. The following result (proved in Section~\ref{sec:proof_thm2}) achieves a rate of $\mathcal{O}\left(\nicefrac{1}{\sqrt{T}}\right)$ for finding a stationary point 
within an error that depends on the data heterogeneity, matching result in \cite{dinh2020personalized}:
\begin{theorem}\label{thm:personalized}
	Under assumptions {\bf A.1-A.7}, consider running Algorithm~\ref{algo:personalized} for $T$ iterations for minimizing \eqref{per1} with $\tau \leq \sqrt{T}$, $\eta^{(i)}_1=\nicefrac{1}{2(\lambda_p(2+L_{D_1}+L_{DQ_1})+(1-\lambda_p)(L+G^{(i)}L_{Q_1}+G^{(i)}_{Q_1}Ll_{Q_1}))}$, $\eta_3 = \nicefrac{1}{4(\lambda_pL_{w}\sqrt{C_L}\sqrt{T})}$, and $\eta^{(i)}_2 = \nicefrac{1}{2(\lambda_p(1+L_{DQ_2})+(1-\lambda_p)(G^{(i)}L_{Q_2}+G^{(i)}_{Q_2}Ll_{Q_2}))}$, where $L_{w}=L_{D_2}+L_{DQ_3}$.
	Let  $\bG_{i}^{t} {:=} [\nabla_{\bx_{i}^{t{}+1}} F_i(\bx_{i}^{t{+}1},\bc_{i}^{t}{,}\bw^{t})^T,\\  {\nabla}_{\bc_{i}^{t{+}1}} F_i(\bx_{i}^{t+1},\bc_{i}^{t+1}{,}\bw^{t})^T, \nabla_{\bw^{t}} F_i(\bx_{i}^{t+1},\bc_{i}^{t+1}{,}\bw^{t})^{T}]^T$. Then
	\begin{align*}
		\frac{1}{T}\sum_{t=0}^{T-1}\frac{1}{n}\sum_{i=1}^{n} \left\|\bG_{i}^{t}\right\|^2 &= \mathcal{O} \bigg(\frac{\tau^2\overline{\kappa}+\overline{\Delta}_F}{\sqrt{T}} + \tau^2\overline{\kappa}\left(\frac{C_1}{T}+\frac{C_2}{T^{\frac{3}{2}}}\right)+ \overline{\kappa}\bigg),
	\end{align*}
for some constants $C_1,C_2$, where $\overline{\Delta}_F=\frac{1}{n}\sum_{i=1}^{n} (L^{(i)}_{\max})^2\left(F_i(\bx^{0}_{i},\bc^{0}_{i},\bw^{0}_{i})-F_i(\bx^T_{i},\bc^T_{i},\bw^{T}_{i})\right)$, $L^{(i)}_{\max} = \max\{\sqrt{\nicefrac{1}{18}},(\frac{\lambda_p}{3}(2+2L_{DQ_2}+L_{DQ}) +(1-\lambda_p)( G^{(i)}L_{Q_2}+G^{(i)}_{Q_2}Ll_{Q_2})), (\frac{\lambda_p}{3}(4+2L_{D_1}+2L_{DQ_1}+L_D+L_{DQ})+(1-\lambda_p)(L+G^{(i)}L_{Q_1}+G^{(i)}_{Q_1}Ll_{Q_1}))\}$, $C_{L}= 1+\frac{\frac{1}{n}\sum_{i=1}^{n}(L^{(i)}_{\max})^2}{(\min_i\{L^{(i)}_{\max}\})^2}$, and $\overline{\kappa}=\frac{1}{n}\sum_{i=1}^n(L^{(i)}_{\max})^2\kappa_i$. 
\end{theorem}

Theorem~\ref{thm:personalized} is proved in Section~\ref{sec:proof_thm2}.

\begin{remark}[Resource and data heterogeneity.] 
	Firstly, our observation from Remark~\ref{remark:thm1-centers} holds here as well. Aggressive quantization has a scaling effect on all the terms through $L^{(i)}_{\max}$. Now the interesting question is: how does having different model structures across clients affect the convergence rate of Theorem \ref{thm:personalized}? Note that in Assumptions {\bf A.3, A.5}, we assume clients have different client-specific gradient bounds; this results in client specific $L^{(i)}_{\max}$, and consequently $\overline{\kappa}$, which couples resource and data heterogeneity across clients. Here we make an important {\em first} observation regarding the coupled effect of data and resource heterogeneity on the convergence rate. Suppose data distributions are fixed across clients (i.e., $\kappa_i$'s are fixed) and we need to choose models for each client in the federated ecosystem. Then, for a fast convergence, for the clients that have local data that is not a representative of the general distribution (large $\kappa_i$), it is critical to choose models with small smoothness parameter (e.g., choosing a less aggressive quantization); whereas, clients with data that is representative of the overall data distribution (small $\kappa_i$) can tolerate having a less smooth model.
\end{remark}
	
	\section{Experiments} \label{sec:experiments}
	In this section, we first compare numerical results for our underlying model quantization scheme (Algorithm \ref{algo:centralized}) in a centralized case against related works \cite{BinaryRelax,bai2018proxquant}. Here, both \cite{BinaryRelax,bai2018proxquant} considers proximal algorithms with $\ell_2$ and $\ell_1$ penalty respectively without optimizing the quantization levels; moreover, both methods are restricted to 1 bit quantization. For a major part of the section, we then compare QuPeD (Algorithm \ref{algo:personalized}) against other personalization schemes \cite{dinh2020personalized,shen2020federated,fallah2020personalized, ozkara2021qupel} for data heterogeneous clients and demonstrate its effectiveness in resource heterogeneous environments. \\

%
%

\subsection{Centralized Training}  
We compare Algorithm \ref{algo:centralized} with \cite{BinaryRelax,bai2018proxquant} for ResNet-20 and ResNet-32 \cite{he2015deep} models trained on CIFAR-10 \cite{cifar10} dataset. 

\begin{table}[h]
	\centering
	\begin{tabular}{lccl} 
		Method	& ResNet-20 & ResNet-32 \\ \midrule
		Full Precision (FP) & $92.05 $ & $92.95$ \\ 
		ProxQuant \cite{bai2018proxquant} (1bit)  & $90.69 $ & $91.55$\\
		BinaryRelax \cite{BinaryRelax} (1bit) & $87.82 $ & $90.65$\\
		Algorithm\ref{algo:centralized} (1bit) & $91.17$ & $92.20$\\ 
		Algorithm\ref{algo:centralized} (2bits) & $91.45$ & $92.47$
	\end{tabular} 
	\caption{Test accuracy (in \%) on CIFAR-10.} 
	\label{tab:Table1}
\end{table}
While both \cite{BinaryRelax,bai2018proxquant} are limited to binary quantization, \cite{bai2018proxquant} can be seen as a specific case of our centralized method where the centers do not get updated. Specifically, previous works did not optimize over centers as we did in \eqref{mainopt}. 
From Table ~\ref{tab:Table1}, we see that updating centers (Algorithm~\ref{algo:centralized}) significantly improves the performance ($0.48\%$ increase in test accuracy). Allowing quantization with 2 bits instead of 1bit for Algorithm \ref{algo:centralized} further increases the test accuracy. Our algorithm allows us to employ any number of bits for quantization. \\

%
%
%
%

{\bf Personalized Training:} 
We consider an image classification task on FEMNIST \cite{caldas2018leaf} and CIFAR-10 \cite{cifar10} datasets. We consider two CNN architectures: {\sf (i)} CNN1 (used in \cite{mcmahan2017communicationefficient}): has 2 convolutional and 3 fully connected layers, {\sf (ii)} CNN2: this is CNN1 with an additional convolutional layer with $32$ filters and $5\times5$ kernel size. For CIFAR-10 we choose a batch size of $25$. For FEMNIST, we choose variable batch sizes to have $60$ iterations for all clients per epoch. We train each algorithm for $250$ epochs on CIFAR-10 and $30$ epochs on FEMNIST.
For quantized training, as standard practice \cite{Bin_survey}, we let the first and last layers of networks to be in full precision.
We use last $50$ epochs on CIFAR-10, and $5$ epochs on FEMNIST for the fine-tuning phase. 


\emph{Data Heterogeneity (DH)}: We consider $n=50$ clients for CIFAR-10 and $n=66$ for FEMNIST. To simulate data heterogeneity on CIFAR-10, similar to \cite{mcmahan2017communicationefficient}, we allow each client to have access to data samples from only 4 randomly chosen classes. Thus, each client has $1000$ training samples and $200$ test samples. On FEMNIST, we use a subset of $198$ writers from the dataset and distribute the data so that each client has access to data samples written by $3$ randomly chosen writers. The number of training samples per client varies between $203$-$336$ and test samples per client varies between $25$-$40$. Test samples are sampled from the same class/writer that training samples are sampled from, in parallel with previous works in heterogeneous FL.


\begin{figure}[t]
	\centering
	\begin{minipage}[b]{0.55\linewidth}
		\captionof{table}{Test accuracy (in \%) for CNN1 model at all clients.\textsuperscript{\ref{footnote:Table FP}}}
		\begin{tabular}{lccl} \toprule 
			Method   & FEMNIST & CIFAR-10\\ \midrule
			FedAvg (FP)  & $94.92 \pm 0.04 $& $61.40 \pm 0.29$\\ 
			Local Training (FP)  & $94.86 \pm 0.93 $& $71.57 \pm 0.28 $\\ 
			Local Training (2 Bits)   & $93.95 \pm 0.23 $& $70.87 \pm 0.15 $\\
			Local Training (1 Bit)   &$93.00 \pm 0.50 $& $69.05 \pm 0.13 $ \\
			\textbf{QuPeD} (FP)  &$\mathbf{97.31} \pm 0.12 $& $\mathbf{75.06} \pm 0.40 $\\
			\textbf{QuPeD} (2 Bits)  &$96.73 \pm 0.27 $& $74.58 \pm 0.44$\\ 
			\textbf{QuPeD} (1 Bit) &$95.15 \pm 0.21 $  & $71.20 \pm 0.33$\\
			QuPeL (2 Bits)  & $96.10 \pm 0.14 $& $73.52 \pm 0.51$\\ 
			QuPeL (1 Bits)  & $94.06 \pm 0.28 $& $71.01 \pm 0.32$\\ 
			pFedMe (FP) \cite{dinh2020personalized} & $96.60 \pm 0.37 $& $73.66 \pm 0.65$ \\
			Per-FedAvg (FP) \cite{fallah2020personalized}  &$97.16 \pm 0.21 $& $74.15 \pm 0.41$ \\
			Federated ML (FP) \cite{shen2020federated} & $96.32 \pm 0.32 $ & $74.34 \pm 0.30$ 
		\end{tabular}
		\label{tab:Table CNN1}
	\end{minipage}
	\begin{minipage}[b]{0.4\linewidth}
		\centering
		\includegraphics[scale=0.32]{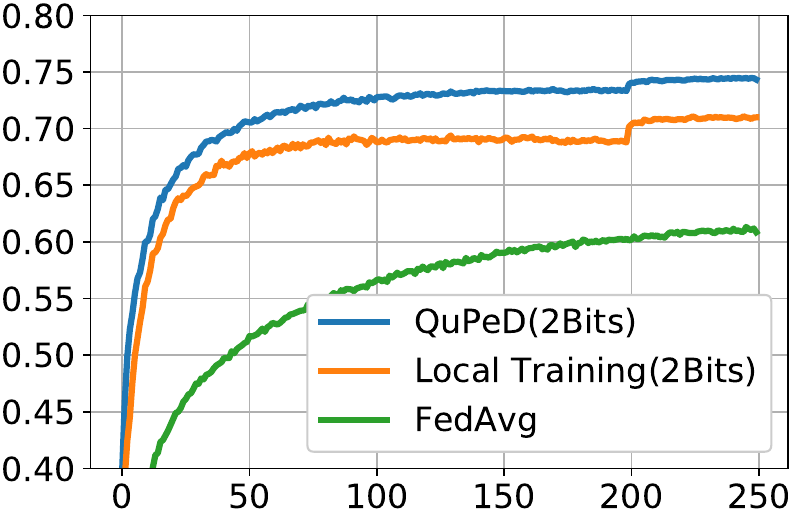}
		\includegraphics[scale=0.32]{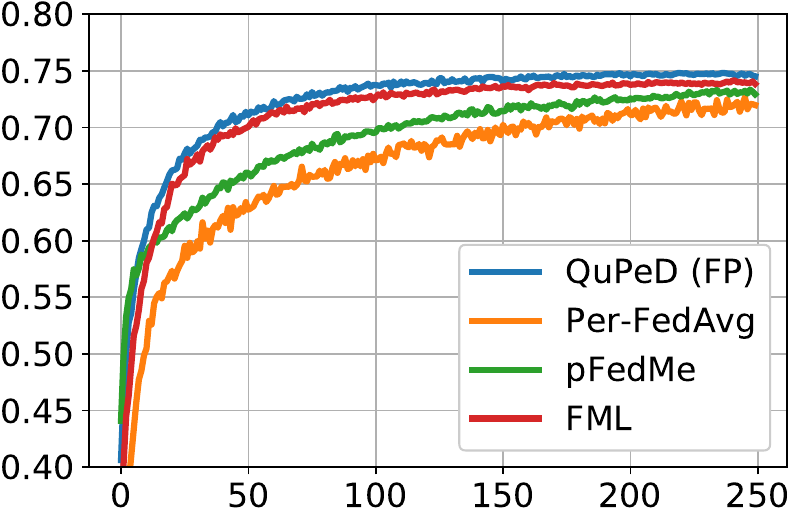}
		\caption{Test Acc. vs epoch (CIFAR-10)}
		\label{fig:baselines}
	\end{minipage}
\end{figure}

\footnotetext{Here QuPeD (FP) corresponds to changing alternating proximal gradient updates with SGD update on model parameters in Algorithm~\ref{algo:personalized}, Local Training (FP) corresponds to SGD updates without communication.\label{footnote:Table FP}}

\setlength\intextsep{5pt}
\begin{figure}[h]
	\centering
	\captionof{table}{Test accuracy (in \%) for CNN1 model at all clients, with client sampling.}
	\begin{tabular}{lccl} \toprule 
		Method   & MNIST ($\frac{K}{n}=0.1$) & FEMNIST($\frac{K}{n}=\frac{1}{3}$) \\ \midrule
		FedAvg (FP) &  $92.87 \pm 0.05 $&  $91.30 \pm 0.43$\\ 
		\textbf{QuPeD} (FP)  &$\mathbf{98.17} \pm 0.32$  &$\mathbf{94.93} \pm 0.25$\\
		\textbf{QuPeD} (2 Bits) &$98.01 \pm 0.15$ &$ 94.56  \pm 0.18$\\ 
		\textbf{QuPeD} (1 Bit)  &$97.58 \pm 0.23$  &$ 92.52  \pm 0.64$ \\
		pFedMe (FP) \cite{dinh2020personalized} & $97.79 \pm 0.03$ & $ 93.70  \pm 0.39$\\
		Per-FedAvg (FP) \cite{fallah2020personalized} &$95.80 \pm 0.29$  & $92.10  \pm 0.22$\\
		Federated ML (FP) \cite{shen2020federated} & $98.03 \pm 0.31 $ &$92.73 \pm 0.36 $ 
	\end{tabular}
	\label{tab:Table client sampling}
\end{figure}

\emph{Resource Heterogeneity (RH)}: To simulate resource heterogeneity for QuPeD, we consider 4 settings: {\sf (i)} half of the clients have CNN1 in full precision (FP) and the other half CNN2 in FP, {\sf (ii)} half of the clients have CNN1 in 2 bits and the other half CNN2 in FP, {\sf (iii)} half of the clients have CNN1 in 2 bits and the other in FP, {\sf (iv)} half of the clients have CNN1 in 2 bits and the other half CNN2 in 2 bits.

\begin{table}[h]
	\centering
	\caption{Test accuracy (in \%) on FEMNIST and CIFAR-10 for heterogeneous resource distribution among clients.}
	\begin{tabular}{ccccc} \toprule 
		Resource Heterogeneity& \multicolumn{2}{c}{FEMNIST} & \multicolumn{2}{c}{CIFAR-10}  \\ \midrule
		& Local Training & QuPeD & Local Training & QuPeD \\ \cmidrule(r){2-3}\cmidrule(l){4-5}
		CNN1(FP) + CNN2(FP)  & $93.41 \pm 0.82$&  $97.44 \pm 0.14$ & $72.81 \pm 0.03$&  $75.50 \pm 0.25$\\
		CNN1(2 Bits)+CNN2(FP) & $92.70 \pm 1.09$ & $97.01 \pm 0.05$ & $72.42 \pm 0.17$ & $75.08 \pm 0.18$  \\
		CNN1(2 Bits)+CNN1(FP) & $93.56 \pm 0.38$& $96.96 \pm 0.15$ & $71.23 \pm 0.08$ & $74.84 \pm 0.30$  \\ 
		CNN1(2 Bits)+CNN2(2 Bits) & $91.11 \pm 0.23$&$96.64 \pm 0.31$ & $72.15 \pm 0.47$&  $74.64 \pm 0.27$
	\end{tabular}
	\label{tab:Table Heterogenous}
\end{table}

\emph{Results (DH):}  We compare QuPeD against FedAvg \cite{mcmahan2017communicationefficient}, local training of clients (without any collaboration), and personalized FL methods: pFedMe \cite{dinh2020personalized}, Per-FedAvg \cite{fallah2020personalized}, Federated Mutual Learning \cite{shen2020federated}, and  QuPeL \cite{ozkara2021qupel}. For all methods, if applicable, we set $\tau = 10$ local iterations, use learning rate decay $0.99$ and use weight decay of $10^{-4}$; we fine tune the initial learning rate for each method independently, see Appendix~\ref{appendix:experiments} for details. The results are provided in Table~\ref{tab:Table CNN1} with full client participation ($K=n$), plotted in Figure \ref{fig:baselines} for CIFAR-10, and in Table~\ref{tab:Table client sampling} with client sampling  where we state average results over $3$ runs; all clients train CNN1 (see Appendix~\ref{appendix:experiments} for CNN2) and quantization values are indicated in parenthesis. Thus, we only consider model personalization for data heterogeneity. In Table~\ref{tab:Table CNN1}, we observe that full precision QuPeD consistently outperforms all other methods for both datasets. Furthermore, we observe QuPeD with 2-bit quantization is the second best performing method on CIFAR-10 (after QuPeD (FP)) and third best performing method on FEMNIST despite the loss due to quantization. Hence, QuPeD is highly effective for quantized training in data heterogeneous settings. Since QuPeD outperforms QuPeL, we can also (empirically) claim that considering KD loss to encourage collaboration is superior to $\ell_2$ distance loss. Lastly, we observe from Table~\ref{tab:Table client sampling} that QuPeD continues to outperform other methods under client sampling.

\emph{Results (DH+RH)}: We now discuss personalized FL setting with both data and resource heterogeneity across clients. Note that since FedAvg, pFedMe, and Per-FedAvg cannot work in settings where clients have different model dimensions, we only provide comparisons of QuPeD with local training (no collaboration) to demonstrate its effectiveness. The results are given in Table~\ref{tab:Table Heterogenous}. We observe that QuPeD (collaborative training) significantly outperforms local training in all cases (about $3.5\%$ or higher on FEMNIST and $2.5\%$ or higher on CIFAR-10) and works remarkably well even in cases where clients have quantized models without any significant loss in performance. 


	\section{Proof of Theorem \ref{thm:centralized}} \label{sec:proof_thm1}
	This proof consists of two parts. First we show the sufficient decrease property by sequentially using Lipschitz properties for each update step in Algorithm~\ref{algo:centralized}. For each variable $\bx$ and $\bc$ we find the decrease inequalities and then combine them to obtain an overall sufficient decrease. Then we bound the norm of the gradient using optimality conditions of the proximal updates in Algorithm~\ref{algo:centralized}. Using sufficient decrease and bound on the gradient we arrive at the result. We leave some of the derivations and proof of the claims to Appendix~\ref{appendix:proof of theorem 1}.

\textbf{Alternating updates.} Remember that for the Algorithm~\ref{algo:centralized} we have the following alternating updates:
\begin{align*}
	\bx^{t+1} &= \text{prox}_{\eta_1\lambda R_{\bc^{t}}}(\bx^{t} - \eta_1 \nabla f(\bx^{t})-\eta_1 \nabla_{\bx^{t}} f(\widetilde{Q}_{\bc^{t}}(\bx^{t})) )\\
	\bc^{t+1} &= \text{prox}_{\eta_2\lambda R_{\bx^{t+1}}}(\bc^{t} - \eta_2 \nabla_{\bc^{t}} f(\widetilde{Q}_{\bc^{t}}(\bx^{t+1})))
\end{align*}
These translate to following optimization problems for $\bx$ and $\bc$ respectively (see end of the section for derivation):
\begin{align}
	\bx^{t+1} &=\underset{\bx \in \mathbb{R}^{d}}{\arg \min }\left\{ \left\langle \bx-\bx^{t}, \nabla_{\bx^{t}} f\left(\bx^{t}\right) \right\rangle+\left\langle \bx-\bx^{t}, \nabla_{\bx^{t}} f(\widetilde{Q}_{\bc^{t}}(\bx^{t}))\right\rangle+\frac{1}{2 \eta_1}\left\|\bx-\bx^{t}\right\|_{2}^{2}+\lambda R(\bx,\bc^{t})\right\} \label{thm1:optimization prob for x}\\
	\bc^{t+1}&= \underset{\bc \in \mathbb{R}^{m}}{\arg \min }\left\{\left\langle \bc-\bc^{t}, \nabla_{\bc^{t}} f(\widetilde{Q}_{\bc^{t}}(\bx^{t+1}))\right\rangle+\frac{1}{2 \eta_2}\left\|\bc-\bc^{t}\right\|_{2}^{2}+\lambda R(\bx^{t+1},\bc)\right\} \label{thm1:optimization prob for c}
\end{align}
\subsection{Sufficient Decrease}
This section is divided into two, first we will show sufficient decrease property with respect to $\bx$, then we will show sufficient decrease property with respect to $\bc$. 
\subsubsection{Sufficient Decrease Due to $\bx$}
\begin{claim}\label{claim: lqxsmooth}
	$f(\bx)+ f(\widetilde{Q}_\bc(\bx))$ is $(L+GL_{Q_1}+G_{Q_1}LL_{Q_1})$-smooth with respect to $\bx$.
\end{claim}
Using Claim \ref{claim: lqxsmooth} we have,
\begin{align}
	F_\lambda(\bx^{t+1},\bc^t)\ +\ & (\frac{1}{2\eta_1}-\frac{L+GL_{Q_1}+G_{Q_1}LL_{Q_1}}{2})\|\bx^{t+1}-\bx^t\|^2 \notag \\ 
	&= f(\bx^{t+1})+f(\widetilde{Q}_{\bc^t}(\bx^{t+1}))+\lambda R(\bx^{t+1},\bc^t)\\
	& \hspace{1cm}+ (\frac{1}{2\eta_1}-\frac{L+GL_{Q_1}+G_{Q_1}LL_{Q_1}}{2})\|\bx^{t+1}-\bx^t\|^2 \notag \\ 
	&\leq f(\bx^t)+f(\widetilde{Q}_{\bc^t}(\bx^t))+\lambda R(\bx^{t+1},\bc^t)+\left\langle \nabla f(\bx^t), \bx^{t+1}-\bx^t\right\rangle  \notag \\
	&\hspace{1cm} + \left\langle \nabla_{\bx^t} f(\widetilde{Q}_{\bc^t}(\bx^t)), \bx^{t+1}-\bx^t\right\rangle + \frac{1}{2\eta_1}\|\bx^{t+1}-\bx^t\|^2\label{thm1:first-part-interim1}
\end{align}
\begin{claim}\label{claim:quantization lower bound 1}
	Let 
	\begin{align*}
		A(\bx^{t+1}) &:= \lambda R(\bx^{t+1},\bc^{t})+\left\langle \nabla f(\bx^{t}), \bx^{t+1}-\bx^{t}\right\rangle 
		+ \left\langle \nabla_{\bx^{t}} f(\widetilde{Q}_{\bc^{t}}(\bx^{t})), \bx^{t+1}-\bx^{t}\right\rangle \notag \\
		& \hspace{1cm}+ \frac{1}{2\eta_1}\|\bx^{t+1}-\bx^{t}\|^2 \notag \\
		A(\bx^{t}) &:= \lambda R(\bx^{t},\bc^{t}).
	\end{align*} 
	Then $A(\bx^{t+1})\leq A(\bx^{t})$.
\end{claim}
Now we use Claim~\ref{claim:quantization lower bound 1} and get,
\begin{align*} 
	& f\left(\bx^{t}\right)+\left\langle \bx^{t+1}-\bx^{t}, \nabla f\left(\bx^{t}\right)\right\rangle + \frac{1}{2 \eta_1}\left\|\bx^{t+1}-\bx^{t}\right\|_{2}^{2} +\lambda R\left(\bx^{t+1}, \bc^{t}\right) + f(\widetilde{Q}_{\bc^{t}}(\bx^{t})) \\
	& \hspace{2cm}+\left\langle \nabla_{\bx^{t}} f(\widetilde{Q}_{\bc^{t}}(\bx^{t})), \bx^{t+1}-\bx^{t} \right\rangle 
	\leq f\left(\bx^{t}\right)+f(\widetilde{Q}_{\bc^{t}}(\bx^{t}))+\lambda R\left(\bx^{t},\bc^{t}\right) \\
	& \hspace{7cm}= F_{\lambda}\left(\bx^{t},\bc^{t}\right).
\end{align*}
Using \eqref{thm1:first-part-interim1} we have,
\begin{align*}
	F_\lambda(\bx^{t+1},\bc^t)+(\frac{1}{2\eta_1}-\frac{L+GL_{Q_1}+G_{Q_1}LL_{Q_1}}{2})\|\bx^{t+1}-\bx^t\|^2 \leq  F_{\lambda}\left(\bx^{t},\bc^{t}\right).
\end{align*}
Now, we choose $\eta_1 = \frac{1}{2(L+GL_{Q_1}+G_{Q_1}Ll_{Q_1})}$ and obtain the decrease property for $\bx$:
\begin{align} \label{thm1:sufficient decrease 1}
	F_{\lambda}\left(\bx^{t+1},\bc^{t}\right) + \frac{L+GL_{Q_1}+G_{Q_1}Ll_{Q_1}}{2}\|\bx^{t+1}-\bx^{t}\|^2 \leq	F_{\lambda}\left(\bx^{t},\bc^{t}\right). 
\end{align}
\subsubsection{Sufficient Decrease Due to $\bc$}
From Claim~\ref{claim: lqcclaim} we have $f(\widetilde{Q}_\bc(\bx))$ is $(GL_{Q_2}+G_{Q_2}LL_{Q_2})$-smooth with respect to $\bc$.
Using Claim~\ref{claim: lqcclaim},
\begin{align}
	&F_\lambda(\bx^{t+1},\bc^{t+1}) +(\frac{1}{2\eta_2}-\frac{GL_{Q_2}+G_{Q_2}LL_{Q_2}}{2})\|\bc^{t+1}-\bc^t\|^2 \notag \\
	&\quad= f(\bx^{t+1})+f(\widetilde{Q}_{\bc^{t+1}}(\bx^{t+1}))+\lambda R(\bx^{t+1},\bc^{t+1})  + (\frac{1}{2\eta_2}-\frac{GL_{Q_2}+G_{Q_2}LL_{Q_2}}{2})\|\bc^{t+1}-\bc^t\|^2 \notag \\ 
	&\quad\leq f(\bx^{t+1})+f(\widetilde{Q}_{\bc^t}(\bx^{t+1}))+\lambda R(\bx^{t+1},\bc^{t+1}) + \left\langle \nabla_{\bc^t} f(\widetilde{Q}_{\bc^t}(\bx^{t+1})), \bc^{t+1}-\bc^t\right\rangle \notag \\
	& \hspace{2cm}+ \frac{1}{2\eta_2}\|\bc^{t+1}-\bc^t\|^2\label{thm1:first-part-interim2}
\end{align}
Now we state the counterpart of Claim~\ref{claim:quantization lower bound 1} for $\bc$.
\begin{claim}\label{claim:quantization lower bound 2}
	Let 
	\begin{align*}
		B(\bc^{t+1}) &:= \lambda R(\bx^{t+1},\bc^{t+1}) + \left\langle \nabla_{\bc^{t}} f(\widetilde{Q}_{\bc^{t}}(\bx^{t+1})), \bc^{t+1}-\bc^{t}\right\rangle \notag + \frac{1}{2\eta_1}\|\bc^{t+1}-\bc^{t}\|^2 \notag \\
		B(\bc^{t}) &:= \lambda R(\bx^{t+1},\bc^{t}).
	\end{align*} 
	Then $B(\bc^{t+1})\leq B(\bc^{t})$.
\end{claim}
Now using Claim~\ref{claim:quantization lower bound 2},
\begin{align} \nonumber
	f(\bx^{t+1}) + \eta_2\left\|\bc^{t+1}-\bc^{t}\right\|_{2}^{2}+\lambda R(\bx^{t+1}, \bc^{t+1}) + f(\widetilde{Q}_{\bc^{t}}(\bx^{t+1}))+\left\langle \bc^{t+1}-\bc^{t}, \nabla_{\bc^{t}} f(\widetilde{Q}_{\bc^{t}}(\bx^{t+1}))\right\rangle \nonumber \\ \nonumber
	\leq f\left(\bx^{t+1}\right)+f(\widetilde{Q}_{\bc^{t}}(\bx^{t+1}))+\lambda R\left(\bx^{t+1},\bc^{t}\right) = F_{\lambda}\left(\bx^{t+1},\bc^{t}\right)
\end{align}
Setting $\eta_2 = \frac{1}{2(GL_{Q_2}+G_{Q_2}Ll_{Q_2})}$ and using the bound in \eqref{thm1:first-part-interim2}, we obtain the sufficient decrease for $\bc$:
\begin{align} \label{thm1:sufficient decrease 2}
	F_{\lambda}\left(\bx^{t+1},\bc^{t+1}\right) + \frac{GL_{Q_2}+G_{Q_2}Ll_{Q_2}}{2}\|\bc^{t+1}-\bc^{t}\|^2 \leq
	F_{\lambda}\left(\bx^{t+1},\bc^{t}\right)
\end{align}
\subsubsection{Overall Decrease} 
Summing the bounds in \eqref{thm1:sufficient decrease 1} and \eqref{thm1:sufficient decrease 2}, we have the overall decrease property:
\begin{align} \label{thm1:overall decrease}
	&F_{\lambda}(\bx^{t+1},\bc^{t+1}) + \frac{L+GL_{Q_1}+G_{Q_1}Ll_{Q_1}}{2}\|\bx^{t+1}-\bx^{t}\|^2 +\frac{GL_{Q_2}+G_{Q_2}Ll_{Q_2}}{2}\|\bc^{t+1}-\bc^{t}\|^2 \notag \\
	&\hspace{9cm} \leq 
	F_{\lambda}\left(\bx^{t},\bc^{t}\right) 
\end{align}
Let us define $L_{\min}=\min\{L+GL_{Q_1}+G_{Q_1}Ll_{Q_1}, GL_{Q_2}+G_{Q_2}Ll_{Q_2}\}$, and  $\bz^t = (\bx^{t},\bc^{t})$. Then from \eqref{thm1:overall decrease}:
\begin{align*}
	& F_{\lambda}\left(\bz^{t+1}\right) + \frac{L_{\min}}{2}(\|\bz^{t+1}-\bz^t\|^2) = F_{\lambda}\left(\bx^{t+1},\bc^{t+1}\right) + \frac{L_{\min}}{2}(\|\bx^{t+1}-\bx^{t}\|^2+\|\bc^{t+1}-\bc^{t}\|^2) \\
	& \hspace{9cm} \leq
	F_{\lambda}\left(\bx^{t},\bc^{t}\right) = 	F_{\lambda}(\bz^t)
\end{align*}

Telescoping the above bound for $t=0, \ldots, T-1,$ and dividing by $T$:
\begin{align} \label{proximity1}
	\frac{1}{T}\sum_{t=0}^{T-1}(\left\|\bz^{t+1}-\bz^{t}\right\|_{2}^{2})\leq \frac{2\left(F_{\lambda}\left(\bz^{0}\right)-F_{\lambda}\left(\bz^{T}\right)\right)}{L_{\min}T} 
\end{align}

\subsection{Bound on the Gradient}


We now find the first order stationarity guarantee. Taking the derivative of \eqref{thm1:optimization prob for x} with respect to $\bx$ at $\bx=\bx^{t+1}$ and setting it to 0 gives us the first order optimality condition:
\begin{align} \label{thm1:foc1}
	\nabla f(\bx^{t})+\nabla_{\bx^{t}} f(\widetilde{Q}_{\bc^{t}}(\bx^{t}))+\frac{1}{\eta_1}\left(\bx^{t+1}-\bx^{t}\right)+\lambda \nabla_{\bx^{t+1}} R\left(\bx^{t+1},\bc^{t}\right)=0
\end{align}
Combining the above equality and Claim~\ref{claim: lqxsmooth}:

\begin{align*}
	\left\|\nabla_{\bx^{t+1}} F_{\lambda}(\bx^{t+1},\bc^{t})\right\|_{2} &= \left\|\nabla f(\bx^{t+1})+\nabla_{\bx^{t+1}} f(\widetilde{Q}_{\bc^{t}}(\bx^{t+1}))+\lambda \nabla_{\bx^{t+1}} R(\bx^{t+1},\bc^{t})\right\|_{2} \\
	&\stackrel{\text{(a)}}{=} \Big\|\frac{1}{\eta}\left(\bx^{t}-\bx^{t+1}\right)+\nabla f(\bx^{t+1})-\nabla f(\bx^{t}) + \nabla_{\bx^{t+1}} f(\widetilde{Q}_{\bc^{t}}(\bx^{t+1}))\\
	& \hspace{2cm} - \nabla_{\bx^{t}} f(\widetilde{Q}_{\bc^{t}}(\bx^{t}))\Big\|_{2} \\ 
	&\leq (\frac{1}{\eta_1}+L+GL_{Q_1}+G_{Q_1}Ll_{Q_1})\left\|\bx^{t+1}-\bx^{t}\right\|_{2} \\
	&\stackrel{\text{(b)}}{=} 3(L+GL_{Q_1}+G_{Q_1}Ll_{Q_1})\left\|\bx^{t+1}-\bx^{t}\right\|_{2} \\
	& \leq 3 (L+GL_{Q_1}+G_{Q_1}Ll_{Q_1})\left\|\bz_{t+1}-\bz_{t}\right\|_{2}
\end{align*}

where (a) is from \eqref{thm1:foc1} and (b) is because we chose $\eta_1 = \frac{1}{2(L+GL_{Q_1}+G_{Q_1}Ll_{Q_1})}$. First order optimality condition in \eqref{thm1:optimization prob for c} for $\bc^{t+1}$ gives:
\begin{align*}
	\nabla_{\bc^{t+1}} f(\widetilde{Q}_{\bc^{t+1}}(\bx^{t+1}))+\frac{1}{\eta_2}(\bc^{t+1}-\bc^{t})+\lambda \nabla_{\bc^{t+1}} R(\bx^{t+1},\bc^{t+1})=0
\end{align*}

Combining the above equality and Claim~\ref{claim: lqcclaim}:

\begin{align*}
	\left\|\nabla_{\bc^{t+1}} F_{\lambda}(\bx^{t+1},\bc^{t+1})\right\|_{2} &=\left\|\nabla_{\bc^{t+1}} f(\widetilde{Q}_{\bc^{t+1}}(\bx^{t+1}))+\lambda \nabla_{\bc^{t+1}} R(\bx^{t+1},\bc^{t+1})\right\|_{2} \\
	&= \left\|\frac{1}{\eta_2}\left(\bc^{t}-\bc^{t+1}\right)+ \nabla_{\bc^{t+1}} f(\widetilde{Q}_{\bc^{t+1}}(\bx^{t+1})) - \nabla_{\bc^{t}} f(\widetilde{Q}_{\bc^{t}}(\bx^{t+1}))\right\|_{2} \\ 
	&\leq (\frac{1}{\eta_2}+GL_{Q_2}+G_{Q_2}Ll_{Q_2})\left\|\bc^{t+1}-\bc^{t}\right\|_{2} \\
	&\stackrel{\text{(a)}}{=}3 (GL_{Q_2}+G_{Q_2}Ll_{Q_2})\left\|\bc^{t+1}-\bc^{t}\right\|_{2} \\
	& \leq 3 (GL_{Q_2}+G_{Q_2}Ll_{Q_2})\left\|\bz^{t+1}-\bz^{t}\right\|_{2}
\end{align*}

where (a) is because we set $\eta_2 = \frac{1}{2(GL_{Q_2}+G_{Q_2}Ll_{Q_2})}$. Then:
\begin{align*}
	\left\|[\nabla_{\bx^{t+1}} F_{\lambda}(\bx^{t+1},\bc^{t})^T,\nabla_{\bc^{t+1}} F_{\lambda}(\bx^{t+1},\bc^{t+1})^T]^T\right\|^2_{2} &=  \left\|\nabla_\bx F_{\lambda}\left(\bx^{t+1},\bc^{t}\right)\right\|^2_{2} \\
	&\hspace{1cm}+ \left\|\nabla_\bc F_{\lambda}\left(\bx^{t+1},\bc^{t+1}\right)\right\|^2_{2} \\ 
	& \hspace{-4cm}\leq 3^2(G(L_{Q_1}+L_{Q_2})+L(1+G_{Q_1}l_{Q_1}+G_{Q_2}l_{Q_2}))^2\|\bz^{t+1}-\bz^{t}\|^2
\end{align*}

Letting $L_{\max} = \max\{L+GL_{Q_1}+G_{Q_1}Ll_{Q_1},GL_{Q_2}+G_{Q_2}Ll_{Q_2}\}$ we have:

\begin{align*}
	\left\|[\nabla_{\bx^{t+1}} F_{\lambda}(\bx^{t+1},\bc^{t})^T,\nabla_{\bc^{t+1}} F_{\lambda}(\bx^{t+1},\bc^{t+1})^T]^T\right\|_{2}^2 \leq 9L_{\max}^2\left\|\bz^{t+1}-\bz^{t}\right\|^2_{2}
\end{align*}

Summing over all time points and dividing by $T$:

\begin{align*}
	\frac{1}{T}\sum_{t=0}^{T-1}\left\|[\nabla_{\bx^{t+1}} F_{\lambda}(\bx^{t+1},\bc^{t})^T,\nabla_{\bc^{t+1}} F_{\lambda}(\bx^{t+1},\bc^{t+1})^T]^T\right\|^2_{2} &\leq \frac{1}{T}\sum_{t=0}^{T-1}9L_{\max}^2(\left\|\bz^{t+1}-\bz^{t}\right\|_{2}) \\ & \leq \frac{18 L_{\max}^2\left(F_{\lambda}\left(\bz^{0}\right)-F_{\lambda}(\bz^{T})\right)}{L_{\min}T},
\end{align*}
where in the last inequality we use \eqref{proximity1}. This concludes the proof of Theorem~\ref{thm:centralized}.
	
	\section{Proof of Theorem \ref{thm:personalized}} \label{sec:proof_thm2}
		In this part, different than Section~\ref{sec:centralized}, we have an additional update due to local iterations. The key is to integrate the local iterations into our alternating update scheme. To do this, we utilize Assumptions \textbf{A.6} and \textbf{A.7}. This proof consists of two parts. First, we show the sufficient decrease property by sequentially using and combining Lipschitz properties for each update step in Algorithm~\ref{algo:personalized}. Then, we bound the norm of the gradient using optimality conditions of the proximal updates in Algorithm~\ref{algo:personalized}. Then, by combining the sufficient decrease results and bounds on partial gradients we will derive our result. We defer proofs of the claims and some derivation details to Appendix~\ref{appendix:proof of theorem 2}. In this analysis we take $\bw^t = \frac{1}{n}\sum_{i=1}^n \bw^t_{i}$, so that $\bw^t$ is defined for every time point.

\textbf{Alternating updates.} Let us first restate the alternating updates for $\bx_i$ and $\bc_i$:
\begin{align*}
	\bx^{t+1}_{i} &= \text{prox}_{\eta_1 \lambda R_{\bc^t_{i}}}\Big(\bx^t_{i}-(1-\lambda_p)\eta_1 \nabla f_i(\bx^t_{i})-(1-\lambda_p) \eta_1 \nabla_{\bx^t_{i}} f_i(\widetilde{Q}_{\bc^t_{i}}(\bx^t_{i}))\\
	& \hspace{2cm}-\eta_1 \lambda_p\nabla_{\bx_{i}^{t}} f^{KD}_i(\bx_{i}^{t}, \bw_{i}^{t}) 
	-\eta_1 \lambda_p\nabla_{\bx_{i}^{t}} f^{KD}_i(\widetilde{Q}_{\bc^t_{i}}(\bx^t_{i}), \bw_{i}^{t})\Big) \\
	\bc^{t+1}_{i} &= \text{prox}_{\eta_2 \lambda R_{\bx^{t+1}_{i}}}\left(\bc^t_{i}-(1-\lambda_p)\eta_2 \nabla_{\bc^t_{i}} f_i(\widetilde{Q}_{\bc^t_{i}}(\bx^{t+1}_{i}))-\eta_2 \lambda_p\nabla_{\bc_{i}^{t}} f^{KD}_i(\widetilde{Q}_{\bc^t_{i}}(\bx^{t+1}_{i}), \bw_{i}^{t})\right)
\end{align*}
The alternating updates are equivalent to solving the following two optimization problems.
\begin{align}
	\bx^{t+1}_{i}&= \underset{\bx \in \mathbb{R}^{d}}{\arg \min }\Big\{(1-\lambda_p)\left\langle \bx-\bx^t_{i}, \nabla f_i\left(\bx^t_{i}\right)\right\rangle+(1-\lambda_p)\left\langle \bx-\bx^t_{i}, \nabla_{\bx^t_{i}} f_i(\widetilde{Q}_{\bc^t_{i}}(\bx^t_{i}))\right\rangle  \notag \\
	& \hspace{2cm}+\left\langle \bx-\bx^t_{i}, \lambda_p\nabla_{\bx_{i}^{t}} f^{KD}_i(\bx_{i}^{t}, \bw_{i}^{t})\right\rangle  + \left\langle \bx-\bx^t_{i}, \lambda_p\nabla_{\bx_{i}^{t}} f^{KD}_i(\widetilde{Q}_{\bc^t_{i}}(\bx^t_{i}), \bw_{i}^{t})\right\rangle \notag \\
	&\hspace{2cm} +\frac{1}{2 \eta_1}\left\|\bx-\bx^t_{i}\right\|_{2}^{2}+\lambda R(\bx,\bc^t_{i})\Big\} \label{thm2:lower-bounding-interim1}\\
	\bc^{t+1}_{i}&= \underset{\bc \in \mathbb{R}^{m}}{\arg \min }\Big\{\left\langle \bc-\bc^t_{i}, (1-\lambda_p)\nabla_{\bc^t_{i}} f_i(\widetilde{Q}_{\bc^t_{i}}(\bx^{t+1}_{i}))\right\rangle + \left\langle \bc-\bc^t_{i}, \lambda_p\nabla_{\bc_{i}^{t}} f^{KD}_i(\widetilde{Q}_{\bc^t_{i}}(\bx^{t+1}_{i}), \bw_{i}^{t})\right\rangle \notag \\
	& \hspace{2cm}+\frac{1}{2 \eta_2}\left\|\bc-\bc^t_{i}\right\|_{2}^{2}+\lambda R(\bx^{t+1}_{i},\bc)\Big\} \label{thm2:lower-bounding-interim2}
\end{align}
Note that the update on $\bw^t$ from Algorithm~\ref{algo:personalized} can be written as:
\begin{align*}
	\bw^{t+1} = \bw^{t} - \eta_3 \bg^{t}, \quad \text{ where }\quad \bg^{t} =\frac{1}{n} \sum_{i=1}^n \nabla_{\bw^{t}_{i}} F_i(\bx^{t+1}_{i},\bc^{t+1}_{i},\bw^{t}_{i}).
\end{align*}
In the convergence analysis we require smoothness of the local functions $F_i$ w.r.t.\ the global parameter $\bw$. Recall the definition of $F_i(\bx_i,\bc_i,\bw) = (1-\lambda_p)\left(f_i(\bx_i)+f_i(\widetilde{Q}_{\bc_i}(\bx_i))\right)+\lambda R(\bx_i,\bc_i) + \lambda_p \left(f^{KD}_i(\bx_i,\bw)+f^{KD}_i(\widetilde{Q}_{\bc_i}(\bx_i),\bw)\right)$ from \eqref{qpfl}. It follows that from Assumption {\bf A.7} that $F_i$ is $(\lambda_p(L_{D_2}+L_{DQ_3}))$-smooth with respect to $\bw$:
Now let us move on with the proof. 
\subsection{Sufficient Decrease} \label{thm2:subsection sufficient decrease}
We will divide this part into three and obtain sufficient decrease properties for each variable: $\bx_i,\bc_i,\bw$.
\subsubsection{Sufficient Decrease Due to $\bx_i$}
We begin with a useful claim.
\begin{claim} \label{thm2:claim lqxsmoothness}
	$(1-\lambda_p)(f_i(\bx)+ f_i(\widetilde{Q}_\bc(\bx)))+\lambda_p (f^{KD}_i(\bx,\bw)+f^{KD}_i(\widetilde{Q}_\bc(\bx),\bw))$ is $(\lambda_p(L_{D_1}+L_{DQ_1})+(1-\lambda_p)(L+G^{(i)}L_{Q_1}+G^{(i)}_{Q_1}Ll_{Q_1}))$-smooth with respect to $\bx$.
\end{claim}
From Claim~\ref{thm2:claim lqxsmoothness} ,after some algebra (deferred to Appendix~\ref{appendix:proof of theorem 2}), we have:
\begin{align}
	&F_i(\bx^{t+1}_{i},\bc^t_{i},\bw^{t}) +(\frac{1}{2\eta_1}-\frac{\lambda_p(L_{D_1}+L_{DQ_1})+(1-\lambda_p)(L+G^{(i)}L_{Q_1}+G^{(i)}_{Q_1}Ll_{Q_1})}{2})\|\bx^{t+1}_{i}-\bx^t_{i}\|^2 \nonumber \\	
	& \hspace{0.5cm} \leq (1-\lambda_p)\left(f_i(\bx^t_{i})+f_i(\widetilde{Q}_{\bc^t_{i}}(\bx^t_{i}))\right)+\lambda_p\left(f^{KD}_i (\bx^{t}_{i},\bw^{t})+f^{KD}_i (\widetilde{Q}_{\bc^t_{i}}(\bx^{t}_{i}),\bw^{t})\right) \notag \\ 
	&\hspace{1cm} +(1-\lambda_p)\Big\langle \nabla f_i(\bx^t_{i}), \bx^{t+1}_{i}-\bx^t_{i}\Big\rangle +(1-\lambda_p)\left\langle \nabla_{\bx^t_{i}} f_i(\widetilde{Q}_{\bc^t_{i}}(\bx^t_{i})), \bx^{t+1}_{i}-\bx^t_{i}\right\rangle \nonumber \\
	&\hspace{1cm} +\lambda_p\left\langle \nabla_{\bx^t_i} f^{KD}_i(\bx^t_{i},\bw^t_{i}), \bx^{t+1}_{i}-\bx^t_{i}\right\rangle + \lambda_p\left\langle \nabla_{\bx^t_i} f^{KD}_i((\widetilde{Q}_{\bc^t_{i}}(\bx^t_{i}),\bw^t_{i}), \bx^{t+1}_{i}-\bx^t_{i}\right\rangle  \nonumber \\
	& \hspace{1cm}  + \frac{\lambda_p}{2}\|\nabla_{\bx^t_i} f^{KD}_i(\widetilde{Q}_{\bc^t_{i}}(\bx^t_{i}),\bw^t)-\nabla_{\bx^t_i} f^{KD}_i(\widetilde{Q}_{\bc^t_{i}}(\bx^t_{i}),\bw^t_{i})\|^2 + \lambda_p\|\bx^{t+1}_{i}-\bx^t_{i}\|^2 \notag  \\
	& \hspace{1cm} + \frac{\lambda_p}{2}\|\nabla_{\bx^t_i} f^{KD}_i(\bx^t_{i},\bw^t)-\nabla_{\bx^t_i} f^{KD}_i(\bx^t_{i},\bw^t_{i})\|^2 + \frac{1}{2\eta_1}\|\bx^{t+1}_{i}-\bx^t_{i}\|^2 +\lambda R(\bx^{t+1}_{i},\bc^t_{i}).
	\label{thm2:first-part-interim1}
\end{align}
Where we used: 
\begin{align}	
	&\left\langle\lambda_p(\nabla_{\bx^t_i} f^{KD}_i(\bx^t_{i},\bw^t)-\nabla_{\bx^t_i} f^{KD}_i(\bx^t_{i},\bw^t_{i})), \bx^{t+1}_{i}-\bx^t_{i}\right\rangle \notag \\ 
	& \hspace{4cm}= \left\langle\sqrt{\lambda_p}(\nabla_{\bx^t_i} f^{KD}_i(\bx^t_{i},\bw^t)-\nabla_{\bx^t_i} f^{KD}_i(\bx^t_{i},\bw^t_{i})), \sqrt{\lambda_p}(\bx^{t+1}_{i}-\bx^t_{i})\right\rangle \nonumber \\
	& \hspace{4cm}\leq \frac{\lambda_p}{2}\|\nabla_{\bx^t_i} f^{KD}_i(\bx^t_{i},\bw^t)-\nabla_{\bx^t_i} f^{KD}_i(\bx^t_{i},\bw^t_{i})\|^2 + \frac{\lambda_p}{2}\|\bx^{t+1}_{i}-\bx^t_{i}\|^2 \nonumber
\end{align}
and similarly,
\begin{align}	
	&\left\langle\lambda_p(\nabla_{\bx^t_i} f^{KD}_i(\widetilde{Q}_{\bc^t_{i}}(\bx^t_{i}),\bw^t)-\nabla_{\bx^t_i} f^{KD}_i(\widetilde{Q}_{\bc^t_{i}}(\bx^t_{i}),\bw^t_{i})), \bx^{t+1}_{i}-\bx^t_{i}\right\rangle \notag \\
	& \hspace{5cm} \leq \frac{\lambda_p}{2}\|\nabla_{\bx^t_i} f^{KD}_i(\widetilde{Q}_{\bc^t_{i}}(\bx^t_{i}),\bw^t)-\nabla_{\bx^t_i} f^{KD}_i(\widetilde{Q}_{\bc^t_{i}}(\bx^t_{i}),\bw^t_{i})\|^2 \nonumber \\
	& \hspace{6cm}+ \frac{\lambda_p}{2}\|\bx^{t+1}_{i}-\bx^t_{i}\|^2 \nonumber
\end{align}

\begin{claim}\label{thm2:claim:lower-bound1}
	Let 
	\begin{align*}
		A(\bx^{t+1}_{i}) &:= (1-\lambda_p)\left\langle \nabla f_i(\bx^t_{i}), \bx^{t+1}_{i}-\bx^t_{i}\right\rangle 
		+(1-\lambda_p) \left\langle \nabla_{\bx^t_{i}} f_i(\widetilde{Q}_{\bc^t_{i}}(\bx^t_{i})), \bx^{t+1}_{i}-\bx^t_{i}\right\rangle \notag \\
		&\hspace{1cm} +\left\langle \lambda_p(\nabla_{\bx^t_i} f^{KD}_i(\bx^t_{i},\bw^t_{i})), \bx^{t+1}_{i}-\bx^t_{i}\right\rangle + \left\langle \lambda_p(\nabla_{\bx^t_i} f^{KD}_i(\widetilde{Q}_{\bc^t_{i}}(\bx^t_{i}),\bw^t_{i})), \bx^{t+1}_{i}-\bx^t_{i}\right\rangle\\
		& \hspace{1cm} + \lambda R(\bx^{t+1}_{i},\bc^t_{i})+ \frac{1}{2\eta_1}\|\bx^{t+1}_{i}-\bx^t_{i}\|^2  \notag \\
		A(\bx^t_{i}) &:= \lambda R(\bx^t_{i},\bc^t_{i}).
	\end{align*} 
	Then $A(\bx^{t+1}_{i})\leq A(\bx^t_{i})$.
\end{claim}
Using the inequality from Claim~\ref{thm2:claim:lower-bound1} in \eqref{thm2:first-part-interim1} gives
\begin{align}
	&F_i(\bx^{t+1}_{i},\bc^t_{i},\bw^{t}){+} (\frac{1}{2\eta_1}-\frac{\lambda_p(L_{D_1}+L_{DQ_1})+(1-\lambda_p)(L+G^{(i)}L_{Q_1}+G^{(i)}_{Q_1}Ll_{Q_1})}{2})\|\bx^{t+1}_{i}-\bx^t_{i}\|^2 \notag \\
	&\stackrel{\text{(a)}}{\leq} (1-\lambda_p)\left(f_i(\bx^t_{i})+f_i(\widetilde{Q}_{\bc^t_{i}}(\bx^t_{i}))\right)+\lambda_p\left(f^{KD}_i(\bx^t_i,\bw^t)+f^{KD}_i(\widetilde{Q}_{\bc^t_{i}}(\bx^t_{i}),\bw^t)\right) + A(\bx^{t+1}_{i}) \nonumber \\
	& \hspace{1 cm}+ \frac{\lambda_p}{2}\|\nabla_{\bx^t_i} f^{KD}_i(\bx^t_{i},\bw^t)-\nabla_{\bx^t_i} f^{KD}_i(\bx^t_{i},\bw^t_{i})\|^2 + \frac{\lambda_p}{2}\|\nabla_{\bx^t_i} f^{KD}_i(\widetilde{Q}_{\bc^t_{i}}(\bx^t_{i}),\bw^t) \notag \\
	& \hspace{1cm} -\nabla_{\bx^t_i} f^{KD}_i(\widetilde{Q}_{\bc^t_{i}}(\bx^t_{i}),\bw^t_{i})\|^2 + \lambda_p\|\bx^{t+1}_{i}-\bx^t_{i}\|^2 \notag \\
	&\stackrel{\text{(b)}}{\leq} (1-\lambda_p)\left(f_i(\bx^t_{i})+f_i(\widetilde{Q}_{\bc^t_{i}}(\bx^t_{i}))\right)+\lambda_p\left(f^{KD}_i(\bx^t_i,\bw^t)+f^{KD}_i(\widetilde{Q}_{\bc^t_{i}}(\bx^t_{i}),\bw^t)\right) + \lambda R(\bx^t_{i},\bc^t_{i}) \nonumber \\
	& \hspace{1 cm}+ \frac{\lambda_p}{2}\|\nabla f^{KD}_i(\bx^t_{i},\bw^t)-\nabla f^{KD}_i(\bx^t_{i},\bw^t_{i})\|^2 + \frac{\lambda_p}{2}\|\nabla f^{KD}_i(\widetilde{Q}_{\bc^t_{i}}(\bx^t_{i}),\bw^t) \notag \\
	& \hspace{1cm} -\nabla f^{KD}_i(\widetilde{Q}_{\bc^t_{i}}(\bx^t_{i}),\bw^t_{i})\|^2 + \lambda_p\|\bx^{t+1}_{i}-\bx^t_{i}\|^2 \notag \\
	&\stackrel{\text{(c)}}{\leq} (1-\lambda_p)\left(f_i(\bx^t_{i})+f_i(\widetilde{Q}_{\bc^t_{i}}(\bx^t_{i}))\right)+\lambda_p\left(f^{KD}_i(\bx^t_i,\bw^t)+f^{KD}_i(\widetilde{Q}_{\bc^t_{i}}(\bx^t_{i}),\bw^t)\right) \nonumber \\
	& \hspace{1 cm} +\lambda R(\bx^{t}_{i},\bc^{t}_{i}) + \frac{\lambda_p(L^2_{D}+L^2_{DQ})}{2}\|\bw^t-\bw^t_{i}\|^2 +\lambda_p\|\bx^{t+1}_{i}-\bx^t_{i}\|^2 \notag \\
	&= F_i(\bx^t_{i},\bc^t_{i},\bw^{t}) + \frac{\lambda_p(L^2_{D}+L^2_{DQ})}{2}\|\bw^{t}_{i}-\bw^{t}\|^2+\lambda_p\|\bx^{t+1}_{i}-\bx^t_{i}\|^2. \label{thm2:first-part-interim3}
\end{align}
To obtain (a), we substituted the value of $A(\bx^{t+1}_{i})$ from Claim~\ref{thm2:claim:lower-bound1} into \eqref{thm2:first-part-interim1}. In (b), we used $A(\bx^{t+1}_{i})\leq \lambda R(\bx^t_{i},\bc^t_{i})$, $\|\nabla_{\bx^t_{i}} f^{KD}_i(\bx^t_{i},\bw^t)-\nabla_{\bx^t_{i}} f^{KD}_i(\bx^t_{i},\bw^t_{i})\|^2 \leq \|\nabla f^{KD}_i(\bx^t_{i},\bw^t)-\nabla f^{KD}_i(\bx^t_{i},\bw^t_{i})\|^2$ and the fact that  $\|\nabla_{\bx^t_{i}} f^{KD}_i(\widetilde{Q}_{\bc^t_{i}}(\bx^t_{i})),\bw^t)-\nabla_{\bx^t_{i}} f^{KD}_i(\widetilde{Q}_{\bc^t_{i}}(\bx^t_{i})),\bw^t_{i})\|^2 \leq \|\nabla f^{KD}_i(\bx^t_{i},\bw^t)-\nabla f^{KD}_i(\bx^t_{i},\bw^t_{i})\|^2$. And in (c) we used the assumption that $f^{KD}_i(\bx,\bw)$ is $L_D$-smooth and $f^{KD}_i(\widetilde{Q}_{\bc}(\bx),\bw)$ is $L_{DQ}$-smooth.

Substituting $\eta_1=\frac{1}{2(\lambda_p(2+L_{D_1}+L_{DQ_1})+(1-\lambda_p)(L+G^{(i)}L_{Q_1}+G^{(i)}_{Q_1}Ll_{Q_1}))}$ in \eqref{thm2:first-part-interim3} gives:
\begin{align} \label{thm2:dec1}
	&F_i(\bx^{t+1}_{i},\bc^t_{i},\bw^{t})
	+(\frac{\lambda_p(2+L_{D_1}+L_{DQ_1})+(1-\lambda_p)(L+G^{(i)}L_{Q_1}+G^{(i)}_{Q_1}Ll_{Q_1})}{2})\|\bx^{t+1}_{i}-\bx^t_{i}\|^2 \notag \\
	& \hspace{6cm}\leq F_i(\bx^t_{i},\bc^t_{i},\bw^{t})
	+ \frac{\lambda_p(L^2_D+L^2_{DQ})}{2}\|\bw^{t}_{i}-\bw^{t}\|^2.
\end{align}
\subsubsection{Sufficient Decrease Due to $\bc_i$}
In parallel with Claim~\ref{thm2:claim lqxsmoothness}, we have following smoothness result for $\bc$:
\begin{claim} \label{thm2:claim lqcsmoothness}
	$(1-\lambda_p)f_i(\widetilde{Q}_\bc(\bx))+\lambda_pf^{KD}_{i}(\widetilde{Q}_{\bc}(\bx),\bw))$ is $(\lambda_pL_{DQ_2}+(1-\lambda_p)(G^{(i)}L_{Q_2}+G^{(i)}_{Q_2}Ll_{Q_2}))$-smooth with respect to $\bc$.
\end{claim}
From Claim \ref{thm2:claim lqcsmoothness} we have:
\begin{align}
	&F_i(\bx^{t+1}_{i},\bc^{t+1}_{i},\bw^{t})+(\frac{1}{2\eta_2}-\frac{\lambda_pL_{DQ_2}+(1-\lambda_p)(G^{(i)}L_{Q_2}+G^{(i)}_{Q_2}Ll_{Q_2})}{2})\|\bc^{t+1}_{i}-\bc^t_{i}\|^2 \notag \\
	&\hspace{0.5cm} \leq (1-\lambda_p)\left(f_i(\bx^{t+1}_{i})+f_i(\widetilde{Q}_{\bc^t_{i}}(\bx^{t+1}_{i}))\right)+\lambda R(\bx^{t+1}_{i},\bc^{t+1}_{i})  \notag \\
	& \hspace{1cm} +\lambda_p\left(f^{KD}_i(\bx^{t+1}_i,\bw^t)+f^{KD}_{i}(\widetilde{Q}_{\bc^{t}_i}(\bx^{t+1}_i),\bw^t)\right)  + \frac{1}{2\eta_2}\|\bc^{t+1}_{i}-\bc^t_{i}\|^2 \notag \\
	&\hspace{1cm}+(1-\lambda_p) \left\langle \nabla_{\bc^t_{i}} f_i(\widetilde{Q}_{\bc^t_{i}}(\bx^{t+1}_{i})), \bc^{t+1}_{i}-\bc^t_{i}\right\rangle +\lambda_p \left\langle \nabla_{\bc^t_{i}} f^{KD}_i(\widetilde{Q}_{\bc^t_{i}}(\bx^{t+1}_{i}),\bw^t_i), \bc^{t+1}_{i}-\bc^t_{i}\right\rangle \notag \\
	&\hspace{1cm} +\frac{\lambda_p}{2} \| \nabla_{\bc^t_{i}} f^{KD}_i(\widetilde{Q}_{\bc^t_{i}}(\bx^{t+1}_{i}),\bw^t)-\nabla_{\bc^t_{i}} f^{KD}_i(\widetilde{Q}_{\bc^t_{i}}(\bx^{t+1}_{i}),\bw^t_i)\|^2 + \frac{\lambda_p}{2} \| \bc^{t+1}_{i}-\bc^t_{i}\|^2. 
	\label{thm2:second-part-interim1}
\end{align}
Where we used: 
\begin{align}	
	&\left\langle\lambda_p(\nabla_{\bc^{t}_i} f^{KD}_i(\widetilde{Q}_{\bc^t_{i}}(\bx^{t+1}_{i}),\bw^t)-\nabla_{\bc^{t}_i} f^{KD}_i(\widetilde{Q}_{\bc^t_{i}}(\bx^{t+1}_{i}),\bw^t_{i})), \bc^{t+1}_{i}-\bc^t_{i}\right\rangle
	\notag\\
	& \hspace{4cm} \leq \frac{\lambda_p}{2}\|\nabla_{\bc^t_i} f^{KD}_i(\widetilde{Q}_{\bc^t_{i}}(\bx^{t+1}_{i}),\bw^t)-\nabla_{\bc^t_i} f^{KD}_i(\widetilde{Q}_{\bc^t_{i}}(\bx^{t+1}_{i}),\bw^t_{i})\|^2 \nonumber \\
	& \hspace{5cm} + \frac{\lambda_p}{2}\|\bc^{t+1}_{i}-\bc^t_{i}\|^2 \nonumber
\end{align}
\begin{claim} \label{thm2:claim:lower-bound2}
	Let 
	\begin{align*}
		B(\bc^{t+1}_{i}) &:= \lambda R(\bx^{t+1}_{i},\bc^{t+1}_{i}) + (1-\lambda_p)\left\langle \nabla_{\bc^t_{i}} f_i(\widetilde{Q}_{\bc^t_{i}}(\bx^{t+1}_{i})), \bc^{t+1}_{i}-\bc^t_{i}\right\rangle \notag  \\ 
		& \quad  + \lambda_p \left\langle \nabla_{\bc^t_{i}} f^{KD}_i(\widetilde{Q}_{\bc^t_{i}}(\bx^{t+1}_{i}),\bw^t_i), \bc^{t+1}_{i}-\bc^t_{i}\right\rangle + \frac{1}{2\eta_2}\|\bc^{t+1}_{i}-\bc^t_{i}\|^2 \notag \\
		B(\bc^t_{i}) &:= \lambda R(\bx^{t+1}_{i},\bc^t_{i}).
	\end{align*} 
	Then $B(\bc^{t+1}_{i})\leq B(\bc^t_{i})$.
\end{claim}
Substituting the bound from Claim~\ref{thm2:claim:lower-bound2} in \eqref{thm2:second-part-interim1}, 

\begin{align*}
	&F_i(\bx^{t+1}_{i},\bc^{t+1}_{i},\bw^{t})+(\frac{1}{2\eta_2}-\frac{\lambda_pL_{DQ_2}+(1-\lambda_p)(G^{(i)}L_{Q_2}+G^{(i)}_{Q_2}Ll_{Q_2})}{2})\|\bc^{t+1}_{i}-\bc^t_{i}\|^2  \\
	& \hspace{0.5cm}\leq (1-\lambda_p)\left(f_i(\bx^{t+1}_{i})+f_i(\widetilde{Q}_{\bc^t_{i}}(\bx^{t+1}_{i}))\right)+\lambda_p\left(f^{KD}_i(\bx^{t+1}_i,\bw^t)+f^{KD}_i(\widetilde{Q}_{\bc^t_{i}}(\bx^{t+1}_{i}),\bw^t)\right) \notag \\
	& \hspace{1cm} + \frac{\lambda_p}{2}\|\nabla_{\bc^t_i} f^{KD}_i(\widetilde{Q}_{\bc^t_{i}}(\bx^{t+1}_{i}),\bw^t)-\nabla_{\bc^t_i} f^{KD}_i(\widetilde{Q}_{\bc^t_{i}}(\bx^{t+1}_{i}),\bw^t_{i})\|^2 \\
	& \hspace{1cm} +B(\bc^{t+1}_{i}) + \frac{\lambda_p}{2}\|\bc^{t+1}_{i}-\bc^t_{i}\|^2\\
	&\hspace{0.5cm} \leq (1-\lambda_p)\left(f_i(\bx^{t+1}_{i})+f_i(\widetilde{Q}_{\bc^t_{i}}(\bx^{t+1}_{i}))\right) +\lambda_p\left(f^{KD}_i(\bx^{t+1}_i,\bw^t)+f^{KD}_i(\widetilde{Q}_{\bc^t_{i}}(\bx^{t+1}_{i}),\bw^t)\right)\\
	& \hspace{1cm} + \frac{\lambda_p}{2}\|\nabla f^{KD}_i(\widetilde{Q}_{\bc^t_{i}}(\bx^{t+1}_{i}),\bw^t)-\nabla f^{KD}_i(\widetilde{Q}_{\bc^t_{i}}(\bx^{t+1}_{i}),\bw^t_{i})\|^2 
	+R(\bx^{t+1}_{i},\bc^{t}_{i}) \notag\\ 
	& \hspace{1cm} + \frac{\lambda_p}{2}\|\bc^{t+1}_{i}-\bc^t_{i}\|^2 \\
	& \hspace{0.5cm} \leq (1-\lambda_p)\left(f_i(\bx^{t+1}_{i})+f_i(\widetilde{Q}_{\bc^t_{i}}(\bx^{t+1}_{i}))\right) +\lambda_p\left(f^{KD}_i(\bx^{t+1}_i,\bw^t)+f^{KD}_i(\widetilde{Q}_{\bc^t_{i}}(\bx^{t+1}_{i}),\bw^t)\right)\\
	& \hspace{1cm}+ \frac{\lambda_pL^2_{DQ}}{2}\|\bw^t-\bw^t_{i}\|^2 
	+R(\bx^{t+1}_{i},\bc^{t}_{i}) + \frac{\lambda_p}{2}\|\bc^{t+1}_{i}-\bc^t_{i}\|^2\\
	& \hspace{0.5cm} = F_i(\bx^{t+1}_i,\bc^{t}_i,\bw^t) + \frac{\lambda_pL^2_{DQ}}{2}\|\bw^t-\bw^t_{i}\|^2 + \frac{\lambda_p}{2}\|\bc^{t+1}_{i}-\bc^t_{i}\|^2
\end{align*}

Substituting $\eta_2 = \frac{1}{2(\lambda_p(1+L_{DQ_2})+(1-\lambda_p)(G^{(i)}L_{Q_2}+G^{(i)}_{Q_2}Ll_{Q_2}))}$ gives us: 
\begin{align} \label{thm2:dec2}
	&F_i(\bx^{t+1}_{i},\bc^{t+1}_{i},\bw^{t})+\frac{\lambda_p(1+L_{DQ_2})+(1-\lambda_p)(G^{(i)}L_{Q_2}+G^{(i)}_{Q_2}Ll_{Q_2})}{2}\|\bc^{t+1}_{i}-\bc^t_{i}\|^2  \notag \\
	& \hspace{6cm} \leq	F_i(\bx^{t+1}_{i},\bc^t_{i},\bw^{t}) + \frac{\lambda_pL^2_{DQ}}{2}\|\bw^t-\bw^t_{i}\|^2
\end{align}

\subsubsection{Sufficient Decrease Due to $\bw$}
Now, we use $(\lambda_p(L_{D_2}+L_{DQ_3}))$-smoothness of $F_i(\bx,\bc,\bw)$ with respect to $\bw$:
\begin{align*}
	F_i(\bx^{t+1}_{i},\bc^{t+1}_{i},\bw^{t+1}) &\leq 	F_i(\bx^{t+1}_{i},\bc^{t+1}_{i},\bw^{t}) + \left\langle \nabla_{\bw^{t}} F_i(\bx^{t+1}_{i},\bc^{t+1}_{i},\bw^{t}), \bw^{t+1}-\bw^{t} \right\rangle  \\
	& \hspace{4cm} +\frac{\lambda_p(L_{D_2}+L_{DQ_3})}{2}\|\bw^{t+1}-\bw^{t}\|^2 
\end{align*}
After some algebraic manipulations (see Appendix~\ref{appendix:proof of theorem 2}) we have:
\begin{align} \label{thm2:dec3}
	&F_i(\bx^{t+1}_{i},\bc^{t+1}_{i},\bw^{t+1}) + (\frac{\eta_3}{2}-\lambda_p(L_{D_2}+L_{DQ_3})\eta_3^2)\Big\|\nabla_{\bw^{t}} F_i(\bx^{t+1}_{i},\bc^{t+1}_{i},\bw^{t})\Big\|^2 \notag \\
	& \hspace{3cm}\leq F_i(\bx^{t+1}_{i},\bc^{t+1}_{i},\bw^{t}) \notag \\
	&\hspace{4cm} + (\eta_3+2\lambda_p(L_{D_2}+L_{DQ_3})\eta_3^2)\Big\|\bg^{t} - \nabla_{\bw^{t}_{i}} F_i(\bx^{t+1}_{i},\bc^{t+1}_{i},\bw^{t}_{i})\Big\|^2 \notag \\
	&\hspace{4cm}+ (\eta_3+2\lambda_p(L_{D_2}+L_{DQ_3})\eta_3^2)(\lambda_p(L_{D_2}+L_{DQ_3}))^2\Big\|\bw^{t}_{i}-\bw^{t}\Big\|^2  
\end{align}

\subsubsection{Overall Decrease} 
Let $L^{(i)}_{x},L^{(i)}_{c}$ for any $i\in[n]$ and $L_w$ are defined as follows:
\begin{align}
	L^{(i)}_{x} &= (1-\lambda_p)(L+G^{(i)}L_{Q_1}+G^{(i)}_{Q_1}Ll_{Q_1})+\lambda_p(2+L_{D_1}+L_{DQ_1}) \label{Li_x-defn} \\ 
	L^{(i)}_{c} &= (1-\lambda_p)(G^{(i)}L_{Q_2}+G^{(i)}_{Q_2}Ll_{Q_2})+\lambda_p(1+L_{DQ_2}) \label{Li_c-defn} \\ 
	L_{w} &= L_{D_2}+L_{DQ_3}. \label{L_w-defn} 
\end{align} 
Summing \eqref{thm2:dec1}, \eqref{thm2:dec2}, \eqref{thm2:dec3} we get the overall decrease property:

\begin{align} \label{thm2:dec4}
	&F_i(\bx^{t+1}_{i},\bc^{t+1}_{i},\bw^{t+1}) + (\frac{\eta_3}{2}-\lambda_pL_{w}\eta_3^2)\Big\|\nabla_{\bw^{t}} F_i(\bx^{t+1}_{i},\bc^{t+1}_{i},\bw^{t})\Big\|^2 \ + \frac{L_{x}}{2}\|\bx^{t+1}_{i}-\bx^t_{i}\|^2 \notag \\
	& \hspace{0.5cm} + \frac{L_{c}}{2}\|\bc^{t+1}_{i}-\bc^t_{i}\|^2 \leq   (\eta_3+2\lambda_pL_{w}\eta_3^2)\Big\|\bg^{t} - \nabla_{\bw^{t}_{i}} F_i(\bx^{t+1}_{i},\bc^{t+1}_{i},\bw^{t}_{i})\Big\|^2 \notag \\
	& \hspace{3cm}+ (L^2_{DQ}+\frac{L^2_D}{2}+\eta_3\lambda_pL^2_{w}+2\lambda_p^2L^3_{w}\eta_3^2)\lambda_p\|\bw^{t}_{i}-\bw^{t}\|^2 + F_i(\bx^t_{i},\bc^t_{i},\bw^{t})  
\end{align}
Let $L^{(i)}_{\min}$ for any $i\in[n]$ and $L_{\min}$ are defined as follows:
\begin{align}
	L^{(i)}_{\min} &= \min\{L^{(i)}_{x},L^{(i)}_{c},(\eta_3-2\lambda_pL_{w}\eta_3^2)\} \label{Li_min-defn} \\
	L_{\min} &= \min\{L^{(i)}_{\min}:i\in[n]\}. \label{L_min-defn} 
\end{align}
Then,
\begin{align} \label{thm2:dec5}
	&F_i(\bx^{t+1}_{i},\bc^{t+1}_{i},\bw^{t+1}) + \frac{L_{\min}}{2}\left(\Big\|\nabla_{\bw^{t}} F_i(\bx^{t+1}_{i},\bc^{t+1}_{i},\bw^{t})\Big\|^2 \ +\|\bx^{t+1}_{i}-\bx^t_{i}\|^2 +\|\bc^{t+1}_{i}-\bc^t_{i}\|^2\right) \notag \\
	&\leq (\eta_3+2\lambda_pL_{w}\eta_3^2)\Big\|\bg^{t} - \nabla_{\bw^{t}_{i}} F_i(\bx^{t+1}_{i},\bc^{t+1}_{i},\bw^{t}_{i})\Big\|^2 \notag \\
	& \hspace{2cm}+ (L^2_{DQ}+\frac{L^2_D}{2}+\eta_3\lambda_pL^2_{w}+2\lambda_p^2L^3_{w}\eta_3^2)\lambda_p\|\bw^{t}_{i}-\bw^{t}\|^2 + F_i(\bx^t_{i},\bc^t_{i},\bw^{t})
\end{align}
We have obtained the sufficient decrease property for the alternating steps; now, we need to arrive at the first order stationarity of the gradient of general loss function. To do this we move on with bounding the gradients with respect to each type of variables.
\subsection{Bound on the Gradient}
Now, we will use the first order optimality conditions due to proximal updates and bound the partial gradients with respect to variables $\bx$ and $\bc$. After obtaining bounds for partial gradients we will bound the overall gradient and use our results from Section~\ref{thm2:subsection sufficient decrease} to arrive at the final bound.

\subsubsection{Bound on the Gradient w.r.t.\ $\bx_i$}
Taking the derivative inside the minimization problem \eqref{thm2:lower-bounding-interim1} with respect to $\bx$ at $\bx=\bx^{t+1}_{i}$ and setting it to $0$ gives the following optimality condition:

\begin{align}
	&(1-\lambda_p)\left(\nabla_{\bx^t_{i}} f_i(\bx^t_{i})+\nabla_{\bx^t_{i}} f_i(\widetilde{Q}_{\bc^t_{i}}(\bx^t_{i}))\right) + \lambda_p\left(\nabla_{\bx^t_{i}} f^{KD}_i(\bx^{t}_{i},\bw^t_{i})+ \nabla_{\bx^t_{i}} f^{KD}_i(\widetilde{Q}_{\bc^t_{i}}(\bx^t_{i}),\bw^t_{i})\right) \notag \\ 
	& \hspace{6cm}+\frac{1}{\eta_1}(\bx^{t+1}_{i}-\bx^t_{i})+\lambda\nabla_{\bx^{t+1}_{i}} R(\bx^{t+1}_{i},\bc^t_{i}) = 0  \label{thm2:popt1}
\end{align}

Then we have,
\begin{align*}
	\Big\|\nabla_{\bx^{t+1}_{i}} F_i(\bx^{t+1}_{i},\bc^t_{i},\bw^{t})\Big\| &= \Big\|(1-\lambda_p)\left(\nabla_{\bx^{t+1}_{i}} f_i(\bx^{t+1}_{i}) + \nabla_{\bx^{t+1}_{i}} f_i(\widetilde{Q}_{\bc^t_{i}}(\bx^{t+1}_{i}))\right) \\
	& \hspace{0.5cm}+\lambda_p\left(\nabla_{\bx^{t+1}_i} f^{KD}_i(\bx^{t+1}_{i},\bw^t) + \nabla_{\bx^{t+1}_{i}} f^{KD}_i(\widetilde{Q}_{\bc^t_{i}}(\bx^{t+1}_{i}),\bw^t_{i})\right)  \\
	& \hspace{0.5cm} + \lambda \nabla_{\bx^{t+1}_{i}} R(\bx^{t+1}_{i},\bc^t_{i})\Big\|\\
	&\stackrel{\text{(a)}}{=} \Big\|(1-\lambda_p)\left(\nabla_{\bx^{t+1}_{i}} f_i(\bx^{t+1}_{i})-\nabla_{\bx^t_{i}} f_i(\bx^t_{i}) + \nabla_{\bx^{t+1}_{i}} f_i(\widetilde{Q}_{\bc^t_{i}}(\bx^{t+1}_{i})) \right.\\
	& \left. \hspace{0.5cm}-\nabla_{\bx^t_{i}} f_i(\widetilde{Q}_{\bc^t_{i}}(\bx^t_{i}))\right)-\frac{1}{\eta_1}(\bx^{t+1}_{i}-\bx^t_{i})\\  
	&\hspace{0.5cm} +\lambda_p\left(\nabla_{\bx^{t+1}_i} f^{KD}_i(\bx^{t+1}_{i},\bw^t) - \nabla_{\bx^{t}_i} f^{KD}_i(\bx^{t}_{i},\bw^t_i)  \right. \notag\\
	&\hspace{0.5cm} \left. +\nabla_{\bx^{t+1}_i} f^{KD}_i(\widetilde{Q}_{\bc^t_{i}}(\bx^{t+1}_{i}),\bw^t)  - \nabla_{\bx^{t}_i} f^{KD}_i(\widetilde{Q}_{\bc^t_{i}}(\bx^t_{i}),\bw^t_i)\right) \Big\| \\
	& \stackrel{\text{(b)}}{\leq} \Big(\frac{1}{\eta_1} + (1-\lambda_p)(L+G^{(i)}L_{Q_1}+G^{(i)}_{Q_1}Ll_{Q_1})\Big)\|\bx^{t+1}_{i}-\bx^t_{i}\| \notag\\
	& \hspace{0.5cm} +\lambda_p\|\nabla_{\bx^{t+1}} f^{KD}_i(\bx^{t+1}_{i},\bw^t) - \nabla_{\bx^{t}_i} f^{KD}_i(\bx^{t}_{i},\bw^t_i)\|\\
	&\hspace{0.5cm}+\lambda_p\|\nabla_{\bx^{t+1}_i} f^{KD}_i(\widetilde{Q}_{\bc^t_{i}}(\bx^{t+1}_{i}),\bw^t) - \nabla_{\bx^{t}_i} f^{KD}_i(\widetilde{Q}_{\bc^t_{i}}(\bx^{t}_{i}),\bw^t)\| \\
	& \leq \Big(\frac{1}{\eta_1} + (1-\lambda_p)(L+G^{(i)}L_{Q_1}+G^{(i)}_{Q_1}Ll_{Q_1})\Big)\|\bx^{t+1}_{i}-\bx^t_{i}\| \\
	& \hspace{0.5cm} +\lambda_p\|\nabla f^{KD}_i(\bx^{t+1}_{i},\bw^t) - \nabla f^{KD}_i(\bx^{t}_{i},\bw^t_i)\|\\
	&\hspace{0.5cm}+\lambda_p\|\nabla f^{KD}_i(\widetilde{Q}_{\bc^t_{i}}(\bx^{t+1}_{i}),\bw^t) - \nabla f^{KD}_i(\widetilde{Q}_{\bc^t_{i}}(\bx^{t}_{i}),\bw^t)\| \\
	& \stackrel{\text{(c)}}{\leq} \Big(\frac{1}{\eta_1} + (1-\lambda_p)(L+G^{(i)}L_{Q_1}+G^{(i)}_{Q_1}Ll_{Q_1})\Big)\|\bx^{t+1}_{i}-\bx^t_{i}\| \\
	& \hspace{0.5cm}+\lambda_p(L_D+L_{DQ})(\|\bx^{t+1}_i-\bx^{t}_i\|+\|\bw^{t}_i-\bw^{t}\|)\\
	& = \Big(\frac{1}{\eta_1} {+} \lambda_p(L_D{+}L_{DQ}){+} (1{-}\lambda_p)(L{+}G^{(i)}L_{Q_1}{+}G^{(i)}_{Q_1}Ll_{Q_1})\Big)\|\bx^{t+1}_{i}-\bx^t_{i}\| \\
	& \hspace{0.5cm}+\lambda_p(L_D+L_{DQ})\|\bw^{t}_i-\bw^{t}\|\\
\end{align*}
where (a) is from \eqref{thm2:popt1} by substituting the value of $\lambda\nabla_{\bx^{t+1}_{i}} R(\bx^{t+1}_{i},\bc^t_{i})$, (b) is due to Claim~\ref{claim: lqxclaim} and {\bf A.1}, and (c) is due to {\bf A.7}. This implies: \\
\begin{align*}
	&\Big\|\nabla_{\bx^{t+1}_{i}} F_i(\bx^{t+1}_{i},\bc^t_{i},\bw^{t})\Big\|^2
	\leq2(\lambda_p(L_D+L_{DQ}))^2\|\bw^{t}_{i}-\bw^{t}\|^2 \\
	& \hspace{2cm} + 2\Big(\frac{1}{\eta_1} + \lambda_p(L_D+L_{DQ}) + (1-\lambda_p)(L+G^{(i)}L_{Q_1}+G^{(i)}_{Q_1}Ll_{Q_1})\Big)^2\|\bx^{t+1}_{i}-\bx^t_{i}\|^2\\
\end{align*}
Substituting $\eta_1=\frac{1}{2(\lambda_p(2+L_{D_1}+L_{DQ_1})+(1-\lambda_p)(L+G^{(i)}L_{Q_1}+G^{(i)}_{Q_1}Ll_{Q_1}))}$ we have:
\begin{align} \label{thm2:partial grad1}
	\Big\|\nabla_{\bx^{t+1}_{i}} F_i(\bx^{t+1}_{i},\bc^t_{i},\bw^{t})\Big\|^2 &\leq 2(\lambda_p(L_D+L_{DQ}))^2\|\bw^{t}_{i}-\bw^{t}\|^2 \notag \\
	& \hspace{0.5cm} +2\Big(\lambda_p(4+2L_{D_1}+2L_{DQ_1}+L_D+L_{DQ}) \notag \\
	& \hspace{0.5cm}+ 3(1-\lambda_p)(L+G^{(i)}L_{Q_1}+G^{(i)}_{Q_1}Ll_{Q_1})\Big)^2\|\bx^{t+1}_{i}-\bx^t_{i}\|^2  \notag \\
	& = 18\Big(\frac{\lambda_p}{3}(4+2L_{D_1}+2L_{DQ_1}+L_D+L_{DQ}) \notag \\
	& \hspace{0.5cm}+(1-\lambda_p)(L+G^{(i)}L_{Q_1}+G^{(i)}_{Q_1}Ll_{Q_1})\Big)^2\|\bx^{t+1}_{i}-\bx^t_{i}\|^2 \notag \\
	& \hspace{3cm}+2(\lambda_p(L_D+L_{DQ}))^2\|\bw^{t}_{i}-\bw^{t}\|^2
\end{align}
\subsubsection{Bound on the Gradient w.r.t.\ $\bc_i$}
Similarly, taking the derivative inside the minimization problem \eqref{thm2:lower-bounding-interim2} with respect to $\bc$ at $\bc=\bc^{t+1}_{i}$ and setting it to $0$ gives the following optimality condition:
\begin{align}
	&(1-\lambda_p)\nabla_{\bc^t_{i}} f_i(\widetilde{Q}_{\bc^t_{i}}(\bx^{t+1}_{i})) +\lambda_p \nabla_{\bc^t_{i}} f^{KD}_i(\widetilde{Q}_{\bc^t_{i}}(\bx^{t+1}_{i}),\bw^t_{i}) \notag \\
	& \hspace{6cm}+\frac{1}{\eta_2}(\bc^{t+1}_{i}-\bc^t_{i})+\lambda\nabla_{\bc^{t+1}_{i}} R(\bx^{t+1}_{i},\bc^{t+1}_{i}) = 0 \label{thm2:popt2}
\end{align}
Then we have 
\begin{align*}
	\Big\|\nabla_{\bc^{t+1}_{i}} F_i(\bx^{t+1}_{i},\bc^{t+1}_{i},\bw^{t})\Big\| &= \Big\| (1-\lambda_p)\nabla_{\bc^{t+1}_{i}} f_i(\widetilde{Q}_{\bc^{t+1}_{i}}(\bx^{t+1}_{i})) \\
	& \hspace{0.5cm}+\lambda_p \nabla_{\bc^{t+1}_{i}} f^{KD}_i(\widetilde{Q}_{\bc^{t+1}_{i}}(\bx^{t+1}_{i}),\bw^t)+ \lambda \nabla_{\bc^{t+1}_{i}} R(\bx^{t+1}_{i},\bc^{t+1}_{i})\Big\|\\
	&\stackrel{\text{(a)}}{=} \Big\| (1-\lambda_p)\left(\nabla_{\bc^{t+1}_{i}} f_i(\widetilde{Q}_{\bc^{t+1}_{i}}(\bx^{t+1}_{i})) - \nabla_{\bc^t_{i}} f_i(\widetilde{Q}_{\bc^t_{i}}(\bx^{t+1}_{i}))\right) \\
	& \hspace{0.5cm} + \lambda_p \left(\nabla_{\bc^{t+1}_{i}} f^{KD}_i(\widetilde{Q}_{\bc^{t+1}_{i}}(\bx^{t+1}_{i}),\bw^t)-\nabla_{\bc^{t}_{i}} f^{KD}_i(\widetilde{Q}_{\bc^{t}_{i}}(\bx^{t+1}_{i}),\bw^t_{i})\right) \\
	& \hspace{0.5cm} + \frac{1}{\eta_2}(\bc^t_{i}-\bc^{t+1}_{i})\Big\| \\
	&\leq  (1-\lambda_p)\Big\|\nabla_{\bc^{t+1}_{i}} f_i(\widetilde{Q}_{\bc^{t+1}_{i}}(\bx^{t+1}_{i})) - \nabla_{\bc^t_{i}} f_i(\widetilde{Q}_{\bc^t_{i}}(\bx^{t+1}_{i}))\Big\| \\
	& \hspace{0.5cm} + \lambda_p \Big\| \nabla_{\bc^{t+1}_{i}} f^{KD}_i(\widetilde{Q}_{\bc^{t+1}_{i}}(\bx^{t+1}_{i}),\bw^t)-\nabla_{\bc^{t}_{i}} f^{KD}_i(\widetilde{Q}_{\bc^{t}_{i}}(\bx^{t+1}_{i}),\bw^t_{i})\Big\| \\
	& \hspace{0.5cm} + \frac{1}{\eta_2}\|\bc^t_{i}-\bc^{t+1}_{i} \| \\
	&\leq  (1-\lambda_p)\Big\| \nabla_{\bc^{t+1}_{i}} f_i(\widetilde{Q}_{\bc^{t+1}_{i}}(\bx^{t+1}_{i})) - \nabla_{\bc^t_{i}} f_i(\widetilde{Q}_{\bc^t_{i}}(\bx^{t+1}_{i}))\Big\| \\
	&\hspace{0.5cm} + \lambda_p \Big\| \nabla f^{KD}_i(\widetilde{Q}_{\bc^{t+1}_{i}}(\bx^{t+1}_{i}),\bw^t)-\nabla f^{KD}_i(\widetilde{Q}_{\bc^{t}_{i}}(\bx^{t+1}_{i}),\bw^t_{i})\Big\| \\
	& \hspace{0.5cm} + \frac{1}{\eta_2}\|\bc^t_{i}-\bc^{t+1}_{i} \|\\
	& \leq \Big(\frac{1}{\eta_2}+\lambda_pL_{DQ} +(1-\lambda_p)(G^{(i)}L_{Q_2}+G^{(i)}_{Q_2}Ll_{Q_2})\Big)\|\bc^{t+1}_{i}-\bc^t_{i}\| \\
	& \hspace{0.5cm}+ \lambda_pL_{DQ}\|\bw^t-\bw^t_{i}\|
\end{align*}

where in (a) we substituted the value of $\lambda\nabla_{\bc^{t+1}_{i}} R(\bx^{t+1}_{i},\bc^{t+1}_{i})$ from \eqref{thm2:popt2} and the last inequality is due to Claim~\ref{claim: lqcclaim} and Assumption \textbf{A.7}. As a result we have, 
\begin{align*} 
	\Big\|\nabla_{\bc^{t+1}_{i}} F_i(\bx^{t+1}_{i},\bc^{t+1}_{i},\bw^{t})\Big\|^2 &\leq 2\Big(\frac{1}{\eta_2}+\lambda_pL_{DQ} +(1-\lambda_p)(G^{(i)}L_{Q_2}+G^{(i)}_{Q_2}Ll_{Q_2})\Big)^2\|\bc^{t+1}_{i}-\bc^t_{i}\|^2\\
	& \hspace{0.5cm}+2(\lambda_pL_{DQ})^2\|\bw^t-\bw^t_{i}\|^2
\end{align*}

substituting $\eta_2 = \frac{1}{2(\lambda_p(1+L_{DQ_2})+(1-\lambda_p)(G^{(i)}L_{Q_2}+G^{(i)}_{Q_2}Ll_{Q_2}))}$ we have:
\begin{align} \label{thm2:partial grad2}
	\Big\|\nabla_{\bc^{t+1}_{i}} F_i(\bx^{t+1}_{i},\bc^{t+1}_{i},\bw^{t})\Big\|^2 &\leq 18\Big(\frac{\lambda_p}{3}(2+2L_{DQ_2}+L_{DQ}) \notag \\
	& \hspace{0.5cm}+(1-\lambda_p)( G^{(i)}L_{Q_2}+G^{(i)}_{Q_2}Ll_{Q_2})\Big)^2\|\bc^{t+1}_{i}-\bc^t_{i}\|^2 \notag \\
	& \hspace{0.5cm} +2(\lambda_pL_{DQ})^2\|\bw^t-\bw^t_{i}\|^2
\end{align}
\subsubsection{Overall Bound}
Let $\|\bG^{t}_{i}\|^2= \Big\|[\nabla_{\bx^{t+1}_{i}} F_i(\bx^{t+1}_{i},\bc^t_{i},\bw^{t})^T,\nabla_{\bc^{t+1}_{i}} F_i(\bx^{t+1}_{i},\bc^{t+1}_{i},\bw^{t})^T , \nabla_{\bw^{t}} F_i(\bx^{t+1}_{i},\bc^{t+1}_{i},\bw^{t})^T]^T\Big\|^2$. Then,
\begin{align*}
	\|\bG^{t}_{i}\|^2 & = \Big\|[\nabla_{\bx^{t+1}_{i}} F_i(\bx^{t+1}_{i},\bc^t_{i},\bw^{t})^T,\nabla_{\bc^{t+1}_{i}} F_i(\bx^{t+1}_{i},\bc^{t+1}_{i},\bw^{t})^T , \nabla_{\bw^{t}} F_i(\bx^{t+1}_{i},\bc^{t+1}_{i},\bw^{t})^T]^T\Big\|^2 \\
	& = \Big\|\nabla_{\bx^{t+1}_{i}} F_i(\bx^{t+1}_{i},\bc^t_{i},\bw^{t})\Big\|^2 + \Big\|\nabla_{\bc^{t+1}_{i}} F_i(\bx^{t+1}_{i},\bc^{t+1}_{i},\bw^{t})\Big\|^2 + \Big\|\nabla_{\bw^{t}} F_i(\bx^{t+1}_{i},\bc^{t+1}_{i},\bw^{t})\Big\|^2 \\
	& \leq 18\Big(\frac{\lambda_p}{3}(4+2L_{D_1}+2L_{DQ_1}+L_D+L_{DQ})\\
	& \hspace{0.5cm}+(1-\lambda_p)(L+G^{(i)}L_{Q_1}+G^{(i)}_{Q_1}Ll_{Q_1})\Big)^2\|\bx^{t+1}_{i}-\bx^t_{i}\|^2\\
	&\hspace{0.5cm} +18\Big(\frac{\lambda_p}{3}(2+2L_{DQ_2}+L_{DQ}) +(1-\lambda_p)( G^{(i)}L_{Q_2}+G^{(i)}_{Q_2}Ll_{Q_2})\Big)^2\|\bc^{t+1}_{i}-\bc^t_{i}\|^2 \\
	& \hspace{0.5cm} +2(\lambda_p(L_D+2L_{DQ}))^2\|\bw^{t}_{i}-\bw^{t}\|^2+\Big\|\nabla_{\bw^{t}} F_i(\bx^{t+1}_{i},\bc^{t+1}_{i},\bw^{t})\Big\|^2
\end{align*}

where the last inequality is due to \eqref{thm2:partial grad1}, \eqref{thm2:partial grad2} and using that $2(\lambda_pL_{DQ})^2+2(\lambda_p(L_D+L_{DQ}))^2\leq2(\lambda_p(L_{D}+2L_{DQ}))^2$ . 
Let 
\begin{align}
	L^{(i)}_{\max} &= \max\left\{\sqrt{\frac{1}{18}},\left(\frac{\lambda_p}{3}(2+2L_{DQ_2}+L_{DQ}) +(1-\lambda_p)( G^{(i)}L_{Q_2}+G^{(i)}_{Q_2}Ll_{Q_2})\right), \right. \notag \\
	&\hspace{1cm} \left.\left(\frac{\lambda_p}{3}(4+2L_{D_1}+2L_{DQ_1}+L_D+L_{DQ})+(1-\lambda_p)(L+G^{(i)}L_{Q_1}+G^{(i)}_{Q_1}Ll_{Q_1})\right)\right\} \label{Li_max-defn}
\end{align}
Then,
\begin{align} \label{thm2:bound over gradient}
	&\Big\|[\nabla_{\bx^{t+1}_{i}} F_i(\bx^{t+1}_{i},\bc^t_{i},\bw^{t})^T,\nabla_{\bc^{t+1}_{i}} F_i(\bx^{t+1}_{i},\bc^{t+1}_{i},\bw^{t})^T , \nabla_{\bw^{t}} F_i(\bx^{t+1}_{i},\bc^{t+1}_{i},\bw^{t})^T]^T\Big\|^2 \notag\\
	& \leq 18(L^{(i)}_{\max})^2(\|\bx^{t+1}_{i}-\bx^t_{i}\|^2+\|\bc^{t+1}_{i}-\bc^t_{i}\|^2+\Big\|\nabla_{\bw^{t}} F_i(\bx^{t+1}_{i},\bc^{t+1}_{i},\bw^{t})\Big\|^2) \notag\\
	& \hspace{1cm} +2(\lambda_p(L_D+2L_{DQ}))^2\|\bw^{t}_{i}-\bw^{t}\|^2 \notag\\
	&\stackrel{\text{(a)}}{\leq}  36\frac{(L^{(i)}_{\max})^2}{L_{\min}} \Big[(\eta_3+2\lambda_pL_{w}\eta_3^2)\Big\|\bg^{t} - \nabla_{\bw^{t}_{i}} F_i(\bx^{t+1}_{i},\bc^{t+1}_{i},\bw^{t}_{i})\Big\|^2 \notag \\
	& \hspace{1cm}+ (L^2_{DQ}+\frac{L^2_D}{2}+\eta_3\lambda_pL^2_{w}+2\lambda_p^2L^3_{w}\eta_3^2)\lambda_p\|\bw^{t}_{i}-\bw^{t}\|^2  \notag\\ 
	& \hspace{1cm} + F_i(\bx^t_{i},\bc^t_{i},\bw^{t}) - F_i(\bx^{t+1}_{i},\bc^{t+1}_{i},\bw^{t+1}) \Big]+2(\lambda_p(L_D+2L_{DQ}))^2\|\bw^{t}_{i}-\bw^{t}\|^2, 
\end{align}
where in (a) we use the bound from \eqref{thm2:dec5}, and $L_{\min}$ is defined in \eqref{L_min-defn}.

Now we state a useful lemma that bounds the average deviation between the local versions of the global model at all clients, and the global model itself. See Appendix~\ref{appendix:proof of theorem 2} for a proof.
\begin{lemma}\label{thm2:lemma1}
	Let $\eta_3$ be chosen such that $\eta_3 \leq \sqrt{\frac{1}{6\tau^2(\lambda_pL_{w})^2\left(1+\frac{\overline{L}_{\max}^2}{(L^{(\min)}_{\max})^2}\right)}}$ where $L^{(\min)}_{\max}=\min\{L^{(i)}_{\max}:i\in[n]\}$ and $\overline{L}_{\max} = \sqrt{\frac{1}{n}\sum_{i=1}^n(L^{(i)}_{\max})^2}$ (where $L^{(i)}_{\max}$ is defined in \eqref{Li_max-defn}), then we have,
	\begin{align*} 
		\frac{1}{T}\sum_{t=0}^{T-1}\frac{1}{n}\sum_{i=1}^{n} (L^{(i)}_{\max})^2
		\|\bw^{t}-\bw^{t}_{i}\|^2 \leq \frac{1}{T} \sum_{t=0}^{T-1} \gamma_t \leq  6\tau^2\eta_3^2\frac{1}{n}\sum_{i=1}^n(L^{(i)}_{\max})^2\kappa_i
	\end{align*}
\end{lemma}

As a corollary:
\begin{corollary}\label{thm2:corollary diversity}
	Recall, $\bg^{t} = \frac{1}{n} \sum_{i=1}^n \nabla_{\bw^{t}} F_i(\bx^{t+1}_{i},\bc^{t+1}_{i},\bw^{t}_{i})$. Then, we have:
	\begin{align*}
		\frac{1}{T}\sum_{t=0}^{T-1} \frac{1}{n} \sum_{i=1}^n (L^{(i)}_{\max})^2 \Big\|\bg^{t} - \nabla_{\bw^{t}_{i}} F_i(\bx^{t+1}_{i},\bc^{t+1}_{i},\bw^{t}_{i})\Big\|^2 &\leq  3 \frac{1}{n}\sum_{i=1}^n(L^{(i)}_{\max})^2\kappa_i\\
		& \hspace{-2cm} + 3(\lambda_pL_{w})^2\Big(1+\frac{\overline{L}_{\max}^2}{(L^{(\min)}_{\max})^2}\Big) 6\tau^2 \eta_3^2 \frac{1}{n}\sum_{i=1}^n(L^{(i)}_{\max})^2\kappa_i,
	\end{align*}
	where $L^{(i)}_{\max}$ is defined in \eqref{Li_max-defn}, and $\overline{L}_{\max},L^{(\min)}_{\max}$ are defined in Lemma~\ref{thm2:lemma1}.
\end{corollary}

Let $\overline{\kappa} := \frac{1}{n}\sum_{i=1}^n(L^{(i)}_{\max})^2\kappa_i$ and $C_{L}:= 1+\frac{\frac{1}{n}\sum_{i=1}^{n}(L^{(i)}_{\max})^2}{(\min_i\{L^{(i)}_{\max}\})^2}$ , using Lemma~\ref{thm2:lemma1} and Corollary~\ref{thm2:corollary diversity}, summing the bound in \eqref{thm2:bound over gradient} over time and clients, dividing by $T$ and $n$:
\begin{align} \label{thm2:res1}
	\frac{1}{T}\sum_{t=0}^{T-1}\frac{1}{n}\sum_{i=1}^{n} \|\bG^{t}_{i}\|^2 &\leq 
	\frac{36}{L_{\min}} \Bigg[(\eta_3+2\lambda_pL_{w}\eta_3^2)\times\bigg(3\overline{\kappa} + 3(\lambda_pL_{w})^2 C_L 6\tau^2 \eta_3^2 \overline{\kappa}\bigg) \notag \\
	& \hspace{0.5cm}+ (L^2_{DQ}+\frac{L^2_D}{2}+\eta_3\lambda_pL^2_{w}+2\lambda_p^2L^3_{w}\eta_3^2)\lambda_p\times 6\tau^2\eta_3^2\overline{\kappa} \notag\\ 
	& \hspace{1cm} + \frac{1}{T}\sum_{t=0}^{T-1}\frac{1}{n}\sum_{i=1}^{n}(L^{(i)}_{\max})^2\left(F_i(\bx^t_{i},\bc^t_{i},\bw^{t}) - F_i(\bx^{t+1}_{i},\bc^{t+1}_{i},\bw^{t+1})\right) \Bigg] \notag \\
	&\hspace{1.5cm}+2(\lambda_p(L_D+2L_{DQ}))^2\times 6\tau^2\eta_3^2\overline{\kappa} \notag\\ 
	& = \frac{36}{L_{\min}} \Big[3\tau^2\eta_3^2\overline{\kappa}(2\lambda_pL_{DQ}^2 + \lambda_pL_D^2+\eta_3\lambda_p^2L^2_{w}(2+6C_L)+\eta_3^2\lambda_p^3L^3_{w}(4+12C_L)) \notag \\
	&\hspace{0.5cm}+3\eta_3\overline{\kappa}+6\lambda_pL_{w}\eta_3^2\overline{\kappa}\Big]  +\frac{36}{L_{\min}} \frac{1}{n}\sum_{i=1}^{n} \frac{(L^{(i)}_{\max})^2\Delta^{(i)}_F}{T} \notag \\
	& \hspace{0.5cm}
	+ 12\lambda_p^2(L_D+2L_{DQ})^2\tau^2\eta_3^2\overline{\kappa}
\end{align}

where  $\Delta^{(i)}_{F}=F_i(\bx^{0}_{i},\bc^{0}_{i},\bw^{0}_{i})-F_i(\bx^T_{i},\bc^T_{i},\bw^{T}_{i})$.

\textbf{Choice of $\eta_3$.} Note that in Lemma~\ref{thm2:lemma1} we chose $\eta_3$ such that $ \eta_3 \leq \sqrt{\frac{1}{6\tau^2\lambda_p^2L_{w}^2C_L}}$. Now, we further introduce upper bounds on $\eta_3$. 
\begin{itemize}
	\item We can choose $\eta_3$ small enough so that $L_{\min}=\eta_3-2\lambda_pL_{w}\eta_3^2$; see the definition of $L_{\min}$ in \eqref{L_min-defn}.
	\item We can choose $\eta_3$ small enough so that $\eta_3-2\lambda_pL_{w}\eta_3^2 \geq \frac{\eta_3}{2}$. This is equivalent to choosing $ \eta_3 \leq \frac{1}{4\lambda_pL_{w}}$.
\end{itemize}

These two choices imply $L_{\min} \geq \frac{\eta_3}{2}$. 

In the end, we have 2 critical constraints on $\eta_3, \{\eta_3:\eta_3 \leq \sqrt{\frac{1}{6\tau^2\lambda_p^2L_{w}^2C_L}},\eta_3 \leq \frac{1}{4\lambda_pL_{w}}\} $ . Then, let $\{\eta_3:\eta_3 \leq \frac{1}{4\lambda_pL_{w}\tau\sqrt{C_L}}\}$. Moreover, choosing $\tau \leq \sqrt{T}$ we can take $\eta_3 = \frac{1}{4\lambda_pL_{w}\sqrt{C_L}\sqrt{T}} $ this choice clearly satisfies the above constraints.

From \eqref{thm2:res1} we have,
\begin{align*} 
	\nonumber \frac{1}{T}\sum_{t=0}^{T-1}\frac{1}{n}\sum_{i=1}^{n} \Big\|\bG^{t}_{i}\Big\|^2
	&\stackrel{(a)}{\leq}\frac{72}{\eta_3} \Big[3\tau^2\eta_3^2\overline{\kappa}(2\lambda_pL_{DQ}^2 + \lambda_pL_D^2+\eta_3\lambda_p^2L^2_{w}(2+6C_L)+\eta_3^2\lambda_p^3L^3_{w}(4+12C_L))\\
	& \hspace{1cm}+3\eta_3\overline{\kappa}+6\lambda_pL_{w}\eta_3^2\overline{\kappa}\Big] +\frac{72}{\eta_3} \frac{1}{n}\sum_{i=1}^{n} \frac{(L^{(i)}_{\max})^2\Delta^{(i)}_F}{T} \notag \\
	& \hspace{1cm} 
	+ 12\lambda_p^2(L_D+2L_{DQ})^2\tau^2\eta_3^2\overline{\kappa} \\
	& =72\Big[3\tau^2\eta_3\overline{\kappa}(2\lambda_pL_{DQ}^2 + \lambda_pL_D^2+\eta_3\lambda_p^2L^2_{w}(2+6C_L)+\eta_3^2\lambda_p^3L^3_{w}(4+12C_L))\\
	&\hspace{1cm}+3\overline{\kappa}+6\lambda_pL_{w}\eta_3\overline{\kappa}\Big]+\frac{72}{\eta_3} \frac{1}{n}\sum_{i=1}^{n} \frac{(L^{(i)}_{\max})^2\Delta^{(i)}_F}{T}  \notag \\
	& \hspace{1cm} 
	+ 12\lambda_p^2(L_D+2L_{DQ})^2\tau^2\eta_3^2\overline{\kappa} \\
\end{align*}
In (a) we used $L_{\min} \geq \frac{\eta_3}{2}$. Now, we plug in $\eta_3 = \frac{1}{4\lambda_pL_{w}\sqrt{C_L}\sqrt{T}}$ then:
\begin{align*} 
	\frac{1}{T}\sum_{t=0}^{T-1}\frac{1}{n}\sum_{i=1}^{n} \Big\|\bG^{t}_{i}\Big\|^2 & \leq 72 \Big[ \frac{3}{4}\tau^2\overline{\kappa}\frac{L_D^2+2L_{DQ}^2}{\sqrt{C_L}L_w}\frac{1}{\sqrt{T}}+\frac{3(2+6C_L)}{16C_L}
	\tau^2\overline{\kappa}\frac{1}{T} + \frac{3(4+12C_L)}{64C^{\frac{3}{2}}_L}\tau^2\overline{\kappa}\frac{1}{T^{\frac{3}{2}}}\\
	&\hspace{1cm}+3\overline{\kappa} + \frac{3}{2}\frac{\overline{\kappa}}{\sqrt{C_L}\sqrt{T}}\Big] +288\lambda_pL_{w}\sqrt{C_L}\frac{1}{n}\sum_{i=1}^{n} \frac{(L^{(i)}_{\max})^2\Delta^{(i)}_F}{\sqrt{T}} \\
	& \hspace{1cm} +\frac{3}{4}\tau^2\overline{\kappa}\frac{(L_{D}+2L_{DQ})^2}{C_L L^2_{w}}\frac{1}{T}\\
	& =  \frac{54\frac{L_D^2+2L_{DQ}^2}{\sqrt{C_L}L_{w}}\tau^2\overline{\kappa}+\frac{108\overline{\kappa}}{\sqrt{C_L}}+288\sqrt{C_L}\lambda_pL_{w} \frac{1}{n}\sum_{i=1}^{n} (L^{(i)}_{\max})^2\Delta^{(i)}_F}{\sqrt{T}}  \\
	& \hspace{1cm} + \frac{\frac{27}{2}\frac{2+C_L}{C_L}\tau^2\overline{\kappa}+\frac{3(L_{D}+2L_{DQ})^2}{4C_LL^2_{w}}\tau^2\overline{\kappa}}{T} + \frac{\frac{27}{2}\frac{2+C_L}{C_L^{\frac{3}{2}}}\tau^2\overline{\kappa}}{T^{\frac{3}{2}}} + 216\overline{\kappa}.
\end{align*}
This concludes the proof.
\subsection{Proof Outline with Client Sampling}

Incorporating partial client participation and analyzing the resulting algorithm is fairly simple. Essentially only changes are in Lemma~\ref{thm2:lemma1} and Corollary~\ref{thm2:corollary diversity}, as everything before that is for local updates only. Now we give a summary of what changes:

Let $\mathcal{K}_t$ denote the set of clients that participates at time $t$, where $|\mathcal{K}_t|=K$, i.e., $K$ clients participate in the training process at any time. In this case, we define the average parameter $\bw^t$ and the gradient $\bg^t$ as the average over the respective parameters of only the active clients at time $t$; we also define $\gamma_t$ similarly. 

\begin{itemize}[leftmargin=*,topsep=0pt,nosep]
	\item {\bf Change in the proof of Lemma~\ref{thm2:lemma1}:} In the proof of Lemma~\ref{thm2:lemma1}, the second term on the RHS of the second inequality, with the above modification will be equal to $\|\frac{1}{K}\sum_{k\in\mathcal{K}_t}\nabla_{\mathbf{w}^j}F_k(\mathbf{x}_k^{j+1},\mathbf{c}_k^{j+1},\mathbf{w}^{j}) - \nabla_{\mathbf{w}^j}F_i(\mathbf{x}_i^{j+1},\mathbf{c}_i^{j+1},\mathbf{w}^{j})\|^2$. Earlier, the average was over all clients from $1$ to $n$ and this term was bounded by $\kappa_i$ using Assumption A.6. Now, we can use the Jensen's inequality (iteratively) and Assumption A.6 and bound this by $2\kappa_i+\frac{2}{K}\sum_{j\in\mathcal{K}_t}\kappa_j$. This change will propagate over until the end.
	\item {\bf Change in the proof of Corollary~\ref{thm2:corollary diversity}:} Since this is a corollary to Lemma~\ref{thm2:lemma1}, this will also see a similar change.
	\item {\bf Remaining convergence proof:} Now, continuing the exact same convergence proof and using the modified bounds of Lemma 1 and Corollary~\ref{thm2:corollary diversity} will give the bound of our algorithm with partial client participation. 
\end{itemize}
This is the modification in the entire proof.
	

	\bibliography{bibliography}

\begin{thebibliography}{10}

\bibitem{pmlr-v139-acar21a}
Durmus Alp~Emre Acar, Yue Zhao, Ruizhao Zhu, Ramon Matas, Matthew Mattina, Paul
  Whatmough, and Venkatesh Saligrama.
\newblock Debiasing model updates for improving personalized federated
  training.
\newblock In {\em Proceedings of the 38th International Conference on Machine
  Learning}, volume 139 of {\em Proceedings of Machine Learning Research},
  pages 21--31. PMLR, 18--24 Jul 2021.

\bibitem{bai2018proxquant}
Yu~Bai, Yu-Xiang Wang, and Edo Liberty.
\newblock Proxquant: Quantized neural networks via proximal operators.
\newblock In {\em International Conference on Learning Representations}, 2019.

\bibitem{basu2019qsparse}
Debraj Basu, Deepesh Data, Can Karakus, and Suhas~N. Diggavi.
\newblock Qsparse-local-sgd: Distributed {SGD} with quantization,
  sparsification and local computations.
\newblock In {\em Advances in Neural Information Processing Systems}, pages
  14668--14679, 2019.

\bibitem{Bolte13}
J{\'{e}}r{\^{o}}me Bolte, Shoham Sabach, and Marc Teboulle.
\newblock Proximal alternating linearized minimization for nonconvex and
  nonsmooth problems.
\newblock {\em Math. Program.}, 146(1-2):459--494, 2014.

\bibitem{caldas2018leaf}
Sebastian Caldas, Sai Meher~Karthik Duddu, Peter Wu, Tian Li, Jakub
  Kone{\v{c}}n{\`y}, H~Brendan McMahan, Virginia Smith, and Ameet Talwalkar.
\newblock Leaf: A benchmark for federated settings.
\newblock {\em arXiv preprint arXiv:1812.01097}, 2018.

\bibitem{dbouk2020dbq}
Hassan Dbouk, Hetul Sanghvi, Mahesh Mehendale, and Naresh Shanbhag.
\newblock Dbq: A differentiable branch quantizer for lightweight deep neural
  networks.
\newblock In {\em European Conference on Computer Vision}, pages 90--106.
  Springer, 2020.

\bibitem{deng2020adaptive}
Yuyang Deng, Mohammad~Mahdi Kamani, and Mehrdad Mahdavi.
\newblock Adaptive personalized federated learning.
\newblock {\em arXiv preprint arXiv:2003.13461}, 2020.

\bibitem{dinh2020personalized}
Canh~T. Dinh, Nguyen~H. Tran, and Tuan~Dung Nguyen.
\newblock Personalized federated learning with moreau envelopes.
\newblock In {\em Advances in Neural Information Processing Systems}, 2020.

\bibitem{fallah2020personalized}
Alireza Fallah, Aryan Mokhtari, and Asuman Ozdaglar.
\newblock Personalized federated learning: A meta-learning approach.
\newblock In {\em Advances in Neural Information Processing Systems}, 2020.

\bibitem{ghosh2020efficient}
Avishek Ghosh, Jichan Chung, Dong Yin, and Kannan Ramchandran.
\newblock An efficient framework for clustered federated learning.
\newblock In {\em Advances in Neural Information Processing Systems}, 2020.

\bibitem{gong2019differentiable}
Ruihao Gong, Xianglong Liu, Shenghu Jiang, Tianxiang Li, Peng Hu, Jiazhen Lin,
  Fengwei Yu, and Junjie Yan.
\newblock Differentiable soft quantization: Bridging full-precision and low-bit
  neural networks.
\newblock In {\em Proceedings of the IEEE/CVF International Conference on
  Computer Vision}, pages 4852--4861, 2019.

\bibitem{han2016deep}
Song Han, Huizi Mao, and William~J. Dally.
\newblock Deep compression: Compressing deep neural networks with pruning,
  trained quantization and huffman coding, 2016.

\bibitem{hanzely2020lower}
Filip Hanzely, Slavomír Hanzely, Samuel Horváth, and Peter Richtárik.
\newblock Lower bounds and optimal algorithms for personalized federated
  learning.
\newblock In {\em Advances in Neural Information Processing Systems}, 2020.

\bibitem{hanzely2020federated}
Filip Hanzely and Peter Richtárik.
\newblock Federated learning of a mixture of global and local models.
\newblock {\em arXiv preprint arXiv:2002.05516}, 2020.

\bibitem{he2015deep}
Kaiming He, Xiangyu Zhang, Shaoqing Ren, and Jian Sun.
\newblock Deep residual learning for image recognition.
\newblock In {\em Proceedings of the IEEE conference on computer vision and
  pattern recognition}, pages 770--778, 2016.

\bibitem{hinton2015distilling}
Geoffrey Hinton, Oriol Vinyals, and Jeff Dean.
\newblock Distilling the knowledge in a neural network.
\newblock {\em arXiv preprint arXiv:1503.02531}, 2015.

\bibitem{hou2017loss}
Lu~Hou, Quanming Yao, and James~T. Kwok.
\newblock Loss-aware binarization of deep networks.
\newblock In {\em International Conference on Learning Representations}, 2017.

\bibitem{karimireddy2019error}
Sai~Praneeth Karimireddy, Quentin Rebjock, Sebastian Stich, and Martin Jaggi.
\newblock Error feedback fixes signsgd and other gradient compression schemes.
\newblock In {\em International Conference on Machine Learning}, pages
  3252--3261. PMLR, 2019.

\bibitem{cifar10}
Alex Krizhevsky, Vinod Nair, and Geoffrey Hinton.
\newblock Cifar-10 (canadian institute for advanced research).
\newblock 2009.

\bibitem{leng2018extremely}
Cong Leng, Zesheng Dou, Hao Li, Shenghuo Zhu, and Rong Jin.
\newblock Extremely low bit neural network: Squeeze the last bit out with admm.
\newblock In {\em Thirty-Second AAAI Conference on Artificial Intelligence},
  2018.

\bibitem{li2019fedmd}
Daliang Li and Junpu Wang.
\newblock Fedmd: Heterogenous federated learning via model distillation.
\newblock {\em arXiv preprint arXiv:1910.03581}, 2019.

\bibitem{li2017training}
Hao Li, Soham De, Zheng Xu, Christoph Studer, Hanan Samet, and Tom Goldstein.
\newblock Training quantized nets: A deeper understanding.
\newblock In {\em Advances in Neural Information Processing Systems}, pages
  5811--5821, 2017.

\bibitem{li2020federated}
Tian Li, Anit~Kumar Sahu, Manzil Zaheer, Maziar Sanjabi, Ameet Talwalkar, and
  Virginia Smith.
\newblock Federated optimization in heterogeneous networks.
\newblock In {\em Proceedings of Machine Learning and Systems 2020, MLSys},
  2020.

\bibitem{lin2020ensemble}
Tao Lin, Lingjing Kong, Sebastian~U. Stich, and Martin Jaggi.
\newblock Ensemble distillation for robust model fusion in federated learning.
\newblock In {\em Advances in Neural Information Processing Systems}, 2020.

\bibitem{louizos2018relaxed}
Christos Louizos, Matthias Reisser, Tijmen Blankevoort, Efstratios Gavves, and
  Max Welling.
\newblock Relaxed quantization for discretized neural networks.
\newblock In {\em International Conference on Learning Representations}, 2019.

\bibitem{mansour2020approaches}
Yishay Mansour, Mehryar Mohri, Jae Ro, and Ananda~Theertha Suresh.
\newblock Three approaches for personalization with applications to federated
  learning.
\newblock {\em arXiv preprint arXiv:2002.10619}, 2020.

\bibitem{mcmahan2017communicationefficient}
Brendan McMahan, Eider Moore, Daniel Ramage, Seth Hampson, and Blaise~Aguera
  y~Arcas.
\newblock Communication-efficient learning of deep networks from decentralized
  data.
\newblock In {\em Artificial Intelligence and Statistics}, pages 1273--1282.
  PMLR, 2017.

\bibitem{ozkara2021qupel}
Kaan Ozkara, Navjot Singh, Deepesh Data, and Suhas Diggavi.
\newblock Qupel: Quantized personalization with applications to federated
  learning.
\newblock {\em arXiv preprint arXiv:2102.11786}, 2021.

\bibitem{polino2018model}
Antonio Polino, Razvan Pascanu, and Dan Alistarh.
\newblock Model compression via distillation and quantization.
\newblock In {\em International Conference on Learning Representations}, 2018.

\bibitem{Bin_survey}
Haotong Qin, Ruihao Gong, Xianglong Liu, Xiao Bai, Jingkuan Song, and Nicu
  Sebe.
\newblock Binary neural networks: A survey.
\newblock {\em Pattern Recognition}, 105:107281, Sep 2020.

\bibitem{shen2020federated}
Tao Shen, Jie Zhang, Xinkang Jia, Fengda Zhang, Gang Huang, Pan Zhou, Kun
  Kuang, Fei Wu, and Chao Wu.
\newblock Federated mutual learning.
\newblock {\em arXiv preprint arXiv:2006.16765}, 2020.

\bibitem{smith2017federated}
Virginia Smith, Chao{-}Kai Chiang, Maziar Sanjabi, and Ameet~S. Talwalkar.
\newblock Federated multi-task learning.
\newblock In {\em Advances in Neural Information Processing Systems}, pages
  4424--4434, 2017.

\bibitem{Yang_2019_CVPR}
Jiwei Yang, Xu~Shen, Jun Xing, Xinmei Tian, Houqiang Li, Bing Deng, Jianqiang
  Huang, and Xian-sheng Hua.
\newblock Quantization networks.
\newblock In {\em Proceedings of the IEEE/CVF Conference on Computer Vision and
  Pattern Recognition (CVPR)}, 2019.

\bibitem{BinaryRelax}
Penghang Yin, Shuai Zhang, Jiancheng Lyu, Stanley~J. Osher, Yingyong Qi, and
  Jack Xin.
\newblock Binaryrelax: {A} relaxation approach for training deep neural
  networks with quantized weights.
\newblock {\em {SIAM} J. Imaging Sci.}, 11(4):2205--2223, 2018.

\bibitem{zhang2021personalized}
Michael Zhang, Karan Sapra, Sanja Fidler, Serena Yeung, and Jose~M. Alvarez.
\newblock Personalized federated learning with first order model optimization.
\newblock In {\em International Conference on Learning Representations}, 2021.

\bibitem{zhang2018deep}
Ying Zhang, Tao Xiang, Timothy~M Hospedales, and Huchuan Lu.
\newblock Deep mutual learning.
\newblock In {\em Proceedings of the IEEE Conference on Computer Vision and
  Pattern Recognition}, pages 4320--4328, 2018.

\end{thebibliography}
	\bibliographystyle{plain}

	\renewcommand{\thesection}{\Alph{section}}
	\appendix
	\setcounter{equation}{0}
	\section{Preliminaries}\label{appendix:Preliminaries}
	\allowdisplaybreaks{
	\subsection{Notation}
	\begin{itemize}
		\item Given a composite function $g(\bx,\by)$ we will denote $\nabla g(\bx,\by)$  or $\nabla_{(\bx,\by)} g(\bx,\by)$   as the gradient; $\nabla_\bx g(\bx,\by)$ and $\nabla_\by g(\bx,\by)$ as the partial gradients with respect to $\bx$ and $\by$. 
		\item For a vector $\bu$, $\|\bu\|$ denotes the $\ell_2$-norm $\|\bu\|_2$. For a matrix $\bA$, $\|\bA\|_F$ denotes the Frobenius norm.
		\item Unless otherwise stated, for a given vector $\bx$, $x_i$ denotes the $i'th$ element in vector $\bx$; and $\bx_i$ denotes that the vector belongs to client $i$. Furthermore, $\bx^t_i$ denotes a vector that belongs to client $i$ at time $t$.
	\end{itemize}
	
	\subsection{Equivalence of Assumption \textbf{A.6} to Assumptions in Related Work}
	In particular, diversity assumption (Assumption 5) in \cite{fallah2020personalized} is as follows:
	\begin{align*}
		\frac{1}{n} \sum_{i=1}^n \Big\| \nabla_{\bw} f_i(\bw) -\nabla_{\bw} f(\bw) \Big\|^2 \leq B,
	\end{align*}
	where $f_i$ is local function, $B$ is a constant and $\nabla_{\bw} f(\bw) = \frac{1}{n} \sum_{j=1}^n \nabla_{\bw} f_i(\bw)$.
	Now we will show the equivalence to our stated assumption {\bf A.6}. Let us define,
	\begin{align*}
		x_i(\bw^t) &:= \underset{\bx \in \mathbb{R}^{d}}{\arg \min }\left\{\left\langle \bx-\bx^t_{i}, \nabla f_i\left(\bx^t_{i}\right)\right\rangle+\left\langle \bx-\bx^t_{i}, \nabla_{\bx^t_{i}} f_i(\widetilde{Q}_{\bc^t_{i}}(\bx^t_{i}))\right\rangle +\left\langle \bx-\bx^t_{i}, \lambda_p (\bx^t_{i}-\bw^{t})\right\rangle \right.\\
		& \quad \left. +\frac{1}{2 \eta_1}\left\|\bx-\bx^t_{i}\right\|_{2}^{2}+\lambda R(\bx,\bc^t_{i})\right\} \\
		c_i(\bw^t) &:= \underset{\bc \in \mathbb{R}^{m}}{\arg \min }\left\{\left\langle \bc-\bc^t_{i}, \nabla_{\bc^t_{i}} f_i(\widetilde{Q}_{\bc^t_{i}}(x_i(\bw^t)))\right\rangle  \right.  \left. +\frac{1}{2 \eta_2}\left\|\bc-\bc^t_{i}\right\|_{2}^{2}+\lambda R(x_i(\bw^t),\bc)\right\}
	\end{align*}
	Then we can define,
	\begin{align*}
		\psi_i(x_i(\bw^t),c_i(\bw^t),\bw^t):=F_i(\bx^{t+1}_i,\bc^{t+1}_i,\bw^t)
	\end{align*}
	as a result, we can further define $g_i(\bw^t):=\psi_i(x_i(\bw^t),c_i(\bw^t),\bw^t)$. Therefore, our assumption \textbf{A.6} is equivalent to stating the following assumption:
	At any $t \in [T]$ and any client $i\in[n]$, the variance of the local gradient (at client $i$) w.r.t. the global gradient is bounded, i.e., there exists $\kappa_i<\infty$, such that for every $\bw^t\in\R^d$, we have:
	\begin{align*}
		\Big\| \nabla_{\bw^t} g_i(\bw^t) - \frac{1}{n} \sum_{j=1}^n \nabla_{\bw^t} g_j(\bw^t) \Big\|^2 \leq \kappa_i,
	\end{align*}
	And we also define $\kappa := \frac{1}{n} \sum_{i=1}^n \kappa_i$ and then,
	\begin{align*}
		\frac{1}{n} \sum_{i=1}^n \Big\| \nabla_{\bw^t} g_i(\bw^t) -\nabla_{\bw^t} g(\bw^t) \Big\|^2 \leq \kappa,
	\end{align*}
	here $\nabla_{\bw^t} g(\bw^t) = \frac{1}{n} \sum_{j=1}^n \nabla_{\bw^t} g_j(\bw^t)$. Hence, our assumption is equivalent to assumptions that are found in aforementioned works.
	
	\subsection{Alternating Proximal Steps}
	We define the following functions: $f:\mathbb{R}^d \rightarrow \mathbb{R}, \widetilde{Q}: \mathbb{R}^{d+m} \rightarrow \mathbb{R}^d$, and we also define $h(\bx,\bc) = f(\widetilde{Q}_\bc(\bx)), h:\mathbb{R}^{d+m} \rightarrow \mathbb{R}$ where $\bx \in \mathbb{R}^d \text{ and } \bc \in \mathbb{R}^m$. Note that here $\widetilde{Q}_\bc(\bx)$ denotes  $\widetilde{Q}(\bx,\bc)$. Throughout our paper we will use $\widetilde{Q}_\bc(\bx)$ to denote  $\widetilde{Q}(\bx,\bc)$, in other words both $\bc$ and $\bx$ are inputs to the function $\widetilde{Q}_\bc(\bx)$. We propose an alternating proximal gradient algorithm. Our updates are as follows:
	%
	
	\begin{align} \label{updates}
		&\bx^{t+1} = \text{prox}_{\eta_1\lambda R(\cdot,\bc^{t})}(\bx^{t} - \eta_1 \nabla f(\bx^{t})-\eta_1 \nabla_{\bx^{t}} f(\widetilde{Q}_{\bc^{t}}(\bx^{t})) )\\
		&\bc^{t+1} = \text{prox}_{\eta_2\lambda R(\bx^{t+1},\cdot)}(\bc^{t} - \eta_2 \nabla_{\bc^{t}} f(\widetilde{Q}_{\bc^{t}}(\bx^{t+1}))) \notag
	\end{align}
	
	For simplicity we assume the functions in the objective function are differentiable, however, our analysis could also be done using subdifferentials.
	
	Our method is inspired by \cite{Bolte13} where the authors introduce an alternating proximal minimization algorithm to solve a broad class of non-convex problems as an alternative to coordinate descent methods. In this work we construct another optimization problem that can be used as a surrogate in learning quantized networks where both model parameters and quantization levels are subject to optimization. In particular, \cite{Bolte13} considers a general objective function of the form $F(\bx,\by) = f(\bx)+g(\by)+\lambda H(\bx,\by)$, whereas, our objective function is tailored for learning quantized networks: $F_\lambda(\bx,\bc) = f(\bx)+f(\widetilde{Q}_{\bc}(\bx))+\lambda R(\bx,\bc)$. Furthermore, they consider updates where the proximal mappings are with respect to functions $f,g$, whereas in our case the proximal mappings are with respect to the distance function $R(\bx,\bc)$ to capture the soft projection. 
	
	\subsection{A Soft Quantization Function}
	In this section we give an example of the soft quantization function that can be used in previous sections. In particular, we can define the following soft quantization function: $\widetilde{Q}_\bc(\bx):\mathbb{R}^{d+m}\rightarrow\mathbb{R}^d$ and $\widetilde{Q}_\bc(\bx)_i := \sum_{j=2}^m(c_j-c_{j-1})\sigma(P(x_i-\frac{c_j+c_{j-1}}{2}))+c_1$ where $\sigma$ denotes the sigmoid function and $P$ is a parameter controlling how closely $\widetilde{Q}_\bc(\bx)$ approximates $Q_\bc(\bx)$. Note that as $P \rightarrow \infty$, $\widetilde{Q}_\bc(\bx)\rightarrow Q_\bc(\bx)$. This function can be seen as a simplification of the function that was used in \cite{Yang_2019_CVPR}.
	
	\noindent\textbf{Assumption.}  For all $j\in[m]$, $c_j$ is in a compact set. In other words, there exists a finite $c_{\max}$ such that $|c_j| \leq c_{\max}$ for all $j\in[m]$. 
	
	In addition, we assume that the centers are sorted, i.e., $c_1 < \cdots < c_m$. Now, we state several useful facts.
	\begin{fact}\label{sq:fact1}
		$\widetilde{Q}_\bc(\bx)$ is continuously and infinitely differentiable everywhere.
	\end{fact}
	\begin{fact}\label{sq:fact2}
		$\sigma(\bx)$ is a Lipschitz continuous function.
	\end{fact}
	\begin{fact}\label{sq:fact3}
		Sum of Lipschitz continuous functions is also Lipschitz continuous. 
	\end{fact}
	\begin{fact}\label{sq:fact4}
		Product of bounded and Lipschitz continuous functions is also Lipschitz continuous.
	\end{fact}
	\begin{fact}\label{sq:fact5}
		Let $g:\mathbb{R}^n \rightarrow \mathbb{R}^m$. Then, the coordinate-wise Lipschitz continuity implies overall Lipschitz continuity. In other words, let $g_i$ be the i'th output then if $g_i$ is Lipschitz continuous for all $i$, then $g$ is also Lipschitz continuous.
	\end{fact}
	In our convergence analysis, we require that $\widetilde{Q}_\bc(\bx)$ is Lipschitz continuous as well as smooth with respect to both $\bx$ and $\bc$. 
	\begin{claim} \label{claimlQ_1}
		$\widetilde{Q}_\bc(\bx)$ is $l_{Q_1}$-Lipschitz continuous and $L_{Q_1}$-smooth with respect to $\bx$.
	\end{claim}
	\begin{proof}
		First we prove Lipschitz continuity. Note,
		\begin{align} \label{grad1}
			\frac{\partial \widetilde{Q}_\bc(\bx)_i}{\partial x_j} = 
			\begin{cases}
				0, \text{ if } i \neq j\\ 
				P\sum_{j=2}^m(c_j-c_{j-1})\sigma(P(x_i-\frac{c_j + c_{j-1}}{2}))(1-\sigma(P(x_i-\frac{c_j + c_{j-1}}{2}))),  \text{ if } i = j\\ 
			\end{cases}
		\end{align}
		As a result, $\| \frac{\partial \widetilde{Q}_\bc(\bx)_i}{\partial x_j} \| \leq \frac{P}{4}(c_m-c_1)\leq \frac{P}{2}c_{max}$. The norm of the gradient of $\widetilde{Q}_\bc(\bx)_i$ with respect to $x$ is bounded which implies there exists $l^{(i)}_{Q_1}$ such that $\|\widetilde{Q}_\bc(\bx)_i-\widetilde{Q}_\bc(\bx')_i\|\leq l^{(i)}_{Q_1}\|\bx-\bx'\|$; using Fact~\ref{sq:fact5} and the fact that $i$ was arbitrary, there exists $l_{Q_1}$ such that $\|\widetilde{Q}_\bc(\bx)-\widetilde{Q}_\bc(\bx')\|\leq l_{Q_1}\|\bx-\bx'\|$. In other words, $\widetilde{Q}_\bc(\bx)$ is Lipschitz continuous.
		
		For smoothness note that, $\nabla_\bx \widetilde{Q}_\bc(\bx) = \nabla \widetilde{Q}_\bc(\bx)_{1:d,:}$. Now we focus on an arbitrary term of $\nabla_\bx \widetilde{Q}_\bc(\bx)_{j,i}$. From \eqref{grad1} we know that this term is 0 if $i \neq j$, and a weighted sum of product of sigmoid functions if $i = j$. Then, using the Facts~\ref{sq:fact1}-\ref{sq:fact4} the function $\nabla_\bx \widetilde{Q}_\bc(\bx)_{j,i}$ is Lipschitz continuous. Since $i,j$ were arbitrarily chosen, $\nabla_\bx \widetilde{Q}_\bc(\bx)_{j,i}$ is Lipschitz continuous for all $i,j$. Then, by Fact~\ref{sq:fact5},  $\nabla_\bx \widetilde{Q}_\bc(\bx)$ is Lipschitz continuous, which implies that $\widetilde{Q}_\bc(\bx)$ is $L_{Q_1}$-smooth for some coefficient $L_{Q_1} < \infty$.
	\end{proof}
	\begin{claim}  \label{claimlQ_2}
		$\widetilde{Q}_\bc(\bx)$ is $l_{Q_2}$-Lipschitz continuous and $L_{Q_2}$-smooth with respect to $\bc$.
	\end{claim}
	\begin{proof}
		For Lipschitz continuity we have,
		\begin{align*}
			\frac{\partial \widetilde{Q}_\bc(\bx)_i}{\partial c_j} & = 
			\sigma(P(x_i-\frac{c_j + c_{j-1}}{2}))-\sigma(P(x_i-\frac{c_j + c_{j+1}}{2}))\\
			& \quad +(c_j-c_{j+1})\frac{P}{2}\sigma(P(x_i-\frac{c_j + c_{j+1}}{2}))(1-\sigma(P(x_i-\frac{c_j + c_{j+1}}{2}))) \\
			& \quad - (c_j-c_{j-1})\frac{P}{2}\sigma(P(x_i-\frac{c_j + c_{j-1}}{2}))(1-\sigma(P(x_i-\frac{c_j + c_{j-1}}{2})))
		\end{align*}
		As a result, $\| \frac{\partial \widetilde{Q}_\bc(\bx)_i}{\partial c_j} \| \leq 2 + c_{max}\frac{P}{2}$. Similar to Claim~\ref{claimlQ_1} using the facts that $i$ is arbitrary and the Fact~\ref{sq:fact5}, we find there exists $l_{Q_2}$ such that $\|\widetilde{Q}_\bc(\bx)-\widetilde{Q}_\bd(\bx)\|\leq l_{Q_1}\|\bc-\bd\|$. In other words, $\widetilde{Q}_\bc(\bx)$ is Lipschitz continuous. And for the smoothness, following the same idea from the proof of Claim \ref{claimlQ_2} we find $\widetilde{Q}_\bc(\bx)$ is $L_{Q_2}$-smooth with respect to $\bc$.
	\end{proof}
	The example we gave in this section is simple yet provides technical necessities we require in the analysis. Other examples can also be used as long as they provide the smoothness properties that we utilize in the next sections. 
	
	\subsection{Lipschitz Relations}
	In this section we will use the assumptions \textbf{A.1-5} and show useful relations for partial gradients derived from the assumptions. We have the following gradient for the composite function:
	
	\begin{align}
		\nabla_{(\bx,\bc)} h(\bx,\bc) = \nabla_{(\bx,\bc)} f(\widetilde{Q}_\bc(\bx)) = \nabla_{(\bx,\bc)} \widetilde{Q}_\bc(\bx) \nabla_{\widetilde{Q}_\bc(\bx)} f(\widetilde{Q}_\bc(\bx)) 
	\end{align}
	
	where $dim(\nabla h(\bx,\bc)) = (d+m) \times 1, \ dim(\nabla_{\widetilde{Q}_\bc(\bx)} f(\widetilde{Q}_\bc(\bx))) = d \times 1, \ dim(\nabla \widetilde{Q}_\bc(\bx)) = (d+m) \times d $. Note that the soft quantization functions of our interest are elementwise which implies $\frac{\partial \widetilde{Q}_\bc(\bx)_i}{\partial x_j}=0$ if $i \neq j$. In particular,
	for the gradient of the quantization function we have,
	
	\begin{align}\label{lipschitz_relations2}
		\nabla_{(\bx,\bc)} \widetilde{Q}_\bc(\bx) = \begin{bmatrix} \frac{\partial \widetilde{Q}_\bc(\bx)_1}{\partial x_1} & 0 & \hdots  \\ 0 & \frac{\partial \widetilde{Q}_\bc(\bx)_2}{\partial x_2} & 0\hdots \\ \vdots & \vdots & \vdots \\ \frac{\partial \widetilde{Q}_\bc(\bx)_1}{\partial c_1} & \frac{\partial \widetilde{Q}_\bc(\bx)_2}{\partial c_1} & \hdots \\ \vdots & \vdots & \vdots \\ \frac{\partial \widetilde{Q}_\bc(\bx)_1}{\partial c_m} &  \frac{\partial \widetilde{Q}_\bc(\bx)_2}{\partial c_m} & \hdots \end{bmatrix}
	\end{align}
	
	Moreover for the composite function we have,
	
	\begin{align} \label{lipschitz_relations1}
		\nabla h(\bx,\bc) = \begin{bmatrix} &  \frac{\partial f}{\partial \widetilde{Q}_\bc(\bx)_1} \frac{\partial \widetilde{Q}_\bc(\bx)_1}{\partial  x_1} &+  0 & +  \hdots  \\ & 0 &+  \frac{\partial f}{\partial \widetilde{Q}_\bc(\bx)_2} \frac{\partial \widetilde{Q}_\bc(\bx)_2}{\partial  x_2} &+  \hdots \\ &\vdots \\& \frac{\partial f}{\partial \widetilde{Q}_\bc(\bx)_1} \frac{\partial \widetilde{Q}_\bc(\bx)_1}{\partial  c_1} &+ \frac{\partial f}{\partial \widetilde{Q}_\bc(\bx)_2} \frac{\partial \widetilde{Q}_\bc(\bx)_2}{\partial  c_1} &+ \hdots \\ &\vdots \\ &\frac{\partial f}{\partial \widetilde{Q}_\bc(\bx)_1} \frac{\partial \widetilde{Q}_\bc(\bx)_1}{\partial  c_m} &+ \frac{\partial f}{\partial \widetilde{Q}_\bc(\bx)_2} \frac{\partial \widetilde{Q}_\bc(\bx)_2}{\partial  c_m} &+ \hdots \end{bmatrix} 
	\end{align}
	
	In \eqref{lipschitz_relations1} and \eqref{lipschitz_relations2}  we use $x_i$, $c_j$ to denote $(\bx)_i$ and $(\bc)_j$ (i'th, j'th element respectively) for notational simplicity. Now, we prove two claims that will be useful in the main analysis.
	\begin{claim} \label{claim: lqxclaim}
		\begin{align*}
			\| \nabla_\bx f(\widetilde{Q}_\bc(\bx))-\nabla_\by f(\widetilde{Q}_\bc(\by)) \| =\| \nabla h(\bx,\bc)_{1:d}-\nabla h(\by,\bc)_{1:d} \| \leq  (GL_{Q_1} + G_{Q_1}Ll_{Q_1})\|\bx-\by\|  
		\end{align*}
	\end{claim}
	\begin{proof}
		\begin{align*}
			\| \nabla h(\bx,\bc)_{1:d}-\nabla h(\by,\bc)_{1:d} \| &=\|\nabla_\bx f(\widetilde{Q}_{\bc}(\bx))-\nabla_\by f(\widetilde{Q}_{\bc}(\by))\| \\  &= \| \nabla_{\widetilde{Q}_\bc(\bx)} f(\widetilde{Q}_\bc(\bx)) \nabla_\bx \widetilde{Q}_\bc(\bx) - \nabla_{\widetilde{Q}_\bc(\by)} f(\widetilde{Q}_\bc(\by)) \nabla_\by \widetilde{Q}_\bc(\by)\| \\
			&=\| \nabla_{\widetilde{Q}_\bc(\bx)} f(\widetilde{Q}_\bc(\bx)) \nabla_\bx \widetilde{Q}_\bc(\bx)- \nabla_{\widetilde{Q}_\bc(\bx)} f(\widetilde{Q}_\bc(\bx))  \nabla_\by \widetilde{Q}_\bc(\by) \\ & \quad + \nabla_{\widetilde{Q}_\bc(\bx)} f(\widetilde{Q}_\bc(\bx)) \nabla_\by \widetilde{Q}_\bc(\by) -  \nabla_{\widetilde{Q}_\bc(\by)} f(\widetilde{Q}_\bc(\by))  \nabla_\by \widetilde{Q}_\bc(\by)\| \\ &\stackrel{\text{(a)}}{\leq} \| \nabla_{\widetilde{Q}_\bc(\bx)} f(\widetilde{Q}_\bc(\bx)) \| \| \nabla_\bx \widetilde{Q}_\bc(\bx) - \nabla_\by \widetilde{Q}_\bc(\by) \|_F \\ & \quad + \| \nabla_\by \widetilde{Q}_\bc(\by) \|_F \| \nabla_{\widetilde{Q}_\bc(\bx)} f(\widetilde{Q}_\bc(\bx)) - \nabla_{\widetilde{Q}_\bc(\by)} f(\widetilde{Q}_\bc(\by)) \|  \\ & \leq GL_{Q_1}\|\bx-\by\| + G_{Q_1}L\|\widetilde{Q}_\bc(\bx)-\widetilde{Q}_\bc(\by)\| \\ &\leq GL_{Q_1}\|\bx-\by\| + G_{Q_1}Ll_{Q_1}\|\bx-\by\| \\
			& = (GL_{Q_1} + G_{Q_1}Ll_{Q_1})\|\bx-\by\|    
		\end{align*} 
		To obtain (a) we have used the fact $\|\bA\bx\|_2 \leq \|\bA\|_F\|\bx\|_2$.
	\end{proof}

	\begin{claim} \label{claim: lqcclaim}
		\begin{align*}
			\| \nabla_\bc f(\widetilde{Q}_\bc(\bx))-\nabla_\bd f(\widetilde{Q}_\bd(\bx)) \| &= \| \nabla h(\bx,\bc)_{d+1:m}-\nabla h(\bx,\bd)_{d+1:m} \| \\
			&\leq (GL_{Q_2} + G_{Q_2}Ll_{Q_2})\|\bc-\bd\| 
		\end{align*}
	\end{claim}
	\begin{proof} We can follow similar steps,
		\begin{align*}
			\| \nabla h(\bx,\bc)_{d+1:m}{-}\nabla h(\bx,\bd)_{d+1:m} \|&=\|\nabla_\bc f(\widetilde{Q}_{\bc}(\bx)){-}\nabla_\bd f(\widetilde{Q}_{\bd}(\bx))\| \\ &= \| \nabla_{\widetilde{Q}_\bc(\bx)} f(\widetilde{Q}_\bc(\bx)) \nabla_\bc \widetilde{Q}_\bc(\bx) {-} \nabla_{\widetilde{Q}_\bd(\by)} f(\widetilde{Q}_\bd(\by)) \nabla_\bd \widetilde{Q}_\bd(\by)\| \\
			&=\| \nabla_{\widetilde{Q}_\bc(\bx)} f(\widetilde{Q}_\bc(\bx)) \nabla_\bc \widetilde{Q}_\bc(\bx)- \nabla_{\widetilde{Q}_\bc(\bx)} f(\widetilde{Q}_\bc(\bx)) \nabla_\bd \widetilde{Q}_\bd(\bx) \\ & \quad + \nabla_{\widetilde{Q}_\bc(\bx)} f(\widetilde{Q}_\bc(\bx)) \nabla_\bd \widetilde{Q}_\bd(\by) {-}  \nabla_{\widetilde{Q}_\bd(\bx)} f(\widetilde{Q}_\bd(\by)) \nabla_\bd \widetilde{Q}_\bd(\by)\| \\ &\leq \| \nabla_{\widetilde{Q}_\bc(\bx)} f(\widetilde{Q}_\bc(\bx)) \| \| \nabla_\bc \widetilde{Q}_\bc(\bx) - \nabla_\bd \widetilde{Q}_\bd(\by) \|_F \\ & \quad + \| \nabla_\bd \widetilde{Q}_\bd(\by) \|_F \| \nabla_{\widetilde{Q}_\bc(\bx)} f(\widetilde{Q}_\bc(\bx)) {-} \nabla_{\widetilde{Q}_\bd(\by)} f(\widetilde{Q}_\bd(\by)) \|  \\ & \leq GL_{Q_2}\|\bc-\bd\| + G_{Q_2}L\|\widetilde{Q}_\bc(\bx)-\widetilde{Q}_\bd(\bx)\| \\ &\leq GL_{Q_2}\|\bc-\bd\| + G_{Q_2}Ll_{Q_2}\|\bc-\bd\| \\
			&= (GL_{Q_2} + G_{Q_2}Ll_{Q_2})\|\bc-\bd\|     
		\end{align*} 
		where $\nabla_\bc \widetilde{Q}_\bc(\bx) = \nabla \widetilde{Q}_\bc(\bx)_{(d+1:d+m,:)}$. 
	\end{proof}
	
\subsection{Assumption \textbf{A.7} is a Corollary of other Assumptions}

In this section we discuss how Assumption \textbf{A.7} can be inferred from \textbf{A.1}-\textbf{A.5}. Here we drop client indices $i$ for notational simplicity. Let us define $f_w(\bw)$ as the neural network loss function with model $\bw$. First we will argue that $f^{KD}(\bw,\bx)$ is smooth, given that $f_w(\bw)$ and $f(\bx)$ are two smooth neural network loss functions with Cross Entropy as the loss function, and that $f_w(\bw)$ and $f(\bx)$ have bounded gradients. These two standard assumptions imply the smoothness of $f^{KD}(\bw,\bx)$ individually with respect to both input parameters.
\begin{proposition} \label{proposition 1}
	$f^{KD}(\bw,\bx)$ is $L_{D_1}$-smooth with respect to $\bx$ and $L_{D_2}$-smooth with respect to $\bw$ for some positive constants $L_{D_1}, L_{D_2}$.
\end{proposition}

\begin{proof}
	Note that $f(\bx) = \frac{1}{N} \sum_{i=1}^N \by_i^{T} \log(\frac{1}{s(\bx;\xi_i)})$ where $i$ denotes the index of data sample, $\by_i$ is the one hot encoding label vector, $N$ is the total number of data samples, $\log$ denotes elementwise logarithm and $\frac{1}{s(\bx;\xi_i)}$ denotes elementwise inverse; softmax function $s(\bx): \mathbb{R}^{d}\rightarrow \mathbb{R}^{K}$, where $K$ denotes the number of classes (similarly $s^{w}(\bw): \mathbb{R}^{d}\rightarrow \mathbb{R}^{K}$ is the function whose output is a vector of softmax probabilities and input is global model), is defined in Section~\ref{sec:problem}, here we explicitly state that data samples $\xi_i$ is a parameterization of $s$. Assuming $f(\bx)$ is smooth for any possible pair of $(\by_i, \xi_i)$ implies $\log(\frac{1}{s(\bx)_j})$ is $C_x$-smooth for some constant $C_x$, here we used $s(\bx)_j$ to denote $j$'th output of $s(\bx)$ and we omitted $\xi_i$ since $\log(\frac{1}{s(\bx)_j})$ is smooth independent of $\xi_i$. Note that $s(\bx)_j:\mathbb{R}^{d}\rightarrow\mathbb{R}$.
	We have,
	 \begin{align*}
	 	f^{KD}(\bw,\bx) = \frac{1}{N} \sum_{i=1}^N (s^w(\bw;\xi_i))^{T} \log(\frac{s^w(\bw;\xi_i)}{s(\bx;\xi_i)}) &=\frac{1}{N} \sum_{i=1}^N (s^w(\bw;\xi_i))^{T} \log(s^w(\bw;\xi_i))\\
	 	& \hspace{1cm} + \frac{1}{N}\sum_{i=1}^N (s^w(\bw;\xi_i))^{T} \log(\frac{1}{s(\bx;\xi_i)})
	 \end{align*} 
	where the operations are elementwise as before. In this expression only the last term depends on $\bx$ and for each $i$, since $s^w(\bw;\xi_i)_j\leq1$ with $\sum_{j}s^w(\bw;\xi_i)_j=1$, the expression is a weighted average of smooth functions $\log(\frac{1}{s(\bx)_j})$; as a result, $f^{KD}(\bw,\bx)$ is $L_{D_1}$-smooth with respect to $\bx$ for some constant $L_{D_1}$.\\
	
	Now we investigate smoothness with respect to $\bw$. First, note that we can assume for all $j$ and $\bw$, $s^w(\bw)_j$ is lower bounded by a positive constant $M>0$ and upper bounded by a positive constant $P<1$ since, by definition, output vector of the softmax function contains values between 0 and 1 (we ignore the limiting case when a logit is infinitely large). Then, note that assuming a gradient bound on $f_w(\bw) = \frac{1}{N} \sum_{i=1}^N \by_i^{T} \log(\frac{1}{s(\bw;\xi_i)})$ implies that for all $j$ $\|\nabla \log(\frac{1}{s^w(\bw)_j})\| = \|\frac{\nabla s^w(\bw)_j}{s^w(\bw)_j}\| \leq G_w $ for some constant $G_w$ (again the division operation is elementwise); since $0 < s^w(\bw)_j < 1$ we have $\|\nabla s^w(\bw)_j\| \leq G_w $. Moreover, similar to the first part, by assuming $f_w(\bw)$ is smooth we obtain that $\log(\frac{1}{s^w(\bw)_j})$ is smooth for all $j$. This implies having a bounded Hessian:
	\begin{align*}
		\text{ for some constant $C$ we have } C \geq \Big\|\nabla^2 \log(\frac{1}{s^w(\bw)_j})\Big\|_F &= \Big\|\frac{\nabla s^w(\bw)_j \nabla s^w(\bw)_j^T}{s^w(\bw)_j^2}-\frac{\nabla^2s^w(\bw)_j}{s^w(\bw)_j}\Big\|_F 	\\
		& \stackrel{(a)}{\geq} \Big| \Big\|\frac{\nabla s^w(\bw)_j \nabla s^w(\bw)_j^T}{s^w(\bw)_j}\Big\|_F - \Big\|\nabla^2s^w(\bw)_j\Big\|_F  \Big|	\\
		& = \Big| \frac{G_w^2}{s^w(\bw)_j} - \Big\|\nabla^2s^w(\bw)_j\Big\|_F \Big| \\ 
		&\geq  \Big\|\nabla^2s^w(\bw)_j\Big\|_F - \frac{G_w^2}{s^w(\bw)_j}  \\ 
		& \geq \Big\|\nabla^2s^w(\bw)_j\Big\|_F - \frac{G_w^2}{M} 		
 	\end{align*}
where (a) is due to reverse triangular inequality and $0< s^w(\bw)_j<1$. As a result we obtain $C+\frac{G_w^2}{M}  \geq  \Big\|\nabla^2s^w(\bw)_j\Big\|_F$ for all $j$, i.e., $s^w(\bw)_j$ is $L_{S}$-smooth with some constant $L_{S} \leq C+\frac{G_w^2}{M}$. Thus, both $s^w(\bw)_j$ and $\log(s^w(\bw)_j)$ are smooth functions. Note both $s^w(\bw)_j$ and $\log(s^w(\bw)_j)$ are bounded functions. Consequently, the first summation term in the definition of $f^{KD}(\bw,\bx)$ consists of the sum of product of bounded and smooth functions and the second term consists of sum of smooth functions multiplied with positive constants (as $\log(\frac{1}{s(\bx;\xi_i)})$ does not depend on $\bw$). Using Fact~\ref{sq:fact3} and Fact~\ref{sq:fact4} we conclude there exists a constant $L_{D_2}$ such that $f^{KD}(\bw,\bx)$ is $L_{D_2}$-smooth w.r.t $\bw$.
\end{proof}

\begin{proposition}
	$f^{KD}(\bw,\widetilde{Q}_{\bc}(\bx))$ is $L_{DQ_1}$-smooth with respect to $\bx$, $L_{DQ_2}$-smooth with respect to $\bc$, and $L_{DQ_3}$-smooth with respect to $\bw$ for some constants $L_{DQ_1}, L_{DQ_2}, L_{DQ_3}$.
\end{proposition}
	
\begin{proof}
	The proof is exactly the same as in Proposition~\ref{proposition 1}, instead of using smoothness of $f(\bx)$, using smoothness of $f(\widetilde{Q}_{\bc}(\bx))$ with respect to $\bx, \bc$ from Claims~\ref{claim: lqxclaim} and \ref{claim: lqcclaim} gives the result.
\end{proof}

	
	\section{Omitted Details in Proof of Theorem 1}\label{appendix:proof of theorem 1}
	First we derive the optimization problems that the alternating updates correspond to. Remember we had the following alternating updates:
\begin{align*}
	&\bx^{t+1} = \text{prox}_{\eta_1\lambda R_{\bc^{t}}}(\bx^{t} - \eta_1 \nabla f(\bx^{t})-\eta_1 \nabla_{\bx^{t}} f(\widetilde{Q}_{\bc^{t}}(\bx^{t})) )\\
	&\bc^{t+1} = \text{prox}_{\eta_2\lambda R_{\bx^{t+1}}}(\bc^{t} - \eta_2 \nabla_{\bc^{t}} f(\widetilde{Q}_{\bc^{t}}(\bx^{t+1})))
\end{align*}
For $\bx^{t+1}$, from the definition of proximal mapping we have:
\begin{align}\label{app:quantized argmin1}
	\bx^{t+1}&=\underset{\bx \in \mathbb{R}^{d}}{\arg \min }\left\{\frac{1}{2 \eta_1}\left\|\bx-\bx^{t}+\eta_1 \nabla_{\bx^{t}} f\left(\bx^{t}\right)+\eta_1 \nabla_{\bx^{t}} f(\widetilde{Q}_{\bc^{t}}(\bx^{t}))\right\|_{2}^{2}+\lambda R(\bx,\bc^{t})\right\}\nonumber \\\nonumber
	&=\underset{\bx \in \mathbb{R}^{d}}{\arg \min } \Big\{ \left\langle \bx-\bx^{t}, \nabla_{\bx^{t}} f\left(\bx^{t}\right)\right\rangle+\left\langle \bx-\bx^{t}, \nabla_{\bx^{t}} f(\widetilde{Q}_{\bc^{t}}(\bx^{t}))\right\rangle+\frac{1}{2 \eta_1}\left\|\bx-\bx^{t}\right\|_{2}^{2} \notag \\
	& \hspace{2cm}+\frac{\eta_1}{2}\|\nabla_{\bx^{t}} f(\bx^{t})+\nabla_{\bx^{t}} f(\widetilde{Q}_{\bc^{t}}(\bx^{t}))\|^2 +\lambda R(\bx,\bc^{t})\Big\} \notag \\
	&=\underset{\bx \in \mathbb{R}^{d}}{\arg \min }\Big\{\left\langle \bx-\bx^{t}, \nabla_{\bx^{t}} f\left(\bx^{t}\right)\right\rangle+\left\langle \bx-\bx^{t}, \nabla_{\bx^{t}} f(\widetilde{Q}_{\bc^{t}}(\bx^{t}))\right\rangle+\frac{1}{2 \eta_1}\left\|\bx-\bx^{t}\right\|_{2}^{2} \notag \\
	& \hspace{2cm} +\lambda R(\bx,\bc^{t})\Big\}
\end{align}
Note, in the third equality we remove the terms that do not depend on $\bx$. Similarly, for $\bc^{t+1}$ we have:
\begin{align} \label{app:quantized argmin2}
	\bc^{t+1}&=\underset{\bc \in \mathbb{R}^{m}}{\arg \min }\left\{\frac{1}{2 \eta_2}\left\|\bc-\bc^{t}+\eta_2 \nabla_{\bc^{t}} f(\widetilde{Q}_{\bc^{t}}(\bx^{t+1}))\right\|_{2}^{2}+\lambda R(\bx^{t+1},\bc)\right\}\nonumber \\ \nonumber
	&= \underset{\bc \in \mathbb{R}^{m}}{\arg \min } \Big\{ \left\langle \bc-\bc^{t}, \nabla_{\bc^{t}} f(\widetilde{Q}_{\bc^{t}}(\bx^{t+1}))\right\rangle+\frac{1}{2 \eta_2}\left\|\bc-\bc^{t}\right\|_{2}^{2}+\frac{\eta_2}{2}\|\nabla_{\bc^{t}} f(\widetilde{Q}_{\bc^{t}}(\bx^{t+1}))\|^2  \notag \\
	& \hspace{2cm}+ \lambda R(\bx^{t+1},\bc)  \Big\} \notag \\
	&= \underset{\bc \in \mathbb{R}^{m}}{\arg \min }\left\{\left\langle \bc-\bc^{t}, \nabla_{\bc^{t}} f(\widetilde{Q}_{\bc^{t}}(\bx^{t+1}))\right\rangle+\frac{1}{2 \eta_2}\left\|\bc-\bc^{t}\right\|_{2}^{2}+\lambda R(\bx^{t+1},\bc)\right\}
\end{align}
Minimization problems in \eqref{app:quantized argmin1} and \eqref{app:quantized argmin2} are the main problems to characterize the update rules and we use them in multiple places throughout the section.

{\begin{claim*}[Restating Claim \ref{claim: lqxsmooth}]
		$f(\bx)+ f(\widetilde{Q}_\bc(\bx))$ is $(L+GL_{Q_1}+G_{Q_1}LL_{Q_1})$-smooth with respect to $\bx$.
\end{claim*}}
\begin{proof}
	From our assumptions, we have $f$ is $L$-smooth. And from Claim \ref{claim: lqxclaim} we have $f(\widetilde{Q}_\bc(\bx))$ is $(GL_{Q_1}+G_{Q_1}LL_{Q_1})$-smooth. Using the fact that if two functions $g_1$ and $g_2$ are $L_1$ and $L_2$ smooth respectively, then $g_1+g_2$ is $(L_1+L_2)$-smooth concludes the proof.
\end{proof}

{\begin{claim*}[Restating Claim \ref{claim:quantization lower bound 1}]
		Let
		\begin{align*}
			A(\bx^{t+1}) &:= \lambda R(\bx^{t+1},\bc^{t})+\left\langle \nabla f(\bx^{t}), \bx^{t+1}-\bx^{t}\right\rangle 
			+ \left\langle \nabla_{\bx^{t}} f(\widetilde{Q}_{\bc^{t}}(\bx^{t})), \bx^{t+1}-\bx^{t}\right\rangle \notag \\
			& \hspace{2cm}+ \frac{1}{2\eta_1}\|\bx^{t+1}-\bx^{t}\|^2 \notag \\
			A(\bx^{t}) &:= \lambda R(\bx^{t},\bc^{t}).
		\end{align*} 
		Then $A(\bx^{t+1})\leq A(\bx^{t})$.
\end{claim*}}
\begin{proof}
	Let $A(\bx)$ denote the expression inside the $\arg\min$ in \eqref{app:quantized argmin1} and we know that \eqref{app:quantized argmin1} is minimized when $\bx=\bx^{t+1}$. So we have $A(\bx^{t+1})\leq A(\bx^t)$. This proves the claim.
\end{proof}
{\begin{claim*}[Restating Claim \ref{claim:quantization lower bound 2}]
		Let
		\begin{align*}
			B(\bc^{t+1}) &:= \lambda R(\bx^{t+1},\bc^{t+1}) + \left\langle \nabla_{\bc^{t}} f(\widetilde{Q}_{\bc^{t}}(\bx^{t+1})), \bc^{t+1}-\bc^{t}\right\rangle \notag + \frac{1}{2\eta_1}\|\bc^{t+1}-\bc^{t}\|^2 \notag \\
			B(\bc^{t}) &:= \lambda R(\bx^{t+1},\bc^{t}).
		\end{align*} 
		Then $B(\bc^{t+1})\leq B(\bc^{t})$.	
\end{claim*}}
\begin{proof}
	Let $B(\bc)$ denote the expression inside the $\arg\min$ in \eqref{app:quantized argmin2} and we know that \eqref{app:quantized argmin2} is minimized when $\bc=\bc^{t+1}$. So we have $B(\bc^{t+1})\leq B(\bc^t)$. This proves the claim.
\end{proof}

	
	\section{Omitted Details in Proof of Theorem 2}\label{appendix:proof of theorem 2}
	\begin{claim*} [Restating Claim~\ref{thm2:claim lqxsmoothness}.]
	$(1-\lambda_p)(f_i(\bx)+ f_i(\widetilde{Q}_\bc(\bx)))+\lambda_p (f^{KD}_i(\bx,\bw)+f^{KD}_i(\widetilde{Q}_\bc(\bx),\bw))$ is $(\lambda_p(L_{D_1}+L_{DQ_1})+(1-\lambda_p)(L+G^{(i)}L_{Q_1}+G^{(i)}_{Q_1}Ll_{Q_1}))$-smooth with respect to $\bx$.
\end{claim*}
\begin{proof}
	From our assumptions, we have $f_i$ is $L$-smooth, $f^{KD}_i(\bx,\bw)$ is $L_{D_1}$-smooth and $f^{KD}_i(\widetilde{Q}_\bc(\bx),\bw)$ is $L_{DQ_1}$-smooth with respect to $\bx$. And applying the Claim \ref{claim: lqxclaim} to each client separately gives that $f_i(\widetilde{Q}_\bc(\bx))$ is $(G^{(i)}L_{Q_1}+G^{(i)}_{Q_1}Ll_{Q_1})$-smooth. Using the fact that if two functions $g_1$ and $g_2$ (defined over the same space) are $L_1$ and $L_2$-smooth respectively, then $g_1+g_2$ is $(L_1+L_2)$-smooth, and the fact that $\alpha g_1$ is $\alpha L_1$-smooth for a given constant $\alpha$ concludes the proof.
\end{proof}

\textit{Obtaining} \eqref{thm2:first-part-interim1}:
\begin{align}
	&F_i(\bx^{t+1}_{i},\bc^t_{i},\bw^{t}) +(\frac{1}{2\eta_1}-\frac{\lambda_p(L_{D_1}+L_{DQ_1})+(1-\lambda_p)(L+G^{(i)}L_{Q_1}+G^{(i)}_{Q_1}Ll_{Q_1})}{2})\|\bx^{t+1}_{i}-\bx^t_{i}\|^2 \nonumber \\
	&\hspace{0.5cm}= (1-\lambda_p)\left(f_i(\bx^{t+1}_{i})+f_i(\widetilde{Q}_{\bc^t_{i}}(\bx^{t+1}_{i}))\right) +\lambda_p\left(f^{KD}_i (\bx^{t+1}_{i},\bw^{t})+f^{KD}_i (\widetilde{Q}_{\bc^t_{i}}(\bx^{t+1}_{i}),\bw^{t})\right)  \nonumber \\
	&\hspace{1cm}+(\frac{1}{2\eta_1}-\frac{\lambda_p(L_{D_1}+L_{DQ_1})+(1-\lambda_p)(L+G^{(i)}L_{Q_1}+G^{(i)}_{Q_1}Ll_{Q_1})}{2})\|\bx^{t+1}_{i}-\bx^t_{i}\|^2 \notag \\ 
	& \hspace{1cm} +\lambda R(\bx^{t+1}_{i},\bc^t_{i}) \notag \\
	&\hspace{0.5cm}\leq (1-\lambda_p)\left(f_i(\bx^t_{i})+f_i(\widetilde{Q}_{\bc^t_{i}}(\bx^t_{i}))\right)+\lambda_p\left(f^{KD}_i (\bx^{t}_{i},\bw^{t})+f^{KD}_i (\widetilde{Q}_{\bc^t_{i}}(\bx^{t}_{i}),\bw^{t})\right) \notag \\
	&\hspace{1cm}+(1-\lambda_p) \Big\langle \nabla f_i(\bx^t_{i}), \bx^{t+1}_{i}-\bx^t_{i}\Big\rangle + (1-\lambda_p) \left\langle \nabla_{\bx^t_{i}} f_i(\widetilde{Q}_{\bc^t_{i}}(\bx^t_{i})), \bx^{t+1}_{i}-\bx^t_{i}\right\rangle \notag \\
	&\hspace{1cm}+\lambda_p\left\langle \nabla_{\bx^t_i} f^{KD}_i(\bx^t_{i},\bw^t), \bx^{t+1}_{i}-\bx^t_{i}\right\rangle +\lambda_p\left\langle \nabla_{\bx^t_i} f^{KD}_i(\widetilde{Q}_{\bc^t_{i}}(\bx^t_{i}),\bw^t), \bx^{t+1}_{i}-\bx^t_{i}\right\rangle \notag \\
	& \hspace{1cm} +\lambda R(\bx^{t+1}_{i},\bc^t_{i}) + \frac{1}{2\eta_1}\|\bx^{t+1}_{i}-\bx^t_{i}\|^2 \nonumber \\
	& \hspace{0.5cm} = (1-\lambda_p)\left(f_i(\bx^t_{i})+f_i(\widetilde{Q}_{\bc^t_{i}}(\bx^t_{i}))\right)+\lambda_p\left(f^{KD}_i (\bx^{t}_{i},\bw^{t})+f^{KD}_i (\widetilde{Q}_{\bc^t_{i}}(\bx^{t}_{i}),\bw^{t})\right) \notag \\ 
	&\hspace{1cm} +(1-\lambda_p) \left\langle \nabla f_i(\bx^t_{i}), \bx^{t+1}_{i}-\bx^t_{i}\right\rangle + (1-\lambda_p) \left\langle \nabla_{\bx^t_{i}} f_i(\widetilde{Q}_{\bc^t_{i}}(\bx^t_{i})), \bx^{t+1}_{i}-\bx^t_{i}\right\rangle  \nonumber \\
	&\hspace{1cm}+\lambda_p\left\langle \nabla_{\bx^t_i} f^{KD}_i(\bx^t_{i},\bw^t)-\nabla_{\bx^t_i} f^{KD}_i(\bx^t_{i},\bw^t_{i}), \bx^{t+1}_{i}-\bx^t_{i}\right\rangle \notag \\
	& \hspace{1cm} +\lambda_p\left\langle \nabla_{\bx^t_i} f^{KD}_i(\widetilde{Q}_{\bc^t_{i}}(\bx^t_{i}),\bw^t)-\nabla_{\bx^t_i} f^{KD}_i(\widetilde{Q}_{\bc^t_{i}}(\bx^t_{i}),\bw^t_{i}), \bx^{t+1}_{i}-\bx^t_{i}\right\rangle  \nonumber \\
	&\hspace{1cm} +\lambda_p\left\langle \nabla_{\bx^t_i} f^{KD}_i(\bx^t_{i},\bw^t_{i}), \bx^{t+1}_{i}-\bx^t_{i}\right\rangle + \lambda_p\left\langle \nabla_{\bx^t_i} f^{KD}_i((\widetilde{Q}_{\bc^t_{i}}(\bx^t_{i}),\bw^t_{i}), \bx^{t+1}_{i}-\bx^t_{i}\right\rangle \nonumber\\
	& \hspace{1cm}   + \frac{1}{2\eta_1}\|\bx^{t+1}_{i}-\bx^t_{i}\|^2+\lambda R(\bx^{t+1}_{i},\bc^t_{i}) \notag \\
	& \hspace{0.5cm} \leq (1-\lambda_p)\left(f_i(\bx^t_{i})+f_i(\widetilde{Q}_{\bc^t_{i}}(\bx^t_{i}))\right)+\lambda_p\left(f^{KD}_i (\bx^{t}_{i},\bw^{t})+f^{KD}_i (\widetilde{Q}_{\bc^t_{i}}(\bx^{t}_{i}),\bw^{t})\right) \notag \\ 
	&\hspace{1cm} +(1-\lambda_p)\Big\langle \nabla f_i(\bx^t_{i}), \bx^{t+1}_{i}-\bx^t_{i}\Big\rangle +(1-\lambda_p)\left\langle \nabla_{\bx^t_{i}} f_i(\widetilde{Q}_{\bc^t_{i}}(\bx^t_{i})), \bx^{t+1}_{i}-\bx^t_{i}\right\rangle \nonumber \\
	&\hspace{1cm} +\lambda_p\left\langle \nabla_{\bx^t_i} f^{KD}_i(\bx^t_{i},\bw^t_{i}), \bx^{t+1}_{i}-\bx^t_{i}\right\rangle + \lambda_p\left\langle \nabla_{\bx^t_i} f^{KD}_i((\widetilde{Q}_{\bc^t_{i}}(\bx^t_{i}),\bw^t_{i}), \bx^{t+1}_{i}-\bx^t_{i}\right\rangle  \nonumber \\
	& \hspace{1cm}  + \frac{\lambda_p}{2}\|\nabla_{\bx^t_i} f^{KD}_i(\widetilde{Q}_{\bc^t_{i}}(\bx^t_{i}),\bw^t)-\nabla_{\bx^t_i} f^{KD}_i(\widetilde{Q}_{\bc^t_{i}}(\bx^t_{i}),\bw^t_{i})\|^2 + \lambda_p\|\bx^{t+1}_{i}-\bx^t_{i}\|^2 \notag  \\
	& \hspace{1cm} + \frac{\lambda_p}{2}\|\nabla_{\bx^t_i} f^{KD}_i(\bx^t_{i},\bw^t)-\nabla_{\bx^t_i} f^{KD}_i(\bx^t_{i},\bw^t_{i})\|^2 + \frac{1}{2\eta_1}\|\bx^{t+1}_{i}-\bx^t_{i}\|^2 +\lambda R(\bx^{t+1}_{i},\bc^t_{i}) \notag.
\end{align}

\begin{claim*}[Restating Claim~\ref{thm2:claim:lower-bound1}.]
	Let 
	\begin{align*}
		A(\bx^{t+1}_{i}) &:= (1-\lambda_p)\left\langle \nabla f_i(\bx^t_{i}), \bx^{t+1}_{i}-\bx^t_{i}\right\rangle 
		+(1-\lambda_p) \left\langle \nabla_{\bx^t_{i}} f_i(\widetilde{Q}_{\bc^t_{i}}(\bx^t_{i})), \bx^{t+1}_{i}-\bx^t_{i}\right\rangle \notag \\
		&\hspace{1cm} +\left\langle \lambda_p(\nabla_{\bx^t_i} f^{KD}_i(\bx^t_{i},\bw^t_{i})), \bx^{t+1}_{i}-\bx^t_{i}\right\rangle + \left\langle \lambda_p(\nabla_{\bx^t_i} f^{KD}_i(\widetilde{Q}_{\bc^t_{i}}(\bx^t_{i}),\bw^t_{i})), \bx^{t+1}_{i}-\bx^t_{i}\right\rangle\\
		& \hspace{1cm} + \lambda R(\bx^{t+1}_{i},\bc^t_{i})+ \frac{1}{2\eta_1}\|\bx^{t+1}_{i}-\bx^t_{i}\|^2  \notag \\
		A(\bx^t_{i}) &:= \lambda R(\bx^t_{i},\bc^t_{i}).
	\end{align*} 
	Then $A(\bx^{t+1}_{i})\leq A(\bx^t_{i})$.

\end{claim*}
\begin{proof}
	Let $A(\bx)$ denote the expression inside the $\arg\min$ in \eqref{thm2:lower-bounding-interim1} and we know that \eqref{thm2:lower-bounding-interim1} is minimized when $\bx=\bx^{t+1}_{i}$. So we have $A(\bx^{t+1}_{i})\leq A(\bx^t_{i})$. This proves the claim.
\end{proof}

\begin{claim*} [Restating Claim~\ref{thm2:claim lqcsmoothness}.]
	$(1-\lambda_p)f_i(\widetilde{Q}_\bc(\bx))+\lambda_pf^{KD}_{i}(\widetilde{Q}_{\bc}(\bx),\bw))$ is $(\lambda_pL_{DQ_2}+(1-\lambda_p)(G^{(i)}L_{Q_2}+G^{(i)}_{Q_2}Ll_{Q_2}))$-smooth with respect to $\bc$.
\end{claim*}
\begin{proof}
	Proof is similar to the proof of Claim~\ref{thm2:claim lqxsmoothness}.
\end{proof}
\textit{Obtaining 	\eqref{thm2:second-part-interim1}}:
\begin{align}
	&F_i(\bx^{t+1}_{i},\bc^{t+1}_{i},\bw^{t})+(\frac{1}{2\eta_2}-\frac{\lambda_pL_{DQ_2}+(1-\lambda_p)(G^{(i)}L_{Q_2}+G^{(i)}_{Q_2}Ll_{Q_2})}{2})\|\bc^{t+1}_{i}-\bc^t_{i}\|^2 \notag \\
	& \hspace{0.5cm} = (1-\lambda_p)\left(f_i(\bx^{t+1}_{i})+f_i(\widetilde{Q}_{\bc^{t+1}_{i}}(\bx^{t+1}_{i}))\right) \notag \\
	&\hspace{1cm}+\lambda R(\bx^{t+1}_{i},\bc^{t+1}_{i}) +\lambda_p\left(f^{KD}_i(\bx^{t+1}_i,\bw^t)+f^{KD}_{i}(\widetilde{Q}_{\bc^{t+1}_i}(\bx^{t+1}_i)),\bw^t)\right) \notag \\
	&\hspace{1cm}+(\frac{1}{2\eta_2}-\frac{\lambda_pL_{DQ_2}+(1-\lambda_p)(G^{(i)}L_{Q_2}+G^{(i)}_{Q_2}Ll_{Q_2})}{2})\|\bc^{t+1}_{i}-\bc^t_{i}\|^2 \notag \\ 
	&\hspace{0.5cm}\leq (1-\lambda_p)\left(f_i(\bx^{t+1}_{i})+f_i(\widetilde{Q}_{\bc^t_{i}}(\bx^{t+1}_{i}))\right)+\lambda R(\bx^{t+1}_{i},\bc^{t+1}_{i})  \notag \\
	& \hspace{1cm} +\lambda_p\left(f^{KD}_i(\bx^{t+1}_i,\bw^t)+f^{KD}_{i}(\widetilde{Q}_{\bc^{t}_i}(\bx^{t+1}_i),\bw^t)\right) + \frac{1}{2\eta_2}\|\bc^{t+1}_{i}-\bc^t_{i}\|^2 \notag \\
	&\hspace{1cm}+(1-\lambda_p) \left\langle \nabla_{\bc^t_{i}} f_i(\widetilde{Q}_{\bc^t_{i}}(\bx^{t+1}_{i})), \bc^{t+1}_{i}-\bc^t_{i}\right\rangle +\lambda_p \left\langle \nabla_{\bc^t_{i}} f^{KD}_i(\widetilde{Q}_{\bc^t_{i}}(\bx^{t+1}_{i}),\bw^t), \bc^{t+1}_{i}-\bc^t_{i}\right\rangle \notag \\
	&\hspace{0.5cm} = (1-\lambda_p)\left(f_i(\bx^{t+1}_{i})+f_i(\widetilde{Q}_{\bc^t_{i}}(\bx^{t+1}_{i}))\right)+\lambda R(\bx^{t+1}_{i},\bc^{t+1}_{i})  \notag \\
	& \hspace{1cm} +\lambda_p\left(f^{KD}_i(\bx^{t+1}_i,\bw^t)+f^{KD}_{i}(\widetilde{Q}_{\bc^{t}_i}(\bx^{t+1}_i),\bw^t)\right) + \frac{1}{2\eta_2}\|\bc^{t+1}_{i}-\bc^t_{i}\|^2 \notag \\
	&\hspace{1cm}+(1-\lambda_p) \left\langle \nabla_{\bc^t_{i}} f_i(\widetilde{Q}_{\bc^t_{i}}(\bx^{t+1}_{i})), \bc^{t+1}_{i}-\bc^t_{i}\right\rangle +\lambda_p \left\langle \nabla_{\bc^t_{i}} f^{KD}_i(\widetilde{Q}_{\bc^t_{i}}(\bx^{t+1}_{i}),\bw^t_i), \bc^{t+1}_{i}-\bc^t_{i}\right\rangle \notag \\
	&\hspace{1cm} +\lambda_p \left\langle \nabla_{\bc^t_{i}} f^{KD}_i(\widetilde{Q}_{\bc^t_{i}}(\bx^{t+1}_{i}),\bw^t)-\nabla_{\bc^t_{i}} f^{KD}_i(\widetilde{Q}_{\bc^t_{i}}(\bx^{t+1}_{i}),\bw^t_i), \bc^{t+1}_{i}-\bc^t_{i}\right\rangle  \notag \\
	&\hspace{0.5cm} \leq (1-\lambda_p)\left(f_i(\bx^{t+1}_{i})+f_i(\widetilde{Q}_{\bc^t_{i}}(\bx^{t+1}_{i}))\right)+\lambda R(\bx^{t+1}_{i},\bc^{t+1}_{i})  \notag \\
	& \hspace{1cm} +\lambda_p\left(f^{KD}_i(\bx^{t+1}_i,\bw^t)+f^{KD}_{i}(\widetilde{Q}_{\bc^{t}_i}(\bx^{t+1}_i),\bw^t)\right)  + \frac{1}{2\eta_2}\|\bc^{t+1}_{i}-\bc^t_{i}\|^2 \notag \\
	&\hspace{1cm}+(1-\lambda_p) \left\langle \nabla_{\bc^t_{i}} f_i(\widetilde{Q}_{\bc^t_{i}}(\bx^{t+1}_{i})), \bc^{t+1}_{i}-\bc^t_{i}\right\rangle +\lambda_p \left\langle \nabla_{\bc^t_{i}} f^{KD}_i(\widetilde{Q}_{\bc^t_{i}}(\bx^{t+1}_{i}),\bw^t_i), \bc^{t+1}_{i}-\bc^t_{i}\right\rangle \notag \\
	&\hspace{1cm} +\frac{\lambda_p}{2} \| \nabla_{\bc^t_{i}} f^{KD}_i(\widetilde{Q}_{\bc^t_{i}}(\bx^{t+1}_{i}),\bw^t)-\nabla_{\bc^t_{i}} f^{KD}_i(\widetilde{Q}_{\bc^t_{i}}(\bx^{t+1}_{i}),\bw^t_i)\|^2 + \frac{\lambda_p}{2} \| \bc^{t+1}_{i}-\bc^t_{i}\|^2. \notag 
\end{align}
\begin{claim*} [Restating Claim~\ref{thm2:claim:lower-bound2}.]
	Let 
	\begin{align*}
		B(\bc^{t+1}_{i}) &:= \lambda R(\bx^{t+1}_{i},\bc^{t+1}_{i}) + (1-\lambda_p)\left\langle \nabla_{\bc^t_{i}} f_i(\widetilde{Q}_{\bc^t_{i}}(\bx^{t+1}_{i})), \bc^{t+1}_{i}-\bc^t_{i}\right\rangle \notag  \\ 
		& \quad  + \lambda_p \left\langle \nabla_{\bc^t_{i}} f^{KD}_i(\widetilde{Q}_{\bc^t_{i}}(\bx^{t+1}_{i}),\bw^t_i), \bc^{t+1}_{i}-\bc^t_{i}\right\rangle + \frac{1}{2\eta_2}\|\bc^{t+1}_{i}-\bc^t_{i}\|^2 \notag \\
		B(\bc^t_{i}) &:= \lambda R(\bx^{t+1}_{i},\bc^t_{i}).
	\end{align*} 
	Then $B(\bc^{t+1}_{i})\leq B(\bc^t_{i})$.
\end{claim*}
\begin{proof}
	Let $B(\bc)$ denote the expression inside the $\arg\min$ in \eqref{thm2:lower-bounding-interim2} and we know that \eqref{thm2:lower-bounding-interim2} is minimized when $\bc=\bc^{t+1}_{i}$. So we have $B(\bc^{t+1}_{i})\leq B(\bc^t_{i})$. This proves the claim.
\end{proof}

\textit{Obtaining \eqref{thm2:dec3}}:
\begin{align*}
	&F_i(\bx^{t+1}_{i},\bc^{t+1}_{i},\bw^{t+1}) \leq 	F_i(\bx^{t+1}_{i},\bc^{t+1}_{i},\bw^{t}) + \left\langle \nabla_{\bw^{t}} F_i(\bx^{t+1}_{i},\bc^{t+1}_{i},\bw^{t}), \bw^{t+1}-\bw^{t} \right\rangle\\
	& \hspace{5cm}+ \frac{\lambda_p(L_{D_2}+L_{DQ_3})}{2}\|\bw^{t+1}-\bw^{t}\|^2 \\
	&= F_i(\bx^{t+1}_{i},\bc^{t+1}_{i},\bw^{t}) - \left\langle \nabla_{\bw^{t}} F_i(\bx^{t+1}_{i},\bc^{t+1}_{i},\bw^{t}), \eta_3 \bg^{t} \right\rangle + \frac{\lambda_p(L_{D_2}+L_{DQ_3})}{2}\|\eta_3 \bg^{t}\|^2 \\
	& =	F_i(\bx^{t+1}_{i},\bc^{t+1}_{i},\bw^{t}) \\& \quad -  \eta_3 \left\langle \nabla_{\bw^{t}} F_i(\bx^{t+1}_{i},\bc^{t+1}_{i},\bw^{t}), \bg^{t} - \nabla_{\bw^{t}} F_i(\bx^{t+1}_{i},\bc^{t+1}_{i},\bw^{t}) + \nabla_{\bw^{t}} F_i(\bx^{t+1}_{i},\bc^{t+1}_{i},\bw^{t}) \right\rangle  \\& \quad +\frac{\lambda_p(L_{D_2}+L_{DQ_3})}{2}\eta_3^2\|\bg^{t} - \nabla_{\bw^{t}} F_i(\bx^{t+1}_{i},\bc^{t+1}_{i},\bw^{t}) + \nabla_{\bw^{t}} F_i(\bx^{t+1}_{i},\bc^{t+1}_{i},\bw^{t})\|^2 \\ 
	&=	F_i(\bx^{t+1}_{i},\bc^{t+1}_{i},\bw^{t}) - \eta_3\Big\|\nabla_{\bw^{t}} F_i(\bx^{t+1}_{i},\bc^{t+1}_{i},\bw^{t})\Big\|^2 \\& 
	\quad -  \eta_3 \left\langle \nabla_{\bw^{t}} F_i(\bx^{t+1}_{i},\bc^{t+1}_{i},\bw^{t}), \bg^{t} - \nabla_{\bw^{t}} F_i(\bx^{t+1}_{i},\bc^{t+1}_{i},\bw^{t}) \right\rangle  \\& 
	\quad + \frac{\lambda_p(L_{D_2}+L_{DQ_3})}{2}\eta_3^2\|\bg^{t} - \nabla_{\bw^{t}} F_i(\bx^{t+1}_{i},\bc^{t+1}_{i},\bw^{t}) + \nabla_{\bw^{t}} F_i(\bx^{t+1}_{i},\bc^{t+1}_{i},\bw^{t})\|^2 \\ 
	& \leq F_i(\bx^{t+1}_{i},\bc^{t+1}_{i},\bw^{t}) - \eta_3\Big\|\nabla_{\bw^{t}} F_i(\bx^{t+1}_{i},\bc^{t+1}_{i},\bw^{t})\Big\|^2 + \frac{\eta_3}{2}\Big\|\nabla_{\bw^{t}} F_i(\bx^{t+1}_{i},\bc^{t+1}_{i},\bw^{t})\Big\|^2 \\
	& \quad + \frac{\eta_3}{2}\Big\|\bg^{t} - \nabla_{\bw^{t}} F_i(\bx^{t+1}_{i},\bc^{t+1}_{i},\bw^{t})\Big\|^2 + \lambda_p(L_{D_2}+L_{DQ_3}) \eta_3^2 \Big\|\bg^{t} - \nabla_{\bw^{t}} F_i(\bx^{t+1}_{i},\bc^{t+1}_{i},\bw^{t})\Big\|^2 \\
	& \quad + \lambda_p(L_{D_2}+L_{DQ_3}) \eta_3^2 \Big\|\nabla_{\bw^{t}} F_i(\bx^{t+1}_{i},\bc^{t+1}_{i},\bw^{t})\Big\|^2 
	\\ 
	& = F_i(\bx^{t+1}_{i},\bc^{t+1}_{i},\bw^{t}) - (\frac{\eta_3}{2}-\lambda_p(L_{D_2}+L_{DQ_3})\eta_3^2)\Big\|\nabla_{\bw^{t}} F_i(\bx^{t+1}_{i},\bc^{t+1}_{i},\bw^{t})\Big\|^2 \\& \quad + (\frac{\eta_3}{2}+\lambda_p(L_{D_2}+L_{DQ_3})\eta_3^2)\Big\|\bg^{t} - \nabla_{\bw^{t}_{i}} F_i(\bx^{t+1}_{i},\bc^{t+1}_{i},\bw^{t}_{i}) + \nabla_{\bw^{t}_{i}} F_i(\bx^{t+1}_{i},\bc^{t+1}_{i},\bw^{t}_{i}) \\ & \quad - \nabla_{\bw^{t}} F_i(\bx^{t+1}_{i},\bc^{t+1}_{i},\bw^{t})\Big\|^2  \\
	& \leq F_i(\bx^{t+1}_{i},\bc^{t+1}_{i},\bw^{t}) - (\frac{\eta_3}{2}-\lambda_p(L_{D_2}+L_{DQ_3})\eta_3^2)\Big\|\nabla_{\bw^{t}} F_i(\bx^{t+1}_{i},\bc^{t+1}_{i},\bw^{t})\Big\|^2 \\
	& \quad + (\eta_3+2\lambda_p(L_{D_2}+L_{DQ_3})\eta_3^2)\Big\|\bg^{t} - \nabla_{\bw^{t}_{i}} F_i(\bx^{t+1}_{i},\bc^{t+1}_{i},\bw^{t}_{i})\Big\|^2 \\
	& \quad+(\eta_3+2\lambda_p(L_{D_2}+L_{DQ_3})\eta_3^2)\Big\|\nabla_{\bw^{t}_{i}} F_i(\bx^{t+1}_{i},\bc^{t+1}_{i},\bw^{t}_{i}) - \nabla_{\bw^{t}} F_i(\bx^{t+1}_{i},\bc^{t+1}_{i},\bw^{t})\Big\|^2 \\
	& \leq F_i(\bx^{t+1}_{i},\bc^{t+1}_{i},\bw^{t}) - (\frac{\eta_3}{2}-\lambda_p(L_{D_2}+L_{DQ_3})\eta_3^2)\Big\|\nabla_{\bw^{t}} F_i(\bx^{t+1}_{i},\bc^{t+1}_{i},\bw^{t})\Big\|^2 \\
	& \quad + (\eta_3+2\lambda_p(L_{D_2}+L_{DQ_3})\eta_3^2)\Big\|\bg^{t} - \nabla_{\bw^{t}_{i}} F_i(\bx^{t+1}_{i},\bc^{t+1}_{i},\bw^{t}_{i})\Big\|^2 \\
	&\quad+(\eta_3+2\lambda_p(L_{D_2}+L_{DQ_3})\eta_3^2)(\lambda_p(L_{D_2}+L_{DQ_3}))^2\Big\|\bw^{t}_{i}-\bw^{t}\Big\|^2.
\end{align*}
Rearranging the terms gives \eqref{thm2:dec3}.
\begin{lemma*}[Restating Lemma~\ref{thm2:lemma1}]
	Let $\eta_3$ be chosen such that $\eta_3 \leq \sqrt{\frac{1}{6\tau^2(\lambda_pL_{w})^2\left(1+\frac{\overline{L}_{\max}^2}{(L^{(\min)}_{\max})^2}\right)}}$ where $L^{(\min)}_{\max}=\min\{L^{(i)}_{\max}:i\in[n]\}$ and $\overline{L}_{\max} = \sqrt{\frac{1}{n}\sum_{i=1}^n(L^{(i)}_{\max})^2}$ (where $L^{(i)}_{\max}$ is defined in \eqref{Li_max-defn}), then we have,
	\begin{align*} 
		\frac{1}{T}\sum_{t=0}^{T-1}\frac{1}{n}\sum_{i=1}^{n} (L^{(i)}_{\max})^2
		\|\bw^{t}-\bw^{t}_{i}\|^2 \leq \frac{1}{T} \sum_{t=0}^{T-1} \gamma_t \leq  6\tau^2\eta_3^2\frac{1}{n}\sum_{i=1}^n(L^{(i)}_{\max})^2\kappa_i
	\end{align*}
\end{lemma*}

\begin{proof}
	Let $t_c$ be the latest synchronization time before $t$. Define $\gamma_t = \frac{1}{n} \sum_{i=1}^{n}(L^{(i)}_{\max})^2\|\bw^{t} - \bw^{t}_{i}\|^2$. Then:
	\begin{align*}
		\gamma_t &= \frac{1}{n} \sum_{i=1}^{n}(L^{(i)}_{\max})^2 \Big\|\bw^{t_c}-\frac{\eta_3}{n} \sum_{j=t_c}^{t} \sum_{k=1}^{n} \nabla_{\bw^{j}_{k}} F_k(\bx^{j+1}_{k},\bc^{j+1}_{k},\bw^{j}_{k})\\
		& \hspace{1cm} -(\bw^{t_c}-\eta_3 \sum_{j=t_c}^{t} \nabla_{\bw^{j}_{i}} F_i(\bx^{j+1}_{i},\bc^{j+1}_{i},\bw^{j}_{i})\Big\|^2 \\ 
		&\stackrel{\text{(a)}}{\leq} \tau  \sum_{j=t_c}^t \frac{\eta_3^2}{n} \sum_{i=1}^n(L^{(i)}_{\max})^2 \Big\|\frac{1}{n} \sum_{k=1}^{n} \nabla_{\bw^{j}_{k}} F_k(\bx^{j+1}_{k},\bc^{j+1}_{k},\bw^{j}_{k})- \nabla_{\bw^{j}_{i}} F_i(\bx^{j+1}_{i},\bc^{j+1}_{i},\bw^{j}_{i})\Big\|^2 \tag{$\star1$} \label{app:lm1star1}\\
		&= \tau  \sum_{j=t_c}^t \frac{\eta_3^2}{n} \sum_{i=1}^n(L^{(i)}_{\max})^2 \Big[  \Big\|\frac{1}{n} \sum_{k=1}^{n} \Big(\nabla_{\bw^{j}_{k}} F_k(\bx^{j+1}_{k},\bc^{j+1}_{k},\bw^{j}_{k})-\nabla_{\bw^{j}} F_k(\bx^{j+1}_{k},\bc^{j+1}_{k},\bw^{j}) \\ 
		&\hspace{1cm}  + \nabla_{\bw^{j}} F_k(\bx^{j+1}_{k},\bc^{j+1}_{k},\bw^{j}) \Big) -\nabla_{\bw^{j}} F_i(\bx^{j+1}_{i},\bc^{j+1}_{i},\bw^{j})+\nabla_{\bw^{j}} F_i(\bx^{j+1}_{i},\bc^{j+1}_{i},\bw^{j})\\
		&\hspace{1cm}- \nabla_{\bw^{j}_{i}} F_i(\bx^{j+1}_{i},\bc^{j+1}_{i},\bw^{j}_{i}\Big\|^2 \Big] \\ 
		& \leq \tau  \sum_{j=t_c}^{t_c+\tau} 3\frac{\eta_3^2}{n}  \sum_{i=1}^n(L^{(i)}_{\max})^2 \Big[\Big\|\frac{1}{n} \sum_{k=1}^{n} \left( \nabla_{\bw^{j}_{k}} F_k(\bx^{j+1}_{k},\bc^{j+1}_{k},\bw^{j}_{k})-\nabla_{\bw^{j}} F_k(\bx^{j+1}_{k},\bc^{j+1}_{k},\bw^{j})\right)\Big\|^2\\ 
		&\hspace{1cm} + \Big\|\frac{1}{n}\sum_{k=1}^{n} \nabla_{\bw^{j}} F_k(\bx^{j+1}_{k},\bc^{j+1}_{k},\bw^{j}) -\nabla_{\bw^{j}} F_i(\bx^{j+1}_{i},\bc^{j+1}_{i},\bw^{j})\Big\|^2 \\ 
		&\hspace{1cm} + \Big\|\nabla_{\bw^{j}} F_i(\bx^{j+1}_{i},\bc^{j+1}_{i},\bw^{j})- \nabla_{\bw^{j}_{i}} F_i(\bx^{j+1}_{i},\bc^{j+1}_{i},\bw^{j}_{i}\Big\|^2\Big] \\ 
		& \leq \tau  \sum_{j=t_c}^{t_c+\tau} 3\eta_3^2 \Big[\frac{(\lambda_pL_{w})^2 n}{n^2}\sum_{k=1}^n\frac{1}{n}\sum_{i=1}^n(L^{(i)}_{\max})^2\| \bw^{j}_{k}-\bw^{j}\|^2 + \frac{1}{n}\sum_{i=1}^n(L^{(i)}_{\max})^2\kappa_i \\
		& \hspace{1cm}+ (\lambda_pL_{w})^2\frac{1}{n}\sum_{i=1}^n(L^{(i)}_{\max})^2\|\bw^j-\bw^{j}_{i}\|^2\Big] \\ 
		& \leq \tau  \sum_{j=t_c}^{t_c+\tau} 3\eta_3^2 \Big[\frac{(\lambda_pL_{w})^2 n}{n^2}\frac{1}{n}\sum_{i=1}^n(L^{(i)}_{\max})^2\frac{1}{(L^{(\min)}_{\max})^2}\sum_{k=1}^n(L^{(k)}_{\max})^2\| \bw^{j}_{k}-\bw^{j}\|^2 \\
		& \hspace{1cm}+ \frac{1}{n}\sum_{i=1}^n(L^{(i)}_{\max})^2\kappa_i + (\lambda_pL_{w})^2\gamma_j\|^2\Big] \\ 
		& =\tau  \sum_{j=t_c}^{t_c+\tau} 3\eta_3^2\left((\lambda_pL_{w})^2\frac{1}{n}\sum_{i=1}^n(L^{(i)}_{\max})^2\frac{1}{(L^{(\min)}_{\max})^2}\gamma_j+\frac{1}{n}\sum_{i=1}^n(L^{(i)}_{\max})^2\kappa_i+(\lambda_pL_{w})^2 \gamma_j\right)   \\
		& =\tau  \sum_{j=t_c}^{t_c+\tau} 3\eta_3^2\left((\lambda_pL_{w})^2\frac{\overline{L}_{\max}^2}{(L^{(\min)}_{\max})^2}\gamma_j+\frac{1}{n}\sum_{i=1}^n(L^{(i)}_{\max})^2\kappa_i+(\lambda_pL_{w})^2 \gamma_j\right) \tag{$\star2$} \label{app:lm1star2}
	\end{align*}
	in (a) we use the facts that  $\|\sum_{i=1}^K a_i\|^2 \leq K \sum_{i=1}^K \|a_i\|^2$, $t\leq \tau + t_c$ and that we are summing over non-negative terms. As a result, we have:
	\begin{align*}
		\gamma_t &\leq \tau  \sum_{j=t_c}^{t_c+\tau} 3\eta_3^2\left((\lambda_pL_{w})^2\Big(1+\frac{\overline{L}_{\max}^2}{(L^{(\min)}_{\max})^2}\Big)\gamma_j+\frac{1}{n}\sum_{i=1}^n(L^{(i)}_{\max})^2\kappa_i\right)\\
		\Longrightarrow \sum_{t=t_c}^{t_c+\tau} \gamma_t &\leq \sum_{t=t_c}^{t_c+\tau}  \sum_{j=t_c}^{t_c+\tau} 3\tau\eta_3^2\left((\lambda_pL_{w})^2\Big(1+\frac{\overline{L}_{\max}^2}{(L^{(\min)}_{\max})^2}\Big)\gamma_j+\frac{1}{n}\sum_{i=1}^n(L^{(i)}_{\max})^2\kappa_i\right)\\
		& = 3\tau^2\eta_3^2(\lambda_pL_{w})^2\Big(1+\frac{\overline{L}_{\max}^2}{(L^{(\min)}_{\max})^2}\Big)\sum_{j=t_c}^{t_c+\tau} \gamma_j + 3\tau^3 \eta_3^2\frac{1}{n}\sum_{i=1}^n(L^{(i)}_{\max})^2\kappa_i
	\end{align*}
	Let us choose $\eta_3$ such that $3\tau^2\eta_3^2(\lambda_pL_{w})^2\left(1+\frac{\overline{L}_{\max}^2}{(L^{(\min)}_{\max})^2}\right) \leq \frac{1}{2} \Leftrightarrow \eta_3 \leq \sqrt{\frac{1}{6\tau^2(\lambda_pL_{w})^2\left(1+\frac{\overline{L}_{\max}^2}{(L^{(\min)}_{\max})^2}\right)}} $, sum over all syncronization times, and divide both sides by $T$:
	\begin{align*} 
		&\frac{1}{T} \sum_{t=0}^{T-1} \gamma_t \leq \frac{1}{2} \sum_{j=0}^{T-1} \gamma_j + 3\tau^2 \eta_3^2\frac{1}{n}\sum_{i=1}^n(L^{(i)}_{\max})^2\kappa_i \\ \Longrightarrow &\frac{1}{T} \sum_{t=0}^{T-1} \gamma_t \leq  6\tau^2\eta_3^2\frac{1}{n}\sum_{i=1}^n(L^{(i)}_{\max})^2\kappa_i
	\end{align*}
\end{proof}

\begin{corollary*}[Restating Corollary~\ref{thm2:corollary diversity}.]
	Recall, $\bg^{t} = \frac{1}{n} \sum_{i=1}^n \nabla_{\bw^{t}} F_i(\bx^{t+1}_{i},\bc^{t+1}_{i},\bw^{t}_{i})$. Then, we have:
	\begin{align*}
		\frac{1}{T}\sum_{t=0}^{T-1} \frac{1}{n} \sum_{i=1}^n (L^{(i)}_{\max})^2 \Big\|\bg^{t} - \nabla_{\bw^{t}_{i}} F_i(\bx^{t+1}_{i},\bc^{t+1}_{i},\bw^{t}_{i})\Big\|^2 &\leq  3 \frac{1}{n}\sum_{i=1}^n(L^{(i)}_{\max})^2\kappa_i\\
		& \hspace{-2cm} + 3(\lambda_pL_{w})^2\Big(1+\frac{\overline{L}_{\max}^2}{(L^{(\min)}_{\max})^2}\Big) 6\tau^2 \eta_3^2 \frac{1}{n}\sum_{i=1}^n(L^{(i)}_{\max})^2\kappa_i,
	\end{align*}
	where $L^{(i)}_{\max}$ is defined in \eqref{Li_max-defn}, and $\overline{L}_{\max},L^{(\min)}_{\max}$ are defined in Lemma~\ref{thm2:lemma1}.
\end{corollary*}
\begin{proof}
	From \eqref{app:lm1star1} $\leq$ \eqref{app:lm1star2} in the proof of Lemma~\ref{thm2:lemma1}, we have:
	\begin{align*}
		&\sum_{t=t_c}^{t_c+\tau} \frac{1}{n} \sum_{i=1}^n (L^{(i)}_{\max})^2 \Big\|\bg^{t} - \nabla_{\bw^{t}_{i}} F_i(\bx^{t+1}_{i},\bc^{t+1}_{i},\bw^{t}_{i})\Big\|^2 \\
		&\hspace{3cm}\leq \sum_{t=t_c}^{t_c+\tau} 3\left((\lambda_pL_{w})^2\Big(1+\frac{\overline{L}_{\max}^2}{(L^{(\min)}_{\max})^2}\Big)\gamma_t +\frac{1}{n}\sum_{i=1}^n(L^{(i)}_{\max})^2\kappa_i\right)
	\end{align*}
	Summing over all $t_c$ and dividing by $T$:
	\begin{align*}
		&\frac{1}{T}\sum_{t=0}^{T-1} \frac{1}{n} \sum_{i=1}^n (L^{(i)}_{\max})^2 \Big\|\bg^{t} - \nabla_{\bw^{t}_{i}} F_i(\bx^{t+1}_{i},\bc^{t+1}_{i},\bw^{t}_{i})\Big\|^2 \\
		&\hspace{2cm}\leq 3(\lambda_pL_{w})^2\Big(1+\frac{\overline{L}_{\max}^2}{(L^{(\min)}_{\max})^2}\Big) \frac{1}{T} \sum_{t=0}^{T-1} \gamma_t + 3\frac{1}{n}\sum_{i=1}^n(L^{(i)}_{\max})^2\kappa_i\\
		& \hspace{2cm} \stackrel{\text{(a)}}{\leq} 3(\lambda_pL_{w})^2\Big(1+\frac{\overline{L}_{\max}^2}{(L^{(\min)}_{\max})^2}\Big) \times 6\tau^2 \eta_3^2 \frac{1}{n}\sum_{i=1}^n(L^{(i)}_{\max})^2\kappa_i + 3 \frac{1}{n}\sum_{i=1}^n(L^{(i)}_{\max})^2\kappa_i,
	\end{align*}
	where (a) is from Lemma~\ref{thm2:lemma1}.
\end{proof}

	
	\section{Additional Details for Experiments} \label{appendix:experiments}
	In this section, we first discuss the implementation details for the \texttt{prox} steps for Algorithm \ref{algo:centralized} and Algorithm \ref{algo:personalized} in Section \ref{sec:experiments:proximal updates}. Section \ref{subsec:expts_implementation} discusses implementation details for the algorithms along with hyperparameters which was omitted in Section \ref{sec:experiments} of the main paper due to space constraints.

\subsection{Proximal Updates} \label{sec:experiments:proximal updates}
For the implementation of Algorithm \ref{algo:centralized},\ref{algo:personalized}, we consider $\ell_1$-loss for the distance function $R(\bx,\bc)$. In other words, $R(\bx,\bc)=\min\{ \frac{1}{2} \|\bz-\bx\|_1:z_i \in \{c_1,\cdots,c_m\}, \forall i \}$. For simplicity, we define $\calC = \{\bz:z_i \in \{c_1,\cdots,c_m\}, \forall i\}$.
For the first type of update (update of $\bx$) we have:
\begin{align}
	\text{prox}_{\eta_1 \lambda R(\cdot,\bc)}(\by) &= {\argmin_{\bx \in \mathbb{R}^{d}} }\left\{\frac{1}{2 \eta_1}\left\|\bx-\by\right\|_{2}^{2}+\lambda R(\bx,\bc) \right\}\nonumber \\
	&= {\argmin_{\bx \in \mathbb{R}^{d}} }\left\{\frac{1}{2 \eta_1}\left\|\bx-\by\right\|_{2}^{2}+\frac{\lambda}{2} \min_{\bz \in \calC} \|\bz-\bx\|_1 \right\}\nonumber \\
	&= {\argmin_{\bx \in \mathbb{R}^{d}} } \min_{\bz \in \calC} \left\{\frac{1}{2 \eta_1}\left\|\bx-\by\right\|_{2}^{2}+\frac{\lambda}{2} \|\bz-\bx\|_1 \right\}\nonumber \\
\end{align}
This corresponds to solving:
\begin{align}
	\min_{\bz \in \calC} {\min_{\bx \in \mathbb{R}^{d}} }  \left\{\frac{1}{\eta_1}\left\|\bx-\by\right\|_{2}^{2}+\lambda \|\bz-\bx\|_1 \right\}\nonumber
\end{align}
Since both $\ell_1$ and squared $\ell_2$ norms are decomposable; if we fix $\bz$, for the inner problem we have the following solution to soft thresholding:
\begin{align}
	x^\star(\bz)_i=
	\begin{cases}
		y_i - \frac{\lambda\eta_1}{2}, \quad \text{if } y_i - \frac{\lambda\eta_1}{2} > z_i \\
		y_i + \frac{\lambda\eta_1}{2}, \quad \text{if } y_i + \frac{\lambda\eta_1}{2} < z_i \\
		z_i, \quad \text{otherwise}
	\end{cases}
\end{align}

As a result we have:

\begin{align}
	\min_{\bz \in \calC}  \left\{\frac{1}{\eta_1}\left\|x^\star(\bz)-\by\right\|_{2}^{2}+\lambda \|\bz-x^\star(\bz)\|_1 \right\}\nonumber
\end{align}
This problem is separable, in other words we have:
\begin{align*}
	\bz^\star_i = \argmin_{z_i \in \{c_1, \cdots, c_m\}} \left\{ \frac{1}{\eta_1}(x^\star(\bz)_i-y_i)^2+\lambda |z_i-x^\star(\bz)_i | \right\} \ \forall i
\end{align*}
Substituting $x^\star(\bz)_i$ and solving for $z_i$ gives us:
\begin{align*}
	\bz^\star_i = \argmin_{z_i \in \{c_1, \cdots, c_m\}} \left\{ |z_i - y_i | \right\} \ \forall i
\end{align*}
Or equivalently we have,
\begin{align}
	\bz^\star = \argmin_{\bz \in \calC} \|\bz-\by\|_1 = Q_\bc(\by)
\end{align}
As a result, $\text{prox}_{\eta_1 \lambda R(\cdot,\bc)}(\cdot)$ becomes the soft thresholding operator:
\begin{align}
	\text{prox}_{\eta_1 \lambda R(\cdot,\bc)}(\by)_i = 
	\begin{cases}
		y_i-\frac{\lambda\eta_1}{2}, \quad \text{if } y_i \geq Q_\bc(\by)_i+\frac{\lambda\eta_1}{2} \\
		y_i+\frac{\lambda\eta_1}{2}, \quad \text{if } y_i \leq Q_\bc(\by)_i-\frac{\lambda\eta_1}{2} \\
		Q_\bc(\by)_i, \quad \text{otherwise} 
	\end{cases}
\end{align}
And for the second type of update we have $\text{prox}_{\eta_2 \lambda R(\bx,\cdot)}(\cdot)$ becomes:
\begin{align}
	\text{prox}_{\eta_2 \lambda R(\bx,\cdot)}(\boldsymbol{\mu}
	) &= {\argmin_{\bc \in \mathbb{R}^{m}} }\left\{\frac{1}{2 \eta_2}\left\|\bc-\boldsymbol{\mu}
	\right\|_{2}^{2}+\lambda R(\bx,\bc) \right\}\nonumber \\
	&= {\argmin_{\bc \in \mathbb{R}^{m}} }\left\{\frac{1}{2 \eta_2}\left\|\bc-\boldsymbol{\mu}
	\right\|_{2}^{2}+\frac{\lambda}{2} \min_{\bz \in \cal C} \|\bz-\bx\|_1 \right\}\nonumber \\
	&= {\argmin_{\bc \in \mathbb{R}^{m}} } \left\{\frac{1}{2 \eta_2}\left\|\bc-\boldsymbol{\mu}
	\right\|_{2}^{2}+\frac{\lambda}{2} \|Q_\bc(\bx)-\bx\|_1 \right\} \label{app:prox2}
\end{align}
Then,
\begin{align*}
	\text{prox}_{\eta_2 \lambda R(\bx,\cdot)}(\boldsymbol{\mu}
	)_j &= {\argmin_{\bc_j \in \mathbb{R}^{m}} } \left\{\frac{1}{2 \eta_2}(c_j-\mu_j)^{2}+\frac{\lambda}{2} \sum_{i=1}^d |Q_\bc(\bx)_i-x_i| \right\} \\
	&= {\argmin_{\bc_j \in \mathbb{R}^{m}} } \left\{\frac{1}{2 \eta_2}(c_j-\mu_j)^{2}+\frac{\lambda}{2} \sum_{i=1}^d \mathbbm{1}(Q_{\bc}(\bx)_i=c_{j}) |c_j-x_i| \right\}
\end{align*}
We remark that the second term of the optimization problem is hard to solve; in particular we need to know the assignments of $x_i$ to $c_j$. In the algorithm, at each time point $t$, we are given the previous epoch's assignments. We can utilize that and approximate the optimization problem by assuming $\bc^{t+1}$ will be in a neighborhood of $\bc^t$. We can take the gradient of $R(\bx^{t+1},\bc)$ at $\bc=\bc^t$ while finding the optimal point. This is also equivalent to optimizing the first order Taylor approximation around $\bc=\bc^t$.
As a result we have the following optimization problem:
\begin{align*}
	\text{prox}_{\eta_2 \lambda R(\bx^{t+1},\cdot)}(\boldsymbol{\mu}
	)_j \approx {\argmin_{\bc_j \in \mathbb{R}^{m}} } \left\{\frac{1}{2 \eta_2}(c_j-\mu_j)^{2}+\frac{\lambda}{2} \sum_{i=1}^d \mathbbm{1}(Q_{\bc^t}(\bx^{t+1})_i=c^t_{j}) |c^t_j-x^{t+1}_i| \right. \\
	\left. +(c_j-c^t_j) \frac{\lambda}{2} \sum_{i=1}^d \mathbbm{1}(Q_{\bc^t}(\bx^{t+1})_i=c^t_{j})\frac{\partial |c^t_j-x^{t+1}_i|}{ \partial c^t_j}  \right\}
\end{align*}
In our implementation, we take $\frac{\partial |c^t_j-x^{t+1}_i|}{ \partial c^t_j}$ as $1$ if $c^t_j > x^{t+1}_i$,  $-1$ if $c^t_j < x^{t+1}_i$ and $0$ otherwise. Now taking the derivative with respect to $c_j$ and setting it to 0 gives us:
\begin{align*}
	\text{prox}_{\eta_2 \lambda R(\bx^{t+1},\cdot)}(\boldsymbol{\mu})_j&\approx \mu_j -\frac{\lambda\eta_2}{2}(\sum_{i=1}^d \mathbbm{1}(Q_{\bc^t}(\bx^{t+1})_i=c^t_{j}) \mathbbm{1}(x^{t+1}_i>c^{t}_{j}) \\
	& \hspace{2cm}-\sum_{i=1}^d \mathbbm{1}(Q_{\bc^t}(\bx^{t+1})_i=c^t_{j}) \mathbbm{1}(x^{t+1}_i<c^{t}_{j}))
\end{align*}

Proximal map pulls the updated centers toward the median of the weights that are assigned to them.  

\textbf{Using $P \rightarrow \infty$.} In the experiments we observed that using $P \rightarrow \infty$, i.e. using hard quantization function produces good results and also simplifies the implementation. The implications of $P \rightarrow \infty$ are as follows:
\begin{itemize}[leftmargin=*]
	\item We take $\nabla_\bx f(\widetilde{Q}_\bc(\bx))=0$ and $\nabla_\bx f^{KD}(\widetilde{Q}_\bc(\bx),\bw)=0$.
	\item We take $\nabla_\bc f(\widetilde{Q}_\bc(\bx)) = \nabla_\bc f(Q_\bc(\bx)) = \begin{bmatrix} \sum_{i=1}^d \frac{\partial f(Q_\bc(\bx))}{\partial Q_\bc(\bx)_i} \mathbbm{1}(Q_\bc(\bx)_i = c_1) \\ \vdots \\ \sum_{i=1}^d \frac{\partial f(Q_\bc(\bx))}{\partial Q_\bc(\bx)_i} \mathbbm{1}(Q_\bc(\bx)_i = c_m) \end{bmatrix}$ 
	and $\nabla_\bc f^{KD}(\widetilde{Q}_\bc(\bx),\bw) = \nabla_\bc f^{KD}(Q_\bc(\bx),\bw) = \begin{bmatrix} \sum_{i=1}^d \frac{\partial f^{KD}(Q_\bc(\bx),\bw)}{\partial Q_\bc(\bx)_i} \mathbbm{1}(Q_\bc(\bx)_i = c_1) \\ \vdots \\ \sum_{i=1}^d \frac{\partial f^{KD}(Q_\bc(\bx),\bw)}{\partial Q_\bc(\bx)_i} \mathbbm{1}(Q_\bc(\bx)_i = c_m) \end{bmatrix}$.
\end{itemize}

\subsection{Implementation Details and Hyperparameters} \label{subsec:expts_implementation}
In this section we discuss the implementation details and hyperparameters used for the algorithms considered in Section \ref{sec:experiments} of our main paper.

\textbf{Fine tuning.} In both centralized and federated settings we employ a fine tuning procedure similar to \cite{bai2018proxquant}. At the end of the regular training procedure, model weights are hard-quantized. After the hard-quantization, during the fine tuning epochs we let the unquantized parts of the network to continue training (e.g. batch normalization layers) and different from \cite{bai2018proxquant} we also continue to train quantization levels. 

\subsubsection{Centralized Setting}
For centralized training, we use CIFAR-10 dataset and train a ResNet \cite{he2015deep} model following \cite{bai2018proxquant} and \cite{BinaryRelax}. We employ ADAM with learning rate $0.01$ and no weight decay. We choose $\lambda(t) = 10^{-4}t$. For the implementation of ResNet models we used a toolbox\footnote{{\small \url{https://github.com/akamaster/pytorch_resnet_cifar10}}}. In Table~\ref{tab:Table1} we reported the results from \cite{BinaryRelax} directly and implemented ProxQuant using their published code\footnote{{\small \url{https://github.com/allenbai01/ProxQuant}}}. We use a learning schedule for $\eta_2$, particularly, we start with $\eta_2=10^{-4}$ and multiply it with 0.1 at epochs 80 and 140.

\subsubsection{Federated Setting} 
For each of the methods we tuned the local step learning rate separately on the set $\{0.2, 0.15, 0.125, 0.1, 0.075, 0.05\}$. We observed that except for the two cases, for all other cases, $0.1$ was the best choice for the learning rate in terms of accuracy: The two exceptions are the local training methods on FEMNIST and Per-FedAvg on CIFAR-10, for which, respectively, 0.075 and 0.125 were the best choices for the learning rate.

\begin{itemize}[leftmargin=*]
	\item QuPeD\footnote{For federated experiments we have used Pytorch's Distributed package.}: For CNN1 we choose $\lambda_p = 0.25$, $\lambda(t) = 10^{-6}t$ for 2Bits and $\lambda = 5 \times 10^{-7}t\frac{1}{0.99^t}$ for 1Bit training on CIFAR-10. On FEMNIST \footnote{We use {\small \url{https://github.com/tao-shen/FEMNIST_pytorch}} to import FEMNIST dataset.} and MNIST we choose $\lambda(t) = 5\times 10^{-6}t$ for 2Bits and $\lambda = 10^{-6}t\frac{1}{0.99^t}$ for 1Bit training. For CNN2 we use $\lambda_p = 0.15$. Global model has the same learning schedule as the personalized models. Furthermore, we use $\eta_2=10^{-4}$.
	
	QuPeL: We used $\lambda_p=0.2$, $\eta_3=0.5$ (same as pFedMe \cite{dinh2020personalized}) and took $\lambda$ values from QuPeD.
	\item Per-FedAvg \cite{fallah2020personalized} and pFedMe \cite{dinh2020personalized}:To implement Per-FedAvg, we used the same learning rate as mentioned in Section~\ref{sec:experiments}, schedule for main learning rate and $\alpha = 0.001$ for CNN1 and $\alpha = 2.5 \times 10^{-3}$ for CNN2 (we tuned in the interval $[8\times 10^{-4},5 \times 10^{-3}]$), for the auxiliary learning rate. For pFedMe we used the same learning rate schedule for main learning rate, $K=5$ for the number of local iterations; and we used $\lambda=0.5$, $\eta=0.2$ for CNN1 and $\lambda=0.2$, $\eta=0.15$ for CNN2 (we tuned in the interval $[0.1,1]$ for both parameters).
	
	\item Federated Mutual Learning \cite{shen2020federated}: Since authors do not discuss the hyperparameters in the paper, we used $\alpha=\beta=0.25$ for CNN1 and $\alpha=\beta=0.15$ for CNN2, similar to our use of $\lambda_p$ in QuPeD. Global model has the same learning schedule as the personalized models. 
\end{itemize}
For QuPeD and Federated ML we used CNN1 as the global model in all settings. For the other methods where global and personalized models cannot be different we used the same structure as personalized models.

\subsection{Additional Results for Federated Setting} \label{subsec:expts_additional_fed}

In this section we provide additional experimental results for comparison of QuPeD with other pearsonalized learning schems from literature.

\textbf{Comparison on another CNN architecture (CNN2).}  We first report experimental results on CIFAR-10 for CNN2 in Table~\ref{app:tab:Table CNN2} (with the same setting we have for Table~\ref{tab:Table CNN1}). This is a deeper architecture than CNN1, as described in Section \ref{sec:experiments} in the main paper.
\\

\begin{figure}[h]
	\centering
	\captionof{table}{Test accuracy (in \%) for CNN2 model at all clients, CIFAR-10.}
	\begin{tabular}{lccl} \toprule 
		Method & Test Accuracy in \%  \\ \midrule
		FedAvg (FP) & $62.49 \pm 0.42 $ \\ 
		Local Training (FP) & $73.86 \pm 0.22 $ \\ 
		Local Training (2 Bits)  & $73.24 \pm 0.14 $ \\
		Local Training (1 Bit)  & $70.23 \pm 0.10 $ \\
		\textbf{QuPeD} (FP) & $\mathbf{76.39} \pm 0.36 $\\
		\textbf{QuPeD} (2 Bits) &  $75.32 \pm 0.18$\\ 
		\textbf{QuPeD} (1 Bit) & $72.01 \pm 0.31$\\ 
		PFedMe (FP) \cite{dinh2020personalized}& $74.70 \pm 0.10$ \\
		Per-FedAvg(FP) \cite{fallah2020personalized} & $74.60 \pm 0.48$ \\
		Federated Mutual Learning(FP) \cite{shen2020federated} & $75.74 \pm 0.56 $
	\end{tabular} 
	\label{app:tab:Table CNN2}
\end{figure}

For the results in Table~\ref{app:tab:Table CNN2}, it can be seen that the comments made for Table~\ref{tab:Table CNN1} in the main paper directly hold as QuPeD is able to outperform other schemes by a significant margin. This demonstrates that QuPeD also works for a deeper neural network (than CNN1 considered in the main paper).\\

\begin{figure}[h]
	\centering
	\captionof{table}{Test accuracy (in \%) for CNN1 model at all clients, with 3 classes accessed per client on CIFAR-10.}
	\begin{tabular}{lcl} \toprule 
		Method   & Test Accuracy (in \%) \\ \midrule
		FedAvg (FP)  & $59.23 \pm 0.25 $\\ 
		Local Training (FP)  & $78.03 \pm 0.59 $\\ 
		Local Training (2 Bits)   & $77.47 \pm 0.64 $\\
		Local Training (1 Bit)   &$75.89 \pm 0.66 $ \\
		\textbf{QuPeD} (FP)  &$\mathbf{80.30} \pm 0.60 $\\
		\textbf{QuPeD} (2 Bits)  &$79.31 \pm 0.74 $\\ 
		\textbf{QuPeD} (1 Bit) &$77.23 \pm 0.58 $  \\
		QuPeL (2 Bits)  & $77.87 \pm 0.53 $\\ 
		QuPeL (1 Bits)  & $74.46 \pm 0.73 $\\ 
		pFedMe (FP) \cite{dinh2020personalized} & $78.22 \pm 0.91 $\\
		Per-FedAvg (FP) \cite{fallah2020personalized}  &$75.08 \pm 0.39 $\\
		Federated ML (FP) \cite{shen2020federated} & $79.44 \pm 0.82 $
	\end{tabular}
	\label{tab:appendix:Table 3 clients}
\end{figure}

\textbf{Another Type of Data Heterogeneity.} We report results for another data heterogeneity setting where each client has access to data samples from random 3 classes on CIFAR-10. Sampling data from 3 random classes per client is a more challenging setting compared 4 classes per client considered in Section~\ref{sec:experiments}. In Table~\ref{tab:appendix:Table 3 clients} we see that FedAvg's performance further decreased due to increased heterogeneity. Moreover, most of the other personalized FL methods are outperformed by local training whereas QuPeD still performs better than local training, and other personalized FL methods. We observe that QuPeD with 2 Bits aggressive quantization outperforms all the other competing methods except Federated ML \cite{shen2020federated} (for which it shows a similar accuracy). Moreover, QuPeD (1Bit) is able to outperform Per-FedAvg.\\

\begin{figure}[h]
	\centering
	\captionof{table}{Test accuracy (in \%) comparison between the cases with and without center updates for CNN1 model at all clients, 4 classes accessed per client on FEMNIST.}
	\begin{tabular}{lcl} \toprule 
		Method   & Test Accuracy (in \%) \\ \midrule
		\textbf{QuPeD} (FP)  &$\mathbf{97.31} \pm 0.12 $\\
		\textbf{QuPeD} (2 Bits)  &$96.73 \pm 0.27 $\\ 
		\textbf{QuPeD} (1 Bit) &$95.15 \pm 0.21 $  \\
		\textbf{QuPeD} (2 Bits) no center updates  &$96.48 \pm 0.10 $\\ 
		\textbf{QuPeD} (1 Bit) no center updates &$91.17 \pm 0.58 $  \\
	\end{tabular}
	\label{tab:appendix:Table no cent upd}
\end{figure}

\textbf{Importance of updating the centers.}  
In our proposed schemes: Algorithm \ref{algo:personalized}, we optimize over both the quantization levels and the model parameters. We compare performance of our proposed scheme with the case when we only optimize over model parameters and not quantization levels in Table~\ref{tab:appendix:Table no cent upd}. 
As seen from the results in the table, having the center updates in the optimization problem is critical, particularly, for the 1Bit quantization case for which we observe an increase in the performance by 4\%.

\textbf{Results on MNIST.} We now provide additional results on MNIST dataset to compared QuPeD with other competing schemes. We consider 50 clients in total, where each client samples data from  3 or 4 random classes and uses CNN1. We train for a total of 50 epochs, for quantized training we allocate the last 7 epochs for finetuning.\\

\begin{figure}[h]
	\centering
	\captionof{table}{Test accuracy (in \%) for CNN1 model at all clients, on MNIST.}
	\begin{tabular}{lccl} \toprule 
		Method   & 3 classes per client & 4 classes per client \\ \midrule
		FedAvg (FP) & $98.64 \pm 0.10 $ & $98.65 \pm 0.09 $\\ 
		Local Training (FP)  & $98.79 \pm 0.03 $  & $98.66 \pm 0.15 $\\ 
		Local Training (2 Bits)  & $98.53 \pm 0.07 $  & $98.37 \pm 0.11 $\\
		Local Training (1 Bit)   & $98.41 \pm 0.02 $  &$97.95 \pm 0.20 $ \\
		\textbf{QuPeD} (FP)  &$\mathbf{99.05} \pm 0.10 $  &$98.89 \pm 0.11 $\\
		\textbf{QuPeD} (2 Bits) &$98.96 \pm 0.13 $ &$98.67 \pm 0.18 $\\ 
		\textbf{QuPeD} (1 Bit)  &$98.57 \pm 0.08 $  &$98.25 \pm 0.16 $  \\
		QuPeL (2 Bits)  & $98.95 \pm 0.12 $  & $98.61 \pm 0.19 $\\ 
		QuPeL (1 Bits)   & $98.33 \pm 0.14 $ & $98.11 \pm 0.26 $\\ 
		pFedMe (FP) \cite{dinh2020personalized} & $98.98 \pm 0.05 $ & $98.82 \pm 0.15 $\\
		Per-FedAvg (FP) \cite{fallah2020personalized} &$98.82 \pm 0.05 $  &$\mathbf{98.93} \pm 0.09 $\\
		Federated ML (FP) \cite{shen2020federated} & $99.00 \pm 0.06 $ & $98.84 \pm 0.13 $
	\end{tabular}
	\label{tab:Table 34classes MNIST}
\end{figure}
QuPeD (FP) outperforms all methods except Per-FedAvg on MNIST when clients sample data from 4 random classes. The difference is almost negligible (0.04\%). As we can observe in Table~\ref{tab:Table 34classes MNIST} with the increased heterogeneity QuPeD starts to outperform Per-FedAvg by a 0.20\% margin. Moreover, we observe QuPeD with 2Bit quantization also outperforms Per-FedAvg. 

\textbf{Text classification task on AG News Dataset.} To show that our method can also be applied for tasks different than vision tasks. text classification problem using the AG News dataset (available at {\small\url{https://pytorch.org/text/stable/datasets.html}}). We used half of the dataset to make the training procedure more challenging. We used EmbeddingBag structure available at {\small\url{https://pytorch.org/tutorials/beginner/text_sentiment_ngrams_tutorial.html}} and distributed the data such that each of the 42 clients has access to samples from 3 out of 4 classes. The results we obtained are provided in Table~\ref{tab:Table AG News}\\

\begin{figure}[t]
	\centering
	\captionof{table}{Test accuracy (in \%) for Embedding Bag model at all clients, on AG News.}
	\begin{tabular}{lccl} \toprule
		Method &  \\ \midrule
		FedAvg (FP) & $83.04 \pm 0.60$ \\	
		Local training (FP)& $84.20 \pm 1.53$\\	
		Local training (2 Bits)& $82.68 \pm 0.34$	\\
		Local training (1 Bit)&  $82.12\pm 2.17$	\\
		QuPeD (FP)&  $\mathbf{85.06} \pm 1.07$	\\
		QuPeD (2 Bits)&  $83.62 \pm 0.50$\\	
		QuPeD (1 Bit)&  $82.72\pm 1.20$ 
	\end{tabular}
	\label{tab:Table AG News}
\end{figure}
These results demonstrate the effectiveness of QuPeD on text data in comparison with local training. 

%
%

	\label{submission}

\end{document}